\documentclass[twoside,10pt]{article}

\PassOptionsToPackage{numbers,sort&compress}{natbib}
\usepackage{blindtext}

\usepackage[margin=0.75in]{geometry}

\usepackage{etoolbox}
\usepackage{ifxetex,ifluatex}
\makeatletter
\newcommand{\TNRswitch}{\fontfamily{ptm}\selectfont} 
\ifxetex
  \usepackage{fontspec}
  \renewcommand{\TNRswitch}{\fontspec{Times New Roman}}
\else\ifluatex
  \usepackage{fontspec}
  \renewcommand{\TNRswitch}{\fontspec{Times New Roman}}
\fi\fi
\def\name{\normalsize\mdseries\TNRswitch}
\patchcmd{\@maketitle}{\Large\bf}{\Large\mdseries\TNRswitch}{}{}%
\makeatother

\usepackage[abbrvbib]{jmlr2e}  
\bibpunct{[}{]}{,}{n}{,}{,}   

\usepackage{lastpage}

\usepackage{blindtext}
\usepackage{amsmath}
\usepackage{mathrsfs}
\usepackage{upgreek}
\usepackage{subcaption}

\newcommand{\Kr}{K_{\boldsymbol{\lambda}}} 
\newcommand{\kr}{k_{\boldsymbol{\lambda}}} 
\newcommand{\bx}{\mathcal{B}_{2}(\mathcal{X})} 
\newcommand{\bz}{\mathcal{B}_{2}(\mathcal{X} \times \mathcal{Y})} 
\newcommand{\h}{\mathcal{H}_{\kr}} 
\newcommand{\p}{\psi^{\boldsymbol{\lambda}}} 
\newcommand{\pj}{\psi_{\lambda _{l},j_{l}}} 
\newcommand{\g}{g_{\lambda _{l},j_{l}}}
\newcommand{\vp}{\varphi^{\boldsymbol{\lambda}}} 
\newcommand{\w}{\omega _{\kappa}(\rho)}
\newcommand{\wi}{\omega _{\kappa }(\rho_{X}^{(i)})}
\newcommand{\rk}{\rho _{\kappa}} 
\newcommand{\ka}{\kappa} 
\newcommand{\rb}{r^{\boldsymbol{\lambda}}}

\newcommand{\il}{\mathrm{I}_{\boldsymbol{\lambda}}} 
\newcommand{\hf}{\mathcal{H}_{\mathcal{F}}} 
\newcommand{\hp}{\mathcal{H}_{\mathrm{para}}} 
\newcommand{\bo}{\mathrm{B}(\hp;\mathbb{R}^{d})} 
\newcommand{\es}{e _{s}} 
\newcommand{\pg}{\mathcal{P_{G}^{X}}} 
\newcommand{\bq}{\widetilde{b}_{q}^{(j,l)}(\rho)} 
\newcommand{\gm}{\hat{g}_{\lambda _{l},j_{l}}^{\tilde{m}}} 
\newcommand{\pt}{\mathcal{T}_{1,\odot}}
\newcommand{\tm}{\tilde{m}}

\newcommand{\ts}{\mathscr{T}_{M}(\mathrm{T}_{\mathbb{S},n})}

\newcommand{\eg}{\mathcal{E}_{N}}
\newcommand{\ex}{\mathcal{E}_{N,\mathcal{X}}}
\newcommand{\ho}{\mathcal{H}_{\mathrm{T}_{n}}}
\newcommand{\co}{C(\Omega)}
\newcommand{\tx}{\tilde{X}}
\newcommand{\htt}{\mathcal{H}_{\mathrm{T}_{n}}}
\newcommand{\gp}{\mathcal{J} _{\Phi}}
\newcommand{\Pg}{\Phi _{\mathcal{G}}}
\newcommand{\E}{\mathcal{E}}
\newcommand{\wg}{\widetilde{\mathcal{J}}}

\newcommand{\st}{\widetilde{S}}

\newcommand{\pb}{\bar{\Phi}}
\newcommand{\te}{\tilde{\epsilon}}
\newcommand{\vr}{\varrho}
\newcommand{\lt}{\tilde{\Lambda}}

\newcommand{\tc}{\mathcal{T}_{C_\mathcal{B}}}
\newcommand{\bc}{C_{\mathcal{B}}}
\newcommand{\gb}{\mathcal{G}_{[t_{0},T]}(\mathrm{B})}
\newcommand{\kt}{\tilde{\kappa}}
\newcommand{\ob}{\Omega _{\mathcal{B}}}
\newcommand{\Bb}{\mathcal{B}_{2,b}(\mathcal{X})}
\newcommand{\vk}{v_{\tilde{\kappa}}}
\newcommand{\bk}{b_{\tilde{\kappa}}}
\newcommand{\upp}{\uppsi_{2}(\rho_{X}^{(i)})}
\newcommand{\upn}{\overline{\uppsi_{2}(\rho_{X})}}
\newcommand{\nx}{\nu_{\mathcal{G}}}
\newcommand{\T}{\mathrm{T}}
\newcommand{\ngg}{\nu_{\mathcal{G}}}
\newcommand{\mg}{\mu_{\mathcal{G}}}
\newcommand{\tmg}{\tilde{\mu}_{\mathcal{G}}}
\newcommand{\at}{{a}_{s}^{(j,l)}(x,y)}
\newcommand{\ad}{{a}_{s}^{(j,l)}(x,\cdot)}
\newcommand{\di}{d_{\vartheta}^{(i)}}
\newcommand{\qedblack}{\hfill\ensuremath{\blacksquare}}

\ShortHeadings{Sample JMLR Paper}{One and Two}
\firstpageno{1}

\begin{document}

\title{Ghost in the Kernel: In-Context Learning with Efficient Transformers via Domain Generalization}

\author{\name Peilin Liu  $^1$\thanks{Email: peilin.6liu@gmail.com}  
       \quad
       \name Ding-Xuan Zhou$^1$\thanks{Email: dingxuan.zhou@sydney.edu.au} \\
       \vspace{1em}
       \addr $^1$School of Mathematics and Statistics, University of Sydney, NSW Australia
       }

\maketitle
\pagestyle{plain}

\begin{abstract}
Transformer-based large models have demonstrated remarkable generalization abilities across different tasks by leveraging a context-aware attention module for in-context learning. With richer context, transformers adapt more effectively to the current use case without any parameter updates. However, the quadratic computational and memory complexity with respect to context length significantly slows data processing in softmax transformers. Linear transformers were proposed to address this issue by reducing the complexity to linear dependence on context length, but the design and understanding of the feature mapping in linear attention, from a theoretical viewpoint, remain unclear. In this paper, we investigate the approximation and generalization abilities of linear transformers under a two-staged sampling process from domain generalization. We show that linear transformers perform in-context learning as learning a mapping from context distributions to response functions. A dimension-independent convergence rate is obtained for our generalization analysis, which also exhibits the tradeoff between the regularities of data distributions and latent features. Guided by our theoretical framework, we propose a new perspective on activation and loss design for linearizing pretrained softmax large language models. 
\end{abstract}


\section{Introduction}

Transformer-based neural networks have become the foundation of modern deep learning frameworks for natural language processing \citep{devlinBERTPretrainingDeep2018,brownLanguageModelsAre2020a} and computer vision \citep{heMaskedAutoencodersAre2021,peebles2023scalable}. Especially when trained with large and diverse corpora, transformers exhibit remarkable few/zero-shot generalization capabilities across various downstream tasks \citep{brownLanguageModelsAre2020a}.
An underlying mechanism for this behavior is in-context learning, in which pretrained large language models (LLMs) condition on instructions or a few input-output pairs (both referred to as \textit{prompts}) and make predictions on test examples without parameter updates. Many theoretical and empirical studies \citep{xie2022explanation,garg2023what,akyurek2023what,zhang2023trained} have demonstrated that the emergence of in-context learning capability is closely related with the context-aware structure of the attention mechanism \citep[see][]{vaswaniAttentionAllYou2017} which enables each token in a sequence to adaptively weight information from all other tokens and to produce a representation conditioned on the sequence context.

However, the original softmax attention mechanism in \citet{vaswaniAttentionAllYou2017} is a double-edged sword: while it is beneficial for context-based representation learning, it suffers from quadratic computational complexity with respect to the context length \citep[see][]{zaheer2020big}. As context length grows dramatically, the quadratic computational and memory costs of the standard attention increasingly hinder autoregressive training and inference, undermining LLM performance in long-context modeling scenarios such as processing entire codebases, preserving coherence in long conversations and performing in-depth reasoning across several documents. Therefore, it's crucial for the design of LLMs to alleviate the curse of quadratic complexity and improve long-context processing capabilities. Numerous works have been recently proposed with this motivation and show performance comparable to the standard attention, including RetNet \citep{sun2023retentive}, Mamba \citep{guMambaLinearTimeSequence2023}, and Gated Linear Attention \citep{yang2024gated}. One branch of these works is known as linear attention \citep[see][]{katharopoulos2020transformers,choromanskiRethinkingAttentionPerformers2021, yang2024gated,zhang2024hedgehog}, which replaces the exponential similarity function with a dot product of key/query functions and yields a linear computational complexity with respect to the context length. This reduction in time and memory cost enables much longer contexts and lower latency during the inference. Meanwhile, the introduction of the hidden-state memory matrix and the forgetting gate improves algorithm stability over utra-long contexts \citep[see][]{yang2024gated}. Although linear attention models have outperformed the softmax attention on some long-context modeling tasks, the theoretical understanding and design principles of these models, especially for in-context learning, remain limited and unexplored. 

In this work, we investigate the approximation and generalization abilities of linear transformers with context-augmented inputs to reveal the advantage of the linear attention mechanism for the in-context learning scenario. We establish a connection between in-context learning and domain generalization frameworks and show that transformer-based neural networks perform in-context learning as domain generalization \citep{blanchard2011generalizing,blanchard2021domain}, and this connection demonstrates the essence of LLMs' remarkable few/zero-shot generalization capabilities without parameter updates during testing. We work with the formulation in \citet{liu2025generalization} by representing the context information as a kernel embedding from context probability distributions to vector-valued functions, and exhibit how each word token interacts with the context embedding through an inner product of a tensor product Hilbert space. Based on this framework, we construct a linear transformer to perform in-context learning via a two-staged sampling process, which shows the internal mechanism of the robust generalization capabilities of LLMs. Our main contributions are as follows.

\begin{itemize}
  \item We present a theoretical analysis framework for the family of linear transformers, one of the most compelling alternatives to the softmax attention in practice. By connecting domain generalization framework with in-context learning, we rigorously prove that linear transformers perform a robust generalization ability under distribution shifts, which builds the theoretical foundation for understanding the generalization abilities of linear transformers in long-context modeling.
  \item We observe a fast eigendecay phenomenon in the softmax attention weight matrix products of LLMs. We prove that this phenomenon helps linear transformers alleviate the negative effects of distribution shifts, achieve dimension-independent convergence rates in approximation and generalization analysis and efficiently mimic the behavior of the standard attention. 
  \item We investigate the application of likelihood ratio moments to control distribution shifts in domain generalization. We apply a relaxed condition for unbounded likelihood ratios of probability distributions defined on the noncompact space $\mathbb{R}^d$ and obtain a distribution-dependent concentration inequality for a second stage estimation by the connection between subgaussian norm and finite Rényi divergence.
  \item We propose a new perspective for linear conversion of LLMs with the softmax attentions based on our theoretical analysis framework. This conversion scheme captures information from data distributions and parameter matrices in pretrained softmax LLMs. It provides a new perspective on designing new activation functions and training loss for linear conversion of softmax LLMs. 
\end{itemize}

In the following part of this paper, we first introduce the motivation and definition of linear transformers and a two-staged sampling process as our learning framework. Section \ref{main_re} presents the main results on approximation and generalization, with a proof sketch in Subsection \ref{proof_sketch}. Section \ref{discussion} provides further discussion, and Appendix \ref{proof} contains the full proofs.

\section{Linear Transformers and Formulations for In-Context Learning}

In this Section, we first define the structure of linear transformers whose inputs are pairs $(\hat{\rho},x)$, where $\hat{\rho}$ the empirical version of distribution $\rho$ from which $x$ is sampled. We refer to learning with samples $(\hat{\rho},x)$ as \textit{in-context learning}, and these samples are generated from the two-staged sampling process defined in Subsection \ref{two-stage}.  

\subsection{Linear Transformers}

The standard Transformer \citep{vaswaniAttentionAllYou2017} consists of blocks of attention mechanisms and shallow networks to process sequential inputs. Let the input sequence $Q = [x_1, \cdots, x_n]^T$ with token vectors $x_i \in \mathbb{R}^d$ for $1 \leq i \leq n$. Then $Q$ is an input sequence of length $n$ with feature dimension $d$. The softmax attention is defined as, for $1 \leq i \leq n$, 
\begin{align} \label{eq:attention}
    \operatorname{SoftmaxAttn}(x_i|Q) = \frac{\sum_{j=1}^n \operatorname{sim}(x_i, x_j)(W_v x_j)}{\sum_{j=1}^n \operatorname{sim}(x_i, x_j)} \in \mathbb{R}^d \text{ with } \operatorname{sim}(x_i,x_j) = \exp \left( \frac{\langle W_q x_i, W_k x_j\rangle}{\sqrt{d'}} \right) 
\end{align}
where $W_v \in \mathbb{R}^{d \times d}, W_{q} \in \mathbb{R}^{d' \times d}, W_{k} \in \mathbb{R}^{d' \times d}$ are parameter matrices for \textit{value}, \textit{query}, and \textit{key} token vectors respectively. Intuitively, the attention mechanism $\operatorname{SoftmaxAttn}$ takes the input sequence $Q$ as context and produces a refined context-aware representation $\operatorname{SoftmaxAttn}(x_i |Q)$ for each query token $x_i$ in $Q$. 

To establish connections between each token and their context and to demonstrate the benefits of context-aware representation produced by the attention mechanism, we follow \citet{liu2025generalization,furuyatransformers} to assume that $Q$ is a realization with $n$ samples drawn i.i.d. from a Borel probability measure $\rho$ on $\mathbb{R}^d$, with $\rho$ regarded as \textbf{\textit{the ground truth context}}. Then the RHS of the expression (\ref{eq:attention}) can be written as 
\begin{align} \label{eq:context_attn}
    \operatorname{ SoftmaxAttn }(\hat{\rho},x_i)=  \frac{{\int \operatorname{ sim }(x_{i}, x) (W_{v}x) \, d \hat{\rho}(x)}}{\int \operatorname{ sim }(x_{i},x) \, d \hat{\rho}(x)} \in \mathbb{R}^d
\end{align}
where $(\hat{\rho},x_i)$ is a context-augmented input and $\hat{\rho}= \delta([x_1, \cdots, x_n])$ is \textbf{\textit{the accessible context}} with $\delta(\mathcal{S})$ defined as the empirical distribution generated by the dataset $\mathcal{S}$. With a richer accessible context $\hat{\rho}$ by more and more samplings from $\rho$, the empirical distribution $\hat{\rho}$ can recover the information of the population distribution $\rho$ and produce more refined context-aware representation for each token $x_i$, which is consistent with the empirical practice of increasing the context window length $n$. 

However, it's easy to observe that for each query token $x_i$ ($1\leq i \leq n$), we must evaluate $\operatorname{sim}(x_i, x_j)$ for all $1 \leq j \leq n$ and then normalize by $\sum_{j=1}^n \operatorname{sim}(x_i,x_j)$, which creates an $n \times n$ attention score matrix and causes both time and memory cost to scale quadratically in the context length $n$ \citep[see][]{zaheer2020big}. Such quadratic growth poses significant challenges for efficient algorithm design in long-context modeling scenarios. The key idea of linear transformers is to reduce the quadratic time and memory cost to linear dependence on the context length $n$ by decoupling queries and keys  in $\operatorname{sim}(x_i,x_j)$. By replacing the similarity function $\operatorname{sim}(x_i,x_j)$ with $\phi(x_i)^T\phi(x_j)$ where $\phi:\mathbb{R}^d \to \mathbb{R}^{d'}$ is a feature mapping, a simple linear attention module can be written as 
\begin{align} \label{eq:linear-attention}
    \operatorname{LinearAttn}(\hat{\rho},x_i) = \frac{{ \phi(x_{i})^{T}\int \phi (x) (W_{v}x)  \, d \hat{\rho}(x) }}{\phi(x_{i})^{T} \int  \phi(x) \, d \hat{\rho}(x)} = \frac{{ \big[\int  (W_{v}x) \phi(x)^T \, d \hat{\rho}(x) }\big]\phi(x_{i})}{\big[\int  \phi(x)^T \, d \hat{\rho}(x)\big]\phi(x_i)}.
\end{align}
It's easy to observe that a universal memory matrix $\int(W_v x) \phi(x)^T d\hat{\rho}(x) \in \mathbb{R}^{d \times d'}$ can be shared across all query tokens, thus eliminating the need to compute and store the quadratic attention score matrix in \eqref{eq:attention}. Beyond this computational efficiency, recent empirical studies \citep{qin2022devil,yang2024gated,zhang2024hedgehog} have achieved performances comparable to the softmax attention using linear attentions. Motivated by normalization-free linear attentions in \citet{qin2022devil} and the design of shallow neural network feature mapping $\phi$ in \citet{zhang2024hedgehog}, we define \textit{Linear Transformers} as follows. Let $\sigma:\mathbb{R} \to \mathbb{R} $ denote the ReLU activation function $\sigma(u)=\max\{u,0\}$, $\sigma_{\tanh}:\mathbb{R} \to \mathbb{R}$ denote the tanh activation function $\sigma_{\tanh}(u) = \frac{\exp(u)-\exp(-u)}{\exp(u)+\exp(-u)}$. 

\begin{definition} \label{def:transformer}
    A Linear Transformer $\mathrm{T}_{n}$ with the structure of $\mathrm{T}_{n,m,\tm}$ and $m=m(n)$, $\tm=\tm(n)$ is defined as 
    \begin{align} \label{eq:lin_tran}
        \mathrm{T}_{n}(\rho,x)= \sum_{j=1}^{n} \alpha _{j} \sigma\left( \sum_{q=1}^{m(n)} {{\phi _{q,\tm(n)}(x)}} {{\left( \sum_{p=1}^{m(n)}\mathcal{T}_{v}\left[ \int \phi _{p,\tm(n)}(y) A_{p,q}^{(j)} y \, d \rho(y) \right]  \right)}}+b_{j}\right)+b_{0} 
    \end{align}
    for context-augmented input $(\rho,x )$ with $A_{p,q}^{(j)} \in \mathbb{R}^{d\times d}$, $b_j,b_0 \in \mathbb{R}^d$ and $\alpha_j \in \mathbb{R}$ for $1\leq j \leq n$, and two-hidden-layer tanh neural networks $\{ \phi _{q,\tm(n)}  \}_{q=1}^{m(n)}$ with the product gate, in the form of $$\phi _{q,\tm(n)}= \mathcal{T}_{1,\odot}\Big(\mathcal{NN} _{q,\tm(n)}\Big) \text{ and } \mathcal{NN} _{q,\tm(n)}(x)=W_{q,2}\sigma_{\tanh}(W_{q,1}\sigma _{\tanh}(W_{q,0}x+b_{q,0})+b_{q,1})$$ with $W_{q,0}\in \mathbb{R}^{8d\tm(n)\times d}, b_{q,0} \in \mathbb{R}^{8d\tm(n)}$, $W_{q,1} \in \mathbb{R}^{8d\tm(n)\times 8d\tm(n)}$, $b_{q,1} \in \mathbb{R}^{8d\tm(n)}$ and $W_{q,2} \in \mathbb{R}^{d \times 8d\tm(n)},$ where $\mathcal{T}_{1,\odot}(z)= \prod_{l} (\mathcal{T}_1(z))^{(l)}$ denotes the product of all entries $(\mathcal{T}_1(z))^{(l)}$ of a truncated vector $\mathcal{T}_1(z)$ and truncation operators $\mathcal{T}_1$ and $\mathcal{T}_{v}$ apply element-wise on input vectors such that $(\mathcal{T}_v(z))^{(l)}= \operatorname{sgn}(z^{(l)}) \min (v, |z^{(l)}|)$ with truncation level $v>0$. 
\end{definition}

\begin{remark}

    Linear transformers in (\ref{eq:lin_tran}) show that for each query $x$, there's a universal memory unit compressing information from the context on each attention head $(p,q)$, in the form of $\mathcal{T}_{v}\left[ \int \phi _{p,\tm(n)}(y) A_{p,q}^{(j)} y \, d \rho(y) \right]$. We note that the truncation operator $\mathcal{T}_{v}$ is actually not required to derive the final generalization bound, but is necessary to obtain an oracle inequality for each linear transformers in hypothesis space, since we consider Borel probability measure $\rho$ defined on the entire $\mathbb{R}^d$ in this paper. To construct the accessible context $\hat{\rho}$, we collect unbounded samples from $\mathbb{R}^d$ and create the memory unit $\int \phi _{p,\tm(n)}(y) A_{p,q}^{(j)} y \, d \hat{\rho}(y)$ that is unbounded and will effect sampling estimation if the outer shallow neural network does not have a special structure. Thus, for the algorithm stability, we introduce the truncation operator $\mathcal{T}_{v}$ to keep the memory unit stable, similar as the idea in \citet{yang2024gated}.

    Linear transformers defined by \eqref{eq:lin_tran} do not have a normalization factor as equation \eqref{eq:linear-attention} does in the form of ${\phi(x_{i})^{T} \int  \phi(x) \, d \hat{\rho}(x)}$. One benefit of normalization-free linear transformer is that it allows a more flexible design of the feature mapping $\phi$. In practice, to ensure the normalization factor in \eqref{eq:linear-attention} is nonzero, $\phi$ is often constrained to be nonnegative. It is not required any more with normalization-free linear transformers. \citet{qin2022devil} also shows that RMSNorm \citep{zhang2019root} can play a better role than normalization factor for stabilizing the optimization of linear transformers. For more details, we refer readers to Subsection \ref{sec:RMS}. 

    We apply two activation functions (ReLU and tanh) and a product gate in the construction of linear Transformer architecture, which is actually components of the modern activation function SwiGLU \citep{shazeer2020glu} for LLMs. The tanh activation is chosen in linear attention to approximate functions in a reproducing kernel Hilbert space (RKHS) induced by a smooth kernel. We include more discussion on this point in Subsection \ref{sec:acf}. 
\end{remark}

\subsection{Two-Staged Sampling Framework for In-Context Learning} \label{two-stage}


Our first novelty is to formulate the sequential modeling of transformers in in-context learning as the processing context-augmented inputs $(\hat\rho,x)$ as in \eqref{eq:context_attn}\eqref{eq:linear-attention}\eqref{eq:lin_tran} where samples like $(\hat{\rho},x)$ are generated by a \textit{\textbf{two-staged sampling framework}} from domain generalization \citep{blanchard2011generalizing,blanchard2021domain}.

Let $\mathcal{X} = \mathbb{R}^{d}$ denote the input space and $\mathcal{Y} = \big\{ y \in \mathbb{R}^{d}: \|y\|_2 \leq M \big\} $ a closed ball with radius $M >0$ be the output space. Let $\mathcal{B}_2({\mathcal{X}})$ and $\mathcal{B}_2({\mathcal{X} \times \mathcal{Y}})$ denote the set of all Borel probability measures with finite second moments on $\mathcal{X}$ and $\mathcal{X} \times \mathcal{Y}$, respectively. We equip $\bx$ and $\bz$ with Wasserstein-2 distances denoted by $W_2$ which metrize the weak topology on $\bx$ and $\bz$. Then for context-augment inputs $(\rho,x)$, we define a complete separable metric space $\Omega = \bx \times \mathcal{X}$ equipped with the metric $d_{\Omega}$ as $$d_{\Omega}\big((\rho, x), (\rho',x') \big)= \sqrt{W_2(\rho,\rho')^2+\|x-x\|_2^2}.$$

\begin{definition} \label{def:two_stage}
    A \textbf{two-staged sampling process} by meta probability measure $\mathcal{P}_{\mathcal{G}}$ is defined as follows: in the first stage sampling with $N \in \mathbb{N}$, $(\rho_{XY}^{(i)})_{i=1}^N$ are independently sampled from a meta Borel probability measure $\mathcal{P}_{\mathcal{G}}$ on $\mathcal{B}_2({\mathcal{X} \times \mathcal{Y}})$; in the second stage sampling with $n_i \in  \mathbb{N}$ for $1\leq i \leq N$, a dataset $\mathbb{S}=\{(\hat{\rho}_X^{(i)}, X_{ij}, Y_{ij})_{j=1}^{n_i}\}_{i=1}^N$ is created by $(X_{ij},Y_{ij})$ sampled independently from $\rho_{XY}$ and $\hat{\rho}_X^{(i)}=\delta([X_{i1},..., X_{in_i}])$. 
\end{definition}

With the two-staged sampling process induced by $\mathcal{P}_{\mathcal{G}}$, for a prediction function $\Phi: \Omega \to \mathbb{R}^d $, we define the population risk for in-context learning as 
\begin{align} \label{eq:loss}
    \mathcal{E}(\Phi)=\mathbb{E}_{\rho_{XY} \sim \mathcal{P}_{\mathcal{G}}} \mathbb{E}_{(X,Y) \sim \rho_{XY}} \| \Phi(\rho_{X},X)-Y \|_{2}^{2}
\end{align}
and the empirical risk with a dataset $\mathbb{S}=\{(\hat{\rho}_X^{(i)}, X_{ij}, Y_{ij})_{j=1}^{n_i}\}_{i=1}^N$ as $${\mathcal{E}}_{\mathbb{S}}(\Phi)=\frac{1}{N} \sum_{i=1}^{N} \frac{1}{n_{i}} \sum_{j=1}^{n_{i}} \| \Phi(\hat{\rho}_{X}^{(i)},X_{ij})- Y_{ij} \|_{2}^{2}.$$

Let $\htt$ be the hypothesis space of a collection of linear transformers in Definition \ref{def:transformer} and $\mathrm{T}_{\mathbb{S},n} \in \htt$ be the function learned from the empirical risk minimization (ERM) algorithm by 
\begin{align} \label{eq:minimizer}
   \mathrm{T}_{\mathbb{S},n}=\arg\min_{\mathrm{T}_{n} \in\mathcal{H}_{\mathrm{T}_{n}}}{\mathcal{E}}_{\mathbb{S}}(\mathrm{T}_{n}). 
\end{align}
Because of the boundedness of the output space $\mathcal{Y}$, our estimator is given by $\ts$ where $\mathscr{T}_{M}$ is a truncation operator defined on $\mathbb{R}^{d}$ such that $\mathscr{T}_{M}(y)=y \text{ if } \| y \|_{2} \leq M$ otherwise  $M \frac{y}{\| y \|_{2}}.$

\begin{remark}
For an in-context learning dataset $\mathbb{S}$, each sample consists of a context-augmented input $(\hat{\rho}_X,X)$ and a label $Y$. Removing the accessible context $\hat{\rho}_X$ from an input reduces the learning problem to classical regression. On the other side, if we drop the query token $X$, the problem will degenerate to distribution regression as considered in our previous work \citep{liu2025generalization}.     
\end{remark}


\subsection{Latent Feature Space for Context-Augmented Inputs}

Our second purpose is to investigate how context-augmented samples interact within the attention mechanism and to formulate this interaction into an inner product of a tensor-product Hilbert space (\textit{\textbf{the latent feature space}}).    

To mimic normalization-free attention for $(\rho,x)$ inspired by \eqref{eq:context_attn} as $\int_{\mathcal{X}} \operatorname{sim}(x, x')x' d\rho(x')$, we first introduce an anisotropic Gaussian kernel $\kr$ on $\mathcal{X} \times \mathcal{X}$ as $$\kr(x,x')= \exp(-(x-x')^{T}\Sigma _{\boldsymbol{\lambda}}(x-x'))$$ with shape parameter vector $\boldsymbol{\lambda}=[\lambda_1, \cdots, \lambda_d]^T \in \mathbb{R}^d$ and $\Sigma_{\boldsymbol{\lambda}}= \operatorname{diag}(\lambda_1^2, \cdots, \lambda_d^2)$. (For more details about the choice of the anisotropic Gaussian kernel, see Subsection \ref{sec:lin_con}.)
We use kernel $\kr$ to measure the similarity between different tokens and then the attention mechanism induced by $\kr$ outputs $\int_{\mathcal{X}} \kr(x,x') x' d\rho(x')$ for the context-augmented input $(\rho,x)$. 

To better understand how the attention mechanism processes context information for context-augmented inputs, we introduce the following definition for context embedding.

\begin{definition} \label{eq:emb_con}
    Let $\h$ be the reproducing kernel Hilbert space (RKHS) induced by $\kr$. For each context $\rho \in \bx$, we define $\Kr(\rho):\mathcal{X} \to \mathbb{R}^d$ as  
    \begin{align*}
   K_{\boldsymbol{\lambda}}(\rho)= \int  _{\mathcal{X}} \kr(\, \cdot \,,x)x \, d\rho(x) \in \h \otimes \mathbb{R}^{d}. 
   \end{align*}
\end{definition}


Intuitively, $\Kr$ embeds context information $\rho$ into a dictionary function mapping each query $x_{\text{query}}$ to a context-aware representation $$\Kr(\rho)(x_{\text{query}}) = \int_{\mathcal{X}}\kr(x_{\text{query}},x) x \, d\rho(x) \in \mathbb{R}^d.$$ 

Next, we define a feature mapping $\il$ and a latent feature space $\hf$ for context-augmented inputs $(\rho,x)\in \Omega$. In the latent feature space $\hf$, the similarity measure between context inputs not only depends on context information, but also has the properties of the attention mechanism.

\begin{definition}
    Define a latent feature space $\hf=(\h \otimes \mathbb{R}^d)\otimes \h$. Define $\il: \bx \times \mathcal{X} \to \hf$ by 
\begin{align} \label{eq:il}
    \il(\rho,x) = \Kr(\rho) \otimes \kr(x,\cdot) \in \hf,
\end{align}
which introduces an inner product for similarity measure between context-augmented inputs $(\rho,x),(\rho', x') \in \Omega$ as 
\begin{equation} \label{eq:Isim}
    \begin{aligned}
\langle  \il(\rho,x),\il(\rho',x') \rangle _{\mathcal{H}_{\mathcal{F}}}&=\langle \Kr(\rho),\Kr(\rho') \rangle _{\h \otimes \mathbb{R}^{d}} \kr(x,x').
\end{aligned}
\end{equation}
\end{definition}



 The similarity between context-augmented inputs $(\rho,x)$ and $(\rho,x')$ defined in (\ref{eq:Isim}), depends not only on the token values but also on the contexts $\rho$ and $\rho'$. This is consistent with a basic observation in natural language processing: a word may have different meanings in different contexts.

\begin{remark}
    The definition of $\il$ is inspired by the similarity measure in the domain generalization literature \citep{blanchard2011generalizing,blanchard2021domain} where the mapping $(\rho,x) \mapsto \Phi_{k_{\mathcal{B}}} (\rho) \otimes \Phi_{k_{\mathcal{X}}}(x)$ is considered with $\Phi_{k_{\mathcal{B}}}$,$\Phi_{k_{\mathcal{X}}}$ the canonical feature maps of kernels $k_{\mathcal{B}}$ \citep[see][]{christmann2010universal}, $k_{\mathcal{X}}$ respectively. In Appendix \ref{app:inj}, we show that both $\Kr$ and $\il$ are injective and continuous mappings. The injection of $\Kr$ is one of the key features of attention modules: compress context distributions into elements in $\h \otimes \mathbb{R}^d$ and fight against the negative effects of distribution shifts, while preserving the ability to distinguish different context distributions. It would be interesting to extend the results in \citet{sriperumbudurHilbertSpaceEmbeddings} to a quantitative analysis on dissimilar distribution with a small distance in $\h \otimes \mathbb{R}^d$ to investigate this tradeoff created by $\Kr$. 
\end{remark}

\section{Main Results} \label{main_re}

This section states the main results of approximation and generalization analysis for in-context learning with linear transformers. We first introduce some assumptions on $\mathcal{P}_{\mathcal{G}}$ for \textit{the two-staged sampling process} and $\kr$ for \textit{the latent feature space}.

The two-staged sampling process induced by $\mathcal{P}_{\mathcal{G}}$ in Definition \ref{def:two_stage} introduces a probability measure $\pg$ on $\bx$ for context information by 
\begin{align} \label{eq:context_measure}
\pg(E) = \mathcal{P}_{\mathcal{G}}\big( \{\mu \in \bz: \mu \circ \pi_{\mathcal{X}}^{-1} \in E \}   \big)    
\end{align}
for any Borel set $E \in \bx$ with the coordinate map $\pi_\mathcal{X}: \mathcal{X} \times  \mathcal{Y}  \to \mathcal{X}$, and also a probability measure $\nx$ for context-augmented inputs on the product $\sigma$-algebra of $\Omega$ by $\nx(B\times A)=\int_B \rho(A) d \pg(\rho)$ for any Borel set $B \in \bx, A\in \mathcal{X}$.

\begin{assumption} \label{ap:1}
    $\pg$ is supported on a subset $\Bb$ of $\bx$ defined as $$\Bb= \big\{ \rho \in\bx: \mathbb{E}_{X \sim \rho} \| X \|_{2}^{4} \leq \bc^2 \big\}$$ with a constant $\bc>1$. We also denote $\ob = \big\{(\rho,x) \in \Omega: \rho \in \Bb \big\}.$
\end{assumption}

\vskip 0.05in 

\begin{assumption} \label{ap:2}
   There exists $\gamma>1$ and $\ka, C_{\mathcal{G}}>0$ such that $$\int_{\bx} \| \w \|^{2} _{L^{\gamma}(\rk)}\, d \mathcal{P}_{\mathcal{G}}^{\mathcal{X}}(\rho) \leq C_{\mathcal{G}}$$ where $\w$ is the likelihood ratio defined as $\frac{d\rho}{d\rk}$ and $\rk$ is Gaussian probability measure with zero mean and covariance matrix $\ka^2\mathrm{I}_d$.  
\end{assumption}

We provide two examples of meta probability measure $\pg$ satisfying Assumptions \ref{ap:1} and \ref{ap:2} in Appendix \ref{sec:mmp}.

\begin{assumption} \label{ap:3}
    For the anisotropic Gaussian kernel $k_{\boldsymbol{\lambda}}$, we assume $\boldsymbol{\lambda}=\left( \lambda _{l} \right)_{l=1}^{d}$ is a sequence of shape parameters such that $\lambda _{(l)} \leq C_{\theta}l^{-\theta}$ for $1\leq l \leq d$ with the order $\lambda _{(1)}\geq \lambda_{(2)} \geq \dots \geq \lambda _{(d)}>0$ where $C_{\theta}, \theta>0$ are two constants independent of $d$.
\end{assumption}

\begin{remark}
    

By the Portmanteau theorem, $\Bb$ is closed in the $W_2$-topology, which allows probability measures supported on $\Bb$ can be extended to probability measures on the entire $\bx$, matching the framework of the two-staged sampling process. The fourth moment condition here is used only to control the second-stage sampling error with accessible contexts. For approximation and all other sampling errors, we only need the condition that $\mathbb{E}_{X\sim\rho}\|X\|_2^2 \leq \bc$ for any $\rho \in \Bb$.

    The likelihood ratio is a popular tool to control distribution shifts in both theoretical and empirical studies \citep{ma2023optimally, shao2024deepseekmath} and actually the quantity $\|\w\|_{L^{\gamma}(\rk)}$ in Assumption \ref{ap:2} is closely related to the Rényi divergence \citep{renyi1961measures,erven2014rényi} of distribution $\rho$ with respect to the reference distribution $\rk$. 
    
    The fast decay of shape parameters in $\boldsymbol{\lambda}$ is often observed in practice. More details about Assumption \ref{ap:3} can be found in Subsection \ref{sec:lin_con}.

\end{remark}

\subsection{Approximation of Variation Normed Functions}

Our third contribution is to propose an explicit hypothesis space under which linear transformers achieve dimension-independent convergence rates for operator approximation without assuming access to the structure of a latent feature space $\hf$, enabling further study on generalization error analysis.

First we impose a regularity condition on the true predictor for \eqref{eq:loss} using the latent feature mapping $\il$ in \eqref{eq:il} and a \textit{variation normed space}. Let $\hp=\mathcal{H}_{\mathcal{F}} \oplus \mathbb{R}$ (the extra dimension is left for bias weights) and 
\begin{align*}
    \mathrm{B}(\hp;\mathbb{R}^{d})= \Bigl\{ W \in \mathcal{L}(\hp;\mathbb{R}^{d}) \,\Big| \, \| W \|_{\mathrm{op}} \leq 1 \Bigr\}
\end{align*}
the closed unit ball with the operator norm in the space $\mathcal{L}(\hp;\mathbb{R}^{d})$ of all bounded linear operators from $\hp$ to $\mathbb{R}^{d}$. Note that 
the finite dimension of the output space makes $\mathcal{L}(\hp;\mathbb{R}^{d})$ identical with the Hilbert space of Hilbert-Schmidt operators, and hence $\bo$ is weakly compact in $\mathcal{L}(\hp;R^{d})$.
Let $\mathcal{M}(\bo)$ denote the space of all signed Radon measures on $\bo$. 

\begin{definition} \label{def:vs}
    For $\mu \in \mathcal{M}(\bo)$, let $F_{\mu}(h)= \int _{\bo} \sigma_{}( Wh ) \, d \mu(W)$ for $h \in \hp$. The variation normed space $\mathcal{F}_1$ of $\mathbb{R}^d$-valued functions on $\hp$ is defined as $$\mathcal{F}_{1}(\hp; \mathbb{R}^{d})=\Bigl\{ F: \hp \to \mathbb{R}^{d} \, \, \Big|\,\, \| F \|_{\mathcal{F}_{1}}:=\inf _{\mu:F=F_{\mu}}\| \mu \|_{\mathcal{M}}  < \infty  \Bigr\}.$$
\end{definition}

Then we give a dimension-independent approximation result with the hypothesis space 
\begin{equation} \label{eq:hyp}
    \begin{aligned}
    &\htt=\Biggl\{ \mathrm{T}_{n}: \| \alpha \|_{1} \leq 2C_F, \sum_{p,q =1}^{m(n)}\left\| A_{p,q}^{(j)} \right\|_{F}^{2}  \leq d,\| b_{j} \|_{2}  \leq \sqrt{ 2d\bc }  \text{ for each } 1\leq j \leq n, \\ &\| b_{0} \|_{2} \leq C_F \sqrt{ 2d\bc }, \,\left\| \Theta _{\tanh} \right\|_{\infty} \leq c_{1}(c_{2}\log(n))^{c_{3}(\log n)^{2}}    \text{ and truncation level }v=\bc         \Biggr\}.
\end{aligned}
\end{equation}
where $C_F >0$ is a constant, $c_1,c_2,c_3$ are constants depending on $\theta,\gamma$ and $\Theta_{\operatorname{tanh}}$ denotes the parameters in two-layered tanh neural networks satisfying a sparse structure such that 
\begin{align*}
    W_{q,j} = \operatorname{diag}(W_{q,j}^{(1)} , ..., W_{q,j}^{(d)}) \text{ for }j=0,1,2
\end{align*}
where for $1 \leq l \leq d$, $W_{q,0}^{(l)} \in \mathbb{R}^{8\tm(n)\times 1}, W_{q,1}^{(l)}\in\mathbb{R}^{8\tm(n)\times 8\tm(n)}$ and $W_{q,2}^{(l)}\in \mathbb{R}^{1 \times 8\tm(n)}$.

Let $\tilde{X}=(\rho_X,X)$ be the context-augmented input with the ground truth context, and $\tilde{X}_{ij}=(\hat{\rho}_{X}^{(i)},X_{ij})$ with the accessible context. We rewrite $\E(\Phi)=\mathbb{E}_{(\tilde{X},Y)\sim \mathbb{P}_{\mathcal{G}}} \|\Phi(\tilde{X}) - Y\|_2 ^2$ where $\mathbb{P}_{\mathcal{G}}$ is a probability measure induced by the two-staged sampling process with $\mathcal{P}_{\mathcal{G}}$ \citep[see][Page 9]{blanchard2021domain}. Then the regression function for the population risk $\E$ is defined as 
\begin{align} \label{eq:reg}
    \Pg(\tilde{X})= \int_{\mathcal{Y}}y \, d\,  \mathbb{P}_{\mathcal{G}}(\cdot| \tilde{X}).
\end{align}

\begin{theorem} \label{thm:app}
    Let $\Pg=F(\il(\cdot),1)$ with $F \in \mathcal{F}_1(\hp;\mathbb{R}^d)$ and $\|F\|_{\mathcal{F}_1}\leq C_F$. For $0<\xi < \theta$ and $n >C'_{\ka,\theta,\gamma}$, there exists a $\T \in \mathcal{H}_{\T_{2n}}$ such that 
    \begin{align*}
         \| \T - \Pg \|^2 _{L^2(\ngg)} \leq C_*^2 n^{-1} \text{ and } \|\T\|_{C(\Omega)}\leq 2 C_{F}\sqrt{ d(1+\bc) }
    \end{align*}
    with $m= \left\lceil{n^{\frac{\gamma}{2(\gamma-1)\xi}}}\right\rceil$, $\tm=\left\lceil{\left( \frac{1}{2}+\frac{\gamma}{4(\gamma-1)\xi} \right) \log n}\right\rceil$ in (\ref{eq:lin_tran}), where $C_*$ is a constant depending on $\ka, \xi,\gamma,\boldsymbol{\lambda},C_{\mathcal{G}}, \bc, C_{F}$ and $\operatorname{poly}(d)$.
\end{theorem}

\begin{remark}
    Variation normed spaces for neural network approximation have been well studied in \citet{barron1993universal,bachbreaking,korolev2022twolayer,siegel2022sharp,yang2023optimal,siegel2025optimal}. Roughly speaking, a variation normed space can be viewed as the collection of shallow neural networks with infinity width and thus can mimic the function class of target functions arising in practice. Definition \ref{def:vs} is a special case of \citet{korolev2022twolayer} where the authors consider neural networks with values in a Banach space, extending the original result in \citet[Theorem 4]{barron1993universal}. 
\end{remark}

\subsection{Generalization Analysis of In-Context Learning}

Our final contribution is to address unbounded sampling (without domain restrictions) in the two-staged sampling process and establish an oracle inequality in Appendix \ref{sec:oracle}. Combining Theorem \ref{thm:app} with this oracle inequality, we obtain a dimension-independent generalization rate.

\begin{theorem} \label{thm:gen}
    Let $d \geq2$ and $n \geq \max\{3,C'_{\ka,\theta,\gamma},\frac{C_*'^2}{64MC_{d,F,\mathcal{B}}}\}$ and $\Pg=F(\il(\cdot),1)$ for some $F\in \mathcal{F}_1(\hp;\mathbb{R}^d)$. If the parameter $n$ of the hypothesis space $\htt$ and the second-stage sample size $\vartheta$ are chosen as 
    $$n= \lfloor \mathcal{K}_{1} N^{\frac{1}{2+ \gamma/[(\gamma-1)\xi]}} \rfloor \text{ and }\vartheta= N^{3}$$ and $m,\tm$ chosen as in Theorem \ref{thm:app}, then for the estimator generated by the ERM framework for the two-staged sampling process, we have 
    \begin{align} \label{ineq:gen}
        \mathbb{E}\{ \| \ts - \Pg \|_{L^2(\nx)}^2 \}\leq \mathcal{K}_{3}N^{- \frac{\xi}{2\xi+ \frac{\gamma}{\gamma-1}}} (\log N)^{3}
    \end{align}
    where $\mathcal{K}_{3}$ is a constant depending on $\ka, \xi,\gamma,\boldsymbol{\lambda},C_{\mathcal{G}}, \bc, C_{F}$ and $\operatorname{poly}(d)$.
\end{theorem}

\begin{remark}
Although controlling distribution shift via likelihood ratio moments with divergence order $\gamma >1$ provides $L^2(\rho)$ error bounds for every target distributions $\rho$ satisfying Assumptions \ref{ap:1} and \ref{ap:2}, the factor $\frac{\gamma-1}{\gamma}$ in RHS of \eqref{ineq:gen} significantly slows the convergence rate as $\gamma$ approaches $1$. This makes a tradeoff between convergence rate and the capacity of admissible target distributions: smaller $\gamma$ allows more distributions satisfying Assumption \ref{ap:2} as shown in Example \ref{exam:2} but greatly slows convergence. However, we observe the fast spectral decay phenomenon in LLMs (shown in Figure \ref{fig:eigendecay} of Subsection \ref{sec:lin_con}) that mitigates this slowdown for in-context learning. Indeed, under Assumption \ref{ap:3}, the exponent becomes $\frac{(\gamma-1)\xi}{\gamma}$ that slows the rate of tending to infinity as $\gamma \to 1$ when $\xi$ is large.  
 
\end{remark}

\subsection{Proof Sketch} \label{proof_sketch}

The proof of the approximation result in Theorem \ref{thm:app} relies on the following error decomposition
\begin{equation} \label{eq:apperror}
    \begin{aligned} 
    &\left\| \Pg- \mathrm{T}_{2n,m,\tm} \right\|_{L^{2}(\ngg)}\\ \leq & \| \Pg-\mathcal{N}_{2n} \|_{L^{2}(\ngg)}+ \| \mathcal{N}_{2n}- \Psi _{2n,m} \|_{L^{2}(\ngg)}+ \left\| \Psi _{2n,m}- \mathrm{T}_{2n,m,\tm} \right\|_{L^{2}(\ngg)} 
\end{aligned}
\end{equation} for $\mathrm{T}_{2n,m,\tm} \in \mathcal{H}_{\T_{2n}}$, where $\mathcal{N}_{2n}$ is a shallow neural network with operator-valued parameters in $\mathcal{L}(\hp;\mathbb{R}^d)$ and $\Psi _{2n,m}$ is a neural network with latent polynomial features depending on the parameterization of $\kr$, both explicitly constructed in Section \ref{sec:app}. 

The approximation error $\| \Pg-\mathcal{N}_{2n} \|_{L^{2}(\ngg)}$ is estimated by random approximation results from \cite{korolev2022twolayer} in Appendix \ref{sec:app}. The estimation of the last two terms in the RHS of (\ref{eq:apperror}) relies on Assumption \ref{ap:2} to bound $L^2$ errors under first-stage samples $\rho_X$ with $L^2$ error under the reference probability distribution $\rk$, as described in Appendix \ref{sec:app2} and \ref{sec:app3}. More specifically, for the second term $\| \mathcal{N}_{2n}- \Psi _{2n,m} \|_{L^{2}(\ngg)}$, $\mathcal{N}_{2n}$ can be regarded as a neural network with countably infinite feature-function parameters in an RKHS, and $\Psi _{2n,m}$ is a neural network constructed by the optimal choice of selecting only $m$ feature-function parameters based on the parameterization of $\kr$. For the last term $\left\| \Psi _{2n,m}- \mathrm{T}_{2n,m,\tm} \right\|_{L^{2}(\ngg)}$, the idea is to show the linear attentions in $\mathrm{T}_{2n,m,\tm}$ are fast universal approximators for any feature-function parameters in RKHS $\mathcal{H}_{\kr}$ without any prior knowledge on the parameterization of $\kr$. The approximation is considered with the supremum norm on the bounded domain $[-B, B]^d$, and the error outside this domain is controlled by a Gaussian tail decay (see Appendix \ref{appendix2}).


For the generalization analysis, we define the first-stage sampling error $\mathcal{E}_{N}$ as $$\mathcal{E}_{N} (\Phi)= \frac{1}{N} \sum_{i=1}^{N} \mathbb{E}\Big[\| \Phi (\tilde{X})-Y \|_{2}^{2} |\rho_{XY}^{(i)}\Big]$$ and the second-stage sampling error $\ex$ with ground truth context as  $$\mathcal{E}_{N,\mathcal{X}}(\Phi)= \frac{1}{N}\sum_{i=1}^{N} \frac{1}{n_{i}} \sum_{j=1}^{n_{i}} \| \Phi(\rho_{X}^{(i)}, X_{ij})- Y_{ij} \|_{2}^{2}.$$ 

Then for any $\Phi \in \ho$, we have the following error decomposition 
\begin{align*}
\mathcal{E}(\ts)- \mathcal{E}(\Phi _{\mathcal{G}}) &=  \mathcal{E}(\ts)- \eg(\ts)+ \eg(\ts)- \ex(\ts)\\&+ \ex(\ts)-\mathcal{E}_{\mathbb{S}}(\ts) + \mathcal{E}_{\mathbb{S}}(\ts)- \mathcal{E}_{\mathbb{S}}(\Phi) \\ &+\mathcal{E_{\mathbb{S}}}(\Phi) - \ex(\Phi) + \ex(\Phi)- \eg(\Phi)+\eg(\Phi)-\mathcal{E}(\Phi) \\ &+ \mathcal{E}(\Phi)- \mathcal{E}(\Phi _{\mathcal{G}})
\end{align*} with $\E(\ts)- \E(\Pg)= \| \ts - \Pg \|_{L^2(\nx)}^2$ and $\mathcal{E}_{\mathbb{S}}(\ts)- \mathcal{E}_{\mathbb{S}}(\Phi) \leq 0.$

Let 
\begin{align*}
& \mathcal{E}_{1}(\ts)= \Big(\mathcal{E}(\ts)- \mathcal{E}(\Phi _{\mathcal{G}})\Big) - \Big(\eg(\ts)-\eg(\Phi _{\mathcal{G}} )\Big), \\
& \mathcal{E}'_{1}(\Phi)=  \Big(\eg(\Phi)-\eg(\Phi _{\mathcal{G}}) \Big)-\Big(\mathcal{E}(\Phi)-\mathcal{E}(\Phi _{\mathcal{G}}) \Big) , \\
& \mathcal{E}_{2}(\ts) = \eg(\ts) - \ex (\ts), \\
& \mathcal{E}'_{2} (\Phi) =  \ex(\Phi ) -\eg(\Phi) ,  \\
& \mathcal{E}_{3}(\ts) = \ex(\ts) - \mathcal{E}_{\mathbb{S}}(\ts), \\
& \mathcal{E}'_{3}(\Phi) =  \mathcal{E}_{\mathbb{S}}(\Phi)-\ex(\Phi) \text{ and } \mathcal{E}_{4}(\Phi) = \mathcal{E}(\Phi)- \mathcal{E}(\Phi _{\mathcal{G}}).
\end{align*}
Then we have 
\begin{equation} \label{eq:error}
   \begin{aligned} 
\mathcal{E}(\ts)- \mathcal{E}(\Phi _{\mathcal{G}})\leq&  \mathcal{E}_{1}(\ts)+\mathcal{E}'_{1}(\Phi)+ \mathcal{E}_{2}(\ts)+ \mathcal{E}'_{2}(\Phi)\\&+\mathcal{E}_{3}(\ts)+\mathcal{E}'_{3}(\Phi)+\mathcal{E}_{4}(\Phi).
   \end{aligned}  
\end{equation}

In the first-stage sampling, we control the effect of unbounded sampling on $\mathcal{E}_{1}(\ts)$ by using the covering number under a pseudo-metric defined by the supremum of distribution expectations over $\Bb$ (Appendix \ref{error:first}). In the second-stage sampling, the situation becomes more complicated with the structure of linear transformers, so we decompose the second-stage sampling error into two parts: $\mathcal{E}_{2}(\ts),\mathcal{E}'_{2}(\Phi)$ defined using \textit{ground truth contexts} (called \textit{pseudo second-stage sampling} in Appendix \ref{error: gtc}) and $\mathcal{E}_{3}(\ts),\mathcal{E}'_{3}(\Phi)$ defined using \textit{accessible contexts} (Appendix \ref{error:acc}). 

For the case with \textit{ground truth contexts}, the similar idea with the first-stage sampling applies: after introducing Rademacher complexity for sampling estimation, we bound the empirical process by the Dudley integral. We then use concavity to move the expectations over second stage samples into the covering number expression, thereby eliminating the effect of unbounded samples. For the last sampling estimation with \textit{accessible contexts}, the problem becomes even more challenging in the presence of memory units in linear transformers, since both the input samples and the memory-unit outputs are unbounded. Here, we apply an extension of Azuma-McDiarmind’s inequality under subgaussian conditions to obtain a distribution-dependent probability concentration inequality where an observation on the relation between $\|\w\|_{L^{\gamma}(\rk)}$ and subgaussian norm plays an important role. To obtain the final generalization error, we apply the expectation identity for non-negative random variables to bring $\|\w\|_{L^{\gamma}(\rk)}$ out of the denominator and the exponential, so that we can take an expectation of $\|\w\|_{L^{\gamma}(\rk)}$ with respect to $\pg$ by Assumption \ref{ap:2}.

\section{Related Works and Discussions} \label{discussion}

\subsection{In-Context Learning}
Prior studies \citep{garg2023what,akyurek2023what,zhang2023trained,shen2025understanding} often formulate in-context learning as predicting the label of a given sample conditioned an input prompt containing other samples and labels. As noted in \citet{zhang2023trained}, a model $\Phi$ performs in-context learning as $$\Phi: \mathcal{S} \times \mathcal{X} \to \mathcal{Y}, \quad \mathcal{S}= \cup_{n \in \mathbb{N}} \big\{(x_1,y_1,...,x_n,y_n): x_i \in \mathcal{X}, y_i \in \mathcal{Y}\big\}$$ where $\mathcal{X}$ is the input space and $\mathcal{Y}$ the output space, and $\Phi$ is trained on prompts of the form $\mathscr{P}=(x_1,h(x_1), ..., x_{\vartheta},h(x_{\vartheta}),x_{\text{query}})$ with $h \sim \mathcal{P}$ a distribution defined on a function space $H$ to minimize the error $\mathbb{E}_{\mathscr{P}}l(\Phi(\mathscr{P}), h(x_{\text{query}}))$ with a loss function $l$. Previous theoretical work has focused on linear function spaces \citep{zhang2023trained} and Hölder spaces \citep{shen2025understanding}. These studies have demonstrated that transformers can perform well on structured prompts of input-output pairs, but this formulation has two limitations to bridge the gap between theory and application \citep[see][]{min2022rethinking}. First, in-context learning emerges as a property of LLMs after pretraining on tasks like autoregression or diffusion-based generation. A pretrained LLM can perform in-context learning without any parameter updates \citep{xie2022explanation}, which is not consistent with theoretical settings that require training on structured prompts.   
Second, prompts for in-context learning are often unstructured and may lack labels. For example, in machine translation from English to French, the input prompt may contain only instructions in English. Empirical studies \citep{min2022rethinking} also show that the correct mapping between inputs and true labels in prompts has little performance gains for in-context learning: model performance with random labels closely matches that with true labels.

 We address these problems by the domain generalization framework \citep{blanchard2011generalizing,blanchard2021domain} and formulate in-context learning as operator learning with the two-staged sampling process: 
\begin{align} \label{def:in-context}
    \Phi: \hat{\rho}_X^{(i)} \mapsto (h: \mathcal{X} \to \mathcal{Y}), \quad \hat{\rho}^{(i)}_X= \delta([x_{i1},..., x_{in_i}]) \text{ with }x_{ij} \in \mathcal{X}.
\end{align}
This formulation suggests that the operator $\Phi$ maps the context distribution $\hat{\rho}^{(i)}_X$ to a response function $h_{\hat{\rho}^{(i)}_X}$ that takes queries from $\mathcal{X}$ and outputs $h_{\hat{\rho}^{(i)}_X}(x_{\text{query}})$ for any $x_{\text{query}} \in \mathcal{X}$, which aligns with both the nature of transformers as context-based representation learning and also the parameter-freezing setting after pretraining for in-context learning. With a richer unstructured prompt $[x_{i1},...,x_{in}]$ by more and more samplings ($X_{ij} {\sim} \rho^{(i)}_X$) from the ground truth context distribution $\rho^{(i)}_X$, the empirical context distribution $\hat{\rho}^{(i)}_X$ can recover $\rho^{(i)}_X$ and then $\hat{\Phi}(\hat{\rho}^{(i)}_X)$ can well approximate $\Phi({\rho}^{(i)}_X)$ without parameter updates, where $\hat{\Phi}$ is a pretrained Transformer model to approximate the operator $\Phi$.

\subsection{Normalization Factor and RMSNorm} \label{sec:RMS}

Early linear attention models \citep{katharopoulos2020transformers,choromanskiRethinkingAttentionPerformers2021} have some softmax-inspired design features \eqref{eq:attention}, such as inserting a normalization denominator as in \eqref{eq:linear-attention}. However, \citet{qin2022devil} demonstrates both theoretically and empirically that linear transformers with \eqref{eq:linear-attention} make the gradients for attention matrices unbounded and lead to a less stable optimization and worse convergence. To alleviate this negative effect, \citet{qin2022devil} removes the normalization factor in \eqref{eq:linear-attention} and applies RMSNorm \citep{zhang2019root} to the linear attention output:  
\begin{align} \label{eq:dev_lin}
  O_{\text{norm}} = \operatorname{RMSNorm}(Q(K^TV))  
\end{align}
where 
\begin{align*}
    Q &= [\phi(Q_1), \cdots, \phi(Q_n)]^T \in \mathbb{R}^{n \times d'}, \\
    K &= [\phi(K_1), \cdots, \phi(K_n)]^T \in \mathbb{R}^{n \times d'}, \\
    V & = [V_1, \cdots, V_n]^T \in \mathbb{R}^{n \times d} 
\end{align*}
and for input $A=(a_{ij})_{i,j} \in \mathbb{R}^{n \times d}$, $A'=\operatorname{RMSNorm}(A) \in \mathbb{R}^{n \times d}$ is defined as
\begin{align*}
    A'=(a'_{ij})_{i,j} \text{ such that }a'_{ij}= \frac{a_{ij}}{\sqrt{\frac{1}{d}\sum_{j=1}^d a_{ij}^2 + \epsilon}} \cdot \beta_j 
\end{align*}
with learnable scaling factor $\boldsymbol{\beta}=(\beta_1, ...,\beta_d)^T \in \mathbb{R}^d.$ RMSNorm reduces the amount of computation and increases efficiency over LayerNorm, and it is widely used in the open-weight LLMs like Qwen3 \cite{yang2025qwen3a}. 

\subsection{Activation Functions in LLM} \label{sec:acf}

The design of activation functions has evolved with the development of LLMs. While the original transformer used ReLU activation by default, early LLMs such as BERT and GPT-2/3 employed the Gaussian Error Linear Unit (GeLU) activation \citep{hendrycks2023gaussian}, and it then became the standard choice. For $x \in \mathbb{R}^d$, GeLU is defined as $$\operatorname{GeLU}(x)= x \odot \operatorname{F}(x)$$ where $\odot$ denotes the Hadamard product and $\operatorname{F}$ denotes the cumulative distribution function for the standard gaussian and applies element-wise on $d$-dimensional vectors. In practice, GeLU activation function is often implemented via a tanh approximation \footnote{Refer to GeLU implementation in Pytorch 2.8 documentation: \url{https://docs.pytorch.org/docs/stable/generated/torch.nn.GELU.html}} as 
\begin{align*}
    \operatorname{GeLU}(x) \approx 0.5x\odot\left[1+\sigma_{\tanh}\left(\sqrt{\frac{2}{\pi}}
\left(x+0.044715x^{\odot3}\right)\right)\right].
\end{align*}
where $x^{\odot3}$ denotes the Hadamard power of order $3$. Similar activation functions such as Swish activation \citep{ramachandran2017searching} were introduced later. It's worth noting that Swish activation is defined as 
\begin{align*}
    \operatorname{Swish}_{\beta}(x) = x \odot \operatorname{Sigmoid}(\beta\, x)
\end{align*}
where $\beta \in \mathbb{R}$ is a constant or learnable parameter and $\operatorname{Sigmoid}$ applies element-wise on $x$ with 
\begin{align} \label{eq:tanh}
   \operatorname{Sigmoid}(a) = \frac{1}{1+ \exp(-a)}=\frac{1}{2}(1+ \sigma_{\tanh}(a/2)) \text{ for } a \in \mathbb{R}. 
\end{align}
SiLU is a special case of Swish with $\beta=1$. 

Finally, combining all tricks above, \citet{shazeer2020glu} proposed SwiGLU which is widely used in modern LLMs and defined as 
\begin{align*}
    \operatorname{SwiGLU}_{\beta}(x) = W_o((W_1x+b_1) \odot \operatorname{Swish}_{\beta}( W_2x+b_2)).
\end{align*}
In our definition of linear transformers \eqref{def:transformer}, we use three nonlinear components: Tanh activation, product gate, and ReLU activation. It is easy to observe the connection between tanh activation and SwiGLU by Equation \eqref{eq:tanh}. For ReLU activation, when $\beta$ is large enough, $\operatorname{Sigmoid}(\beta \, a)$ approximates the indicator function (except at zero) and $\operatorname{Swish}_{\beta}$ behaves like a ReLU activation function. For product gate, when $\beta=0$, $\operatorname{Swish}_{\beta}(x)=x/2$ and then we can obtain $x \odot x$ by SwiGLU activation with a suitable choice of parameters \citep[see][]{ramachandran2017searching}.

\subsection{Linear Conversion of Softmax LLMs} \label{sec:lin_con}

Distilling knowledge from pretrained softmax LLMs into subquadratic models \citep{zhang2024hedgehog} has recently attracted interest in the research community. Here, we present a perspective on linearizing pretrained softmax LLMs, derived from our theoretical analysis framework. 

\begin{figure}[!t]
  \centering
  \begin{subfigure}[t]{0.5\linewidth}
    \centering
    \includegraphics[width=\linewidth]{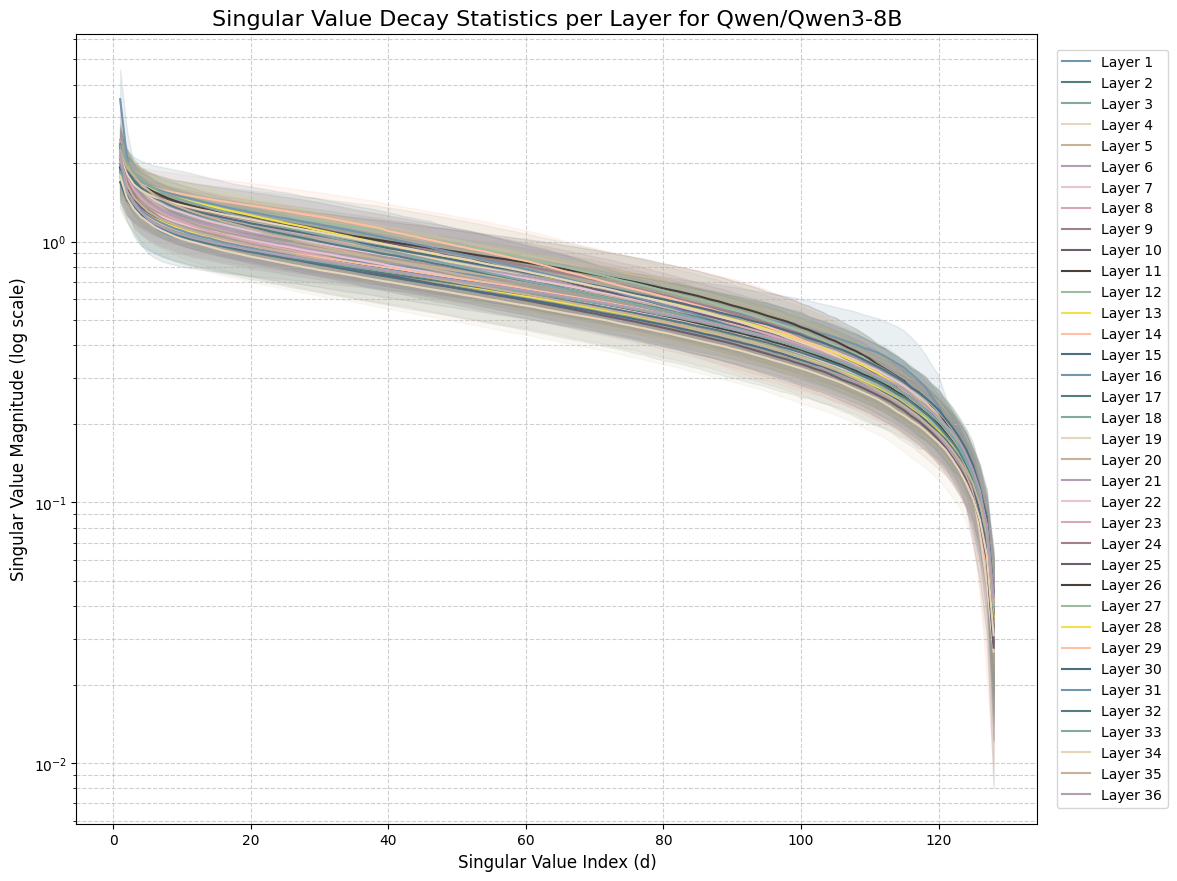}
    \caption{Fast Decay of Singular Values}
  \end{subfigure}\hfill
  \begin{subfigure}[t]{0.5\linewidth}
    \centering
    \includegraphics[width=\linewidth]{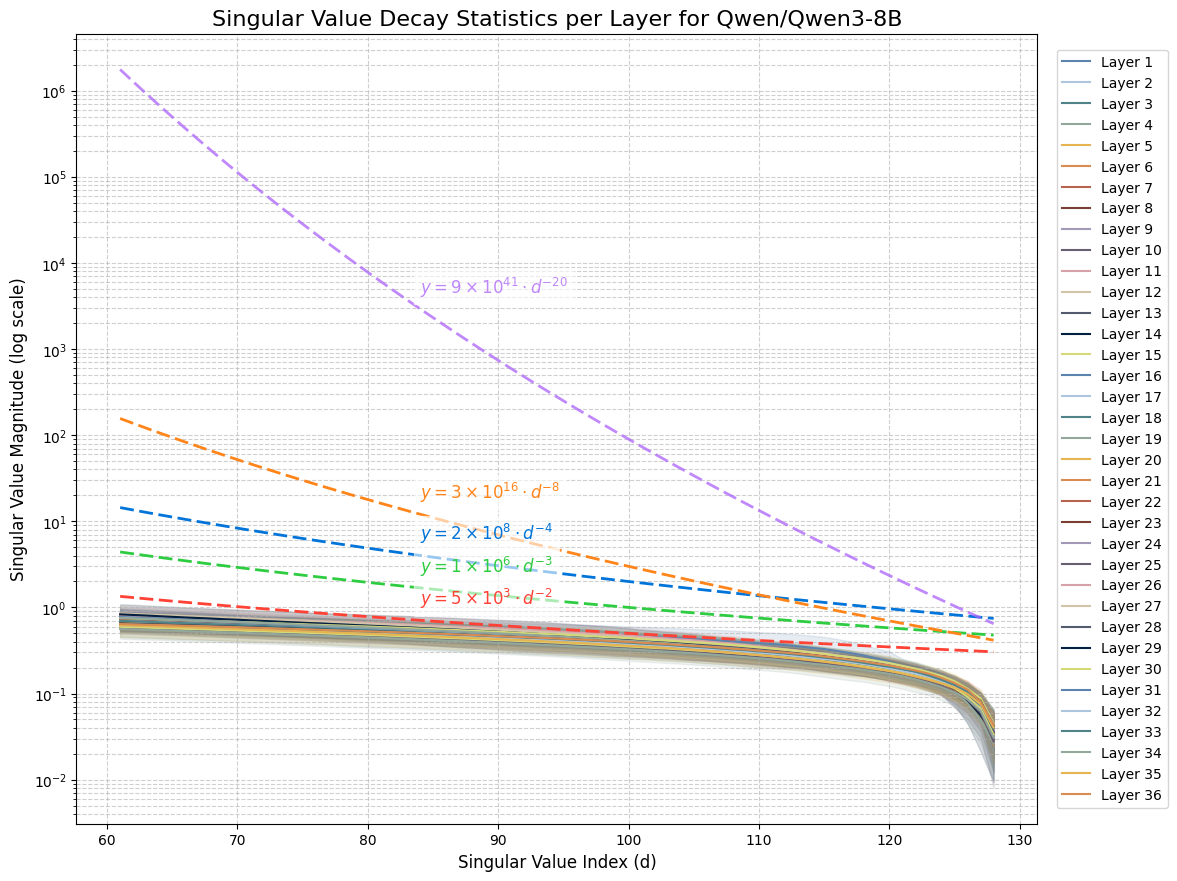}
    \caption{Polynomial Fit of Decay Trend}
  \end{subfigure}
  \caption{Fast Eigendecay of Qwen3-8B (Ghost in the Kernel)}
  \label{fig:eigendecay}
\end{figure}

Recall the softmax attention module \eqref{eq:attention}. Following \citet{tsai2019transformer,liu2025generalization}, we take a kernelized viewpoint of attention modules by letting similarity function $\operatorname{sim}(x_{i},x_{j})= \exp(\langle W_{q}x_{i},W_{k} x_{j}\rangle)$. Then \eqref{eq:attention} can be written as 
\begin{align} \label{eq:attn}
    \operatorname{SoftmaxAttn}(x_{i}|Q)= \frac{1}{\mathrm{Z}(x_{i})}\sum_{j=1}^{n}\operatorname{sim}(x_{i},x_{j})(W_{v}x_{j}) 
\end{align}
where $\mathrm{Z}(x_{i})= \sum_{j=1}^{n}\operatorname{sim}(x_{i},x_{j})$ is a normalization factor. For each pair of query and key weight matrices $(W_{q},W_{k})$, we perform singular value decomposition $W_{q}^{T}W_{k}=W_1^{T} \Sigma_{\boldsymbol{\lambda}} W_2$ where $W_1,W_2$ are $d \times d$ orthogonal matrices and $\Sigma_{\boldsymbol{\lambda}}$ is a positive diagonal matrix with rank $d_{hidden}$. It's obtained that

\begin{equation} \label{eq:sim}
   \begin{aligned} 
\operatorname{sim}(x_{i},x_{j})&=\exp(x_{i}^{T}W_1^{T}\Sigma_{\boldsymbol{\lambda}} W_2x_{j})= \exp(\tilde{x}_{i}^{T}\Sigma_{\boldsymbol{\lambda}}\hat{x}_{j})\\&=\exp\left( \frac{1}{2} \tilde{x}^{T}_{i}\Sigma_{\boldsymbol{\lambda}} \tilde{x}_{i} \right)\exp \left( \frac{1}{2} \hat{x}^{T}_{j} \Sigma_{\boldsymbol{\lambda}}\hat{x}_{j}\right) \exp\left( -\frac{1}{2} (\tilde{x}_{i}-\hat{x}_{j})^{T}\Sigma_{\boldsymbol{\lambda}} (\tilde{x}_{i}-\hat{x}_{j}) \right) 
   \end{aligned} 
\end{equation}
with query $\tilde{x}_{i}=W_1x_{i}$ and key $\hat{x}_{j}=W_2x_{j}$. From the above expression, we observe that the roles of the key and query matrices can be decomposed into kernel asymmetry between queries and keys, represented by orthogonal matrices $W_1,W_2$, and geometric information, represented by a diagonal matrix $\Sigma_{\boldsymbol{\lambda}}$.  In Figure \ref{fig:eigendecay}, we show a numerical demonstration of singular values in diagonal matrices $\Sigma_{\boldsymbol{\lambda}}$ in the attention modules of large language model Qwen3 \citep{yang2025qwen3a} , which exhibits a rapid decay of singular values across layers, nearly exponential for large $d$. The rapid decay of shape parameters enables us to design linear transformers that efficiently mimic softmax attention and alleviate the negative effects of distribution shifts in our analysis. 

For the construction of a linear attention, the key is to decouple the interaction in $\operatorname{sim}(x_i, x_j)$ between queries and keys as shown in (\ref{eq:linear-attention}). For similarity function in Equation \ref{eq:sim}, \textit{the target is to find a decoupling of query-key interaction between $\tilde{x}$ and $\hat{x}$ for the anisotropic gaussian kernel $\kr(\tilde{x},\hat{x})=\exp (- \frac{1}{2}(\tilde{x}-\hat{x})^T \Sigma_{\boldsymbol{\lambda}}(\tilde{x}-\hat{x}))$ to mimic context modeling in softmax attention.} In Appendix \ref{appendix1}, we know that there's an optimal linear approximation scheme for context embedding, denoted as $(\p_q)_{q=1}^m$ with explicit expressions, which only depends on geometric information of $\Sigma_{\boldsymbol{\lambda}}$ and underlying context distributions on $\mathbb{R}^d$. With a suitable choice of $m$, we have 
\begin{align} \label{sim:decom}
   \operatorname{sim}(x_i, x_j) \approx \sum_{q=1}^m \underbrace{ \exp\left( \frac{1}{2} \tilde{x}_{i}^{T}\Sigma _{\boldsymbol{\lambda}}\tilde{x}_{i}  \right) \p_{q}(\tilde{x}_{i}) }_{ query } \cdot \underbrace{ \exp\left( \frac{1}{2} \hat{x}_{i}^{T}\Sigma _{\boldsymbol{\lambda}}\hat{x}_{i}  \right) \p_{q}(\hat{x}_{i})  }_{ key }.
\end{align}

Compared with previous methods \citep{katharopoulos2020transformers, choromanskiRethinkingAttentionPerformers2021} for linear attention, the choice of $(\p_q)$ captures the latent data structures learned by the pretrained softmax LLM and makes use of its parameters. However, in practice, it may be difficult to directly apply the approximation (\ref{sim:decom}), because the semantic distributions of tokens are highly anisotropic and hard to estimate, while we use an isotropic Gaussian to roughly control the class of distributions in the two-staged sampling process. But it still provides us an idea to design new activation functions and regularity for feature mappings for linear conversion of pretrained softmax LLMs, just as in \citet{zhang2024hedgehog}. 

\subsection{Conclusion}

In this paper, we investigate the approximation and generalization ability of linear transformers under the two-staged sampling process from domain generalization. We demonstrate that the remarkable generalization and in-context learning capacities are closely related with context-aware structures in linear transformers, and we obtain a dimension-independent convergence rate in the generalization analysis that reveals a trade-off between the regularity of data distributions and the rate of spectral decay. It would be interesting to investigate the linearization algorithm proposed in Subsection \ref{sec:lin_con} through extensive empirical experiments with various open-weight LLMs and optimization methods from theoretical aspects in operator learning \citep{mucke2021stochastic,nguyen2024optimal}. An efficient, tunable conversion algorithm would be a breakthrough for adapting LLMs to long-context scenarios. On the other hand, hybrid models with both linear and softmax attention have shown their potentials to take advantage of both architectures, and theoretical analysis for such hybrid models would be important for understanding the mathematical foundations of Transformer variants.

\acks{The work described in this article was partially supported by a Discovery Project (DP240101919) of the Australian Research Council. The title ``Ghost in the Kernel" is inspired from the 1995 film \textit{Ghost in the Shell}, directed by Mamoru Oshii. On its 30th anniversary, this paper is dedicated as a tribute to that film, which has had a profound influence on the first author’s exploration of artificial intelligence.}

\newpage

\appendix

\section{Proof Details} \label{proof}

\subsection{Theorem \ref{thm:app}: Approximation Scheme by Linear Transformers} \label{sec:app}

By Proposition \ref{prop:il}, we know that $\il:\Omega \to \hf$ is continuous, which introduces a pushforward probability measure $\mg$ on $\hf$ by $\mg=\ngg \circ \il^{-1}$. Then we define a probability measure $\tmg=\mg \times \delta_1$ on $\hp$. For $h\in \hp$, let $h_{\mathcal{F}}= \operatorname{Proj}_{\hf}(h)$ and we have 
\begin{align*}\int  _{\hp} \| h \|_{\hp}^{2}  \, d \tmg(h)&= 1+\int _{\hf} \| h_{\mathcal{F}} \|_{\hf}^{2}  \, d \mg(h_{\mathcal{F}}) =1+ \int  _{\ob} \| \Kr(\rho) \otimes \kr(x,\cdot) \|_{\hf}^{2}  \, d \ngg(\rho,x) \\ & = 1+ \int _{\ob} \| \Kr(\rho) \|_{\mathcal{H}_{\kr}\otimes \mathbb{R}^{d}}^{2}\| \kr(x,\cdot) \|_{\mathcal{H}_{\kr}}^{2}   \, d \ngg(\rho,x) \leq 1+ \bc. \end{align*}

We also note that for $W\in \mathcal{L}(\hp;\mathbb{R}^d)$, $\|W\|_{\text{op}}\leq \|W\|_{\text{HS}}\leq \sqrt{d}\|W\|_{\text{op}}$ where $\|\cdot\|_{\text{HS}}$ denotes Hilbert-Schmidt norm of $\mathcal{HS}(\hp;\mathbb{R}^d)$. let $\{e_l\}_{l=1}^d$ denote an orthonormal basis of $\mathbb{R}^d$ and for any $h \in\hp$, we have $\langle Wh,  e _{l} \rangle _{\mathbb{R}^{d}}= \langle  h, W^{*} e _{l} \rangle _{\hp}$. Denote $w_{l}=W^{*}e _{l} \in \hp$ and then 
\begin{align} \label{eq:operator}
    Wh= \sum_{l=1}^{d} \langle h, w_{l} \rangle _{\hp} e _{l}= \sum_{l=1}^{d} (e _{l} \otimes w _{l})h.
\end{align}
It implies that the elements in $\bo$ share the form $W=\sum_{l=1}^{d} e _{l} \otimes w_{l}$ where $w_{l} \in \hp$ such that 
\begin{align*}
    \| W \|_{\mathrm{op}}= \sup_{\| h \|_{\hp}= 1} \left( \sum_{l=1}^{d} \langle w_{l},h \rangle _{\hp}^{2} \right)^{\frac{1}{2}}\leq 1.
\end{align*}

For the probability measure $\tmg$ and $F \in \mathcal{F}_1(\hp;\mathbb{R}^d)$, we have the following lemma from Theorem 3.24 in \cite{korolev2022twolayer} for the approximation by shallow nets with operator-valued parameters.

\begin{lemma} \label{lem1}
    For any probability measure $\mu$ on $\hp$ with a finite second moment and $F \in \mathcal{F}_{1}(\hp;\mathbb{R}^{d})$, there exists a shallow neural network $\mathcal{N}_{n}$ such that $$\| F-\mathcal{N}_{n} \|^{2}_{L^{2}_{\mu}(\hp;\mathbb{R}^{d})} \leq \frac{ \| F \|_{\mathcal{F}_{1}}^{2}\,\mathbb{E}_{h \sim\mu}\|h\|^2_{\hp}}{n},$$ and the shallow neural network $\mathcal{N}_{n}$ has the form   
    \begin{align*}
        \mathcal{N}_{n}(h)&=\sum_{j=1}^{n}\alpha _{j} \sigma(W_{j}h) \text{ with } \alpha_j \in \mathbb{R}, W_j \in \mathrm{B}(\hp;\mathbb{R}^d) \text{ for } 1 \leq j \leq n \text{ and } \|\alpha\|_1 \leq \|F\|_{\mathcal{F}_1}.
    \end{align*}
\end{lemma}

We apply the above lemma with $\mu=\tmg$. Then for $h =(h_{\mathcal{F}},1)$, $\mathcal{N}_n(h)$ can be further written by (\ref{eq:operator}) into 
\begin{align*}
    \mathcal{N}_{n}(h)&=\sum_{j=1}^{n}\alpha _{j} \sigma(W_{j}h)=\sum_{j=1}^{n}\alpha _{j}\sigma\left( \sum_{l=1}^{d}(e _{l} \otimes w_{j,l})h \right)\\&=\sum_{j=1}^{n}\alpha _{j}\sigma\left( \sum_{l=1}^{d} (\langle v _{j,l}, h_{\mathcal{F}}  \rangle  _{\hf} + b_{j,l})e _{l} \right) \\ &=\sum_{j=1}^{n} \alpha _{j} \sigma\left( \sum_{l=1} ^{d} \langle  v_{j,l}, h_{\mathcal{F}} \rangle _{\hf}e _{l} + b_{j} \right)
\end{align*}
where $b_{j} = \sum_{l=1}^{d}b_{j,l}e _{l} \in \mathbb{R}^{d}, w_{j,l}=(v_{j,l}, b_{j,l})$ with $v_{j,l} \in \hf$ and $b_{j,l} \in \mathbb{R}$ such that $$\| W_{j} \|_{\mathrm{op}}=\sup_{\| h \|_{\hp}=1} \left( \sum_{l=1}^{d}  \langle w_{j,l}, h \rangle _{\hp}^{2} \right)^{\frac{1}{2}} \leq 1.$$ 
We also know that for $h=(h_{\mathcal{F}},1)$ in the support of $\tmg$, $h_{\mathcal{F}}=\Kr(\rho) \otimes \kr(x, \cdot)$ for some $(\rho,x) \in \ob$. It follows that $\| h \|_{\hp} \leq \sqrt{ 1+C_{\mathcal{B}} }$ and $\| W_{j} h \|_{2} \leq \sqrt{ 1+\bc }$ for $h \in \operatorname{supp}(\tmg)$. Then we construct a double-width shallow neural network $\mathcal{N}_{2n}$ as 
 \begin{align*}
\mathcal{N}_{2n}(h)&=  \sum_{j=1}^{n} \alpha _{j} \Big(\sigma(W_{j}h+\sqrt{ 1+\bc } \mathbf{1}_{d})-\sigma(W_{j}h- \sqrt{ 1+\bc }\mathbf{1}_{d})-\sqrt{ 1+\bc }\mathbf{1}_{d}\Big) \\ &= \sum_{j=1}^{2n} \alpha' _{j} \sigma\left( W'_{j}h+ (-1)^{\left\lfloor  \frac{j-1}{n}  \right\rfloor}\sqrt{ 1+\bc } \mathbf{1}_{d} \right) - \sum_{j=1}^{n} \alpha _{j}\sqrt{ 1+\bc } \mathbf{1}_{d} \\ &= \sum_{j=1}^{2n} \alpha' _{j} \sigma\left( \sum_{l=1}^{d} \langle v'_{j,l}, h_{\mathcal{F}} \rangle _{\mathcal{H}_{\mathcal{F}}} e _{l} + b'_{j} \right) +b'_{0}
\end{align*}  to realize the same approximant in Lemma \ref{lem1} by adding additional bias vectors and letting $\alpha _{j}'= -\alpha' _{j+n}=\alpha _{j}$,  $v'_{j,l}=v'_{j+n,l}=v_{j,l}$, $b'_{j}=b_{j}+\sqrt{ 1+\bc }\mathbf{1}_{d}$, $b'_{j+n}= b_{j}- \sqrt{ 1+\bc }\mathbf{1}_{d}$ for $1 \leq j \leq n$, and $b_{0}= - \sum_{j=1}^{n} \alpha _{j} \sqrt{ 1+ \bc } \mathbf{1}_{d}$. In the following content, we still use $(\alpha _{j}, v_{j,l},b_{j},b_{0})$ as notations for parameters in $\mathcal{N}_{2n}$ with no confusion.

\subsubsection{Neural Network with Latent Polynomial Features} \label{sec:app2}

Recall that for $(\rho,x)$ sampled from the probability measure $\ngg$ on $\ob$, we take the feature $\il(\rho ,x)=\Kr(\rho) \otimes \kr(x,\cdot) \in \hf$ as the model input. Let $(\p_{q})_{q \in \mathbb{N}}$ be the orthonormal basis of $\h$ and $(\rb_q)_{q \in \mathbb{N}}$ the eigenvalue sequence as defined in Appendix \ref{appendix1}. Then for each pair $(j,l)$, $v_{j,l}$ can be written as 
\begin{align*}
  v_{j,l}=\sum_{p,q \in\mathbb{N}} \sum_{1\leq s\leq d} a_{p,q,s}^{(j,l)}(\p_{p} \otimes e _{s}) \otimes \p_{q} \text{ with } \sum_{p,q \in\mathbb{N}} \sum_{1 \leq s \leq d} (a_{p,q,s}^{(j,l)})^{2} \leq 1 . 
\end{align*}
It follows that 
\begin{align*}\langle v_{j,l},\il(\rho,x) \rangle _{\hf }&= \sum_{p,q \in\mathbb{N}} \sum_{1 \leq s\leq d}a_{p,q,s}^{(j,l)}\langle  \p_{p}\otimes e _{s} \otimes \p_{q}, \Kr(\rho) \otimes \kr(x,\cdot) \rangle _{\hf}\\ &=\sum_{p,q \in\mathbb{N}} \sum_{1 \leq s \leq d} a _{p,q,s}^{(j,l)}\p_{q}(x)\int \p_{p}(y) e _{s}^Ty \, d \rho(y). \end{align*}

The basic idea for constructing a neural network with latent polynomial features is to estimate the truncation error for the query index $q$ and the context memory index $p$.

\vskip 0.1in 

First we consider to estimate the truncation error on index $q$. We make one step further by writing $$ \langle  v_{j,l}, \il(\rho,x) \rangle _{\hf} = \sum_{q \in\mathbb{N}}\left( \sum_{p \in\mathbb{N}} \sum_{1 \leq s \leq d} a_{p,q,s}^{(j,l)} \int  \p_{p}(y) e _{ s}^{T}y \, d \rho(y)  \right)\p_{q}(x)=: \sum_{q \in \mathbb{N}}b_{q}^{(j,l)}(\rho)\p_{q}(x).$$ Define 
\begin{align*}
    A^{(j,l)}:l^{2}(\mathbb{N}\times[d]) \to l^{2}(\mathbb{N}) \text{ such that }(A^{(j,l)}z)_{q}=\sum_{p \in\mathbb{N}} \sum_{s \in[d]}a_{p,q,s}^{(j,l)}z_{p,s}
\end{align*}
where $[d]:=\{ 1,\dots,d \}$ and $q \in\mathbb{N}$. It is easy to see that $A^{(j,l)}$ is a Hilbert-Schmidt operator with $\| A^{(j,l)} \|_{\mathrm{HS}}^{2}=\sum _{p,q \in\mathbb{N},s \in[d]}(a_{p,q,s}^{(j,l)})^{2} \leq 1$. Also note that by dominated convergence theorem, $$\sum_{p \in\mathbb{N}} \sum_{s \in[d]}\left( \int \p_{p}(y)\es ^{T}y \, d\rho (y)  \right)^{2}\leq \int \sum_{p \in \mathbb{N}}(\p_{p}(y))^{2} \| y \|^{2} _{2}  \, d\rho(y)=\int  \kr(y,y) \| y \|^{2} _{2} \, d \rho(y) \leq C_{\mathcal{B}},$$ which shows that $\int \p(y) y  \, d\rho(y):=\left( \int \p_{p}(y)\es ^{T}y \, d\rho (y)  \right)_{p \in \mathbb{N},s \in[d]} \in l^{2}(\mathbb{N} \times[d])$. It follows that \begin{align*}
C_{\mathcal{B}}^{\frac{1}{2}} &\geq \| A^{(j,l)} \|_{\mathrm{HS}} \left\|  \int  \p (y)y\, d\rho(y)  \right\|_{l^{2}(\mathbb{N} \times[d])} \\&\geq \left\| A^{(j,l)} \left( \int  \p (y)y\, d\rho(y) \right) \right\|_{l^{2}(\mathbb{N})}=\left( \sum_{q \in\mathbb{N}} (b_{q}^{(j,l)}(\rho))^{2} \right)^{\frac{1}{2}}
\end{align*} and $\sum_{q \in \mathbb{N}} b_{q}^{(j,l)}(\rho)\p_{q} \in\h$ for each $(j,l)$.
Let $$\Psi _{2n,m_{1}}(\rho,x):= \sum_{j=1}^{2n}\alpha _{j}\sigma\left( \sum _{l=1}^{d}\sum_{q=1}^{m_{1}} b_{q}^{(j,l)}(\rho)\p_{q}(x)e _{l} +b_{j} \right) +b_{0}.$$ Then we have 
\begin{align}&\| \mathcal{N}_{2n}(\il(\cdot),1)-\Psi _{2n,m_{1}} \|^{2}_{L^{2}(\ngg)}\notag\\&= \int  _{\ob}\left\| \mathcal{N}_{2n}(\il(\rho,x),1)-\Psi _{2n,m_{1}}(\rho,x) \right\|^{2}  _{2}\, d \ngg(\rho,x)\notag\\&\leq \int  _{\ob} \left\| \sum_{j=1}^{2n} \left|  \alpha _{j} \right| \sum_{l=1}^{d} \left| \sum_{q\geq m_{1}+1}b_{q}^{(j,l)}(\rho)\p_{q}(x)  \right| e _{l}   \right\|^{2} _{2} \, d \ngg(\rho,x) \notag\\&= \int _{\ob} \sum_{l=1}^{d} \left( \sum_{j=1}^{2n}\left| \alpha _{j} \right|\left| \sum_{q\geq m_{1}+1}b_{q}^{(j,l)}(\rho)\p_{q}(x) \right|   \right)^{2} \, d\ngg(\rho,x) \notag\\ & \leq \int _{\ob} \sum_{l=1}^{d} \| \alpha \|_{1} \left[ \sum_{j=1}^{2n} \left| \alpha _{j} \right|\left( \sum_{q\geq m_{1}+1} b_{q}^{(j,l)}(\rho) \p_{q}(x) \right)^{2}   \right] d\ngg(\rho,x) \notag\\ &= \int  _{\Bb} \| \alpha \|_{1}  \sum_{l=1}^{d} \left[ \sum_{j=1}^{2n}   \left| \alpha _{j} \right| \int  _{\mathcal{X}} \left( \sum_{q \geq m_{1}+1} b_{q}^{(j,l)}(\rho) \p_{q}(x) \right)^{2} \, d \rho(x)  \right]  \, d\pg(\rho) \notag\\ & \leq \int  _{\Bb} \| \alpha \|_{1}^{2}  \sum_{l=1}^{d} \left[  \max_{1 \leq j\leq 2n} \int  _{\mathcal{X}} \left( \sum_{q \geq m_{1}+1} b_{q}^{(j,l)}(\rho) \p_{q}(x) \right)^{2} \, d \rho(x)  \right]\, d \pg(\rho) \notag\\ &=: \int _{\Bb} \| \alpha \|_{1}^{2} \sum_{l=1}^{d} \Big(\max_{1\leq j\leq 2n} \mathcal{E}_{j,l}(\rho)\Big) \, d\pg(\rho). \label{place1} 
\end{align}
If $1< \gamma<\infty$, for each pair $(j,l) \in \{ 1,\dots,n \}\times \{ 1,\dots,d \}$, we obtain that 
\begin{align*}\mathcal{E}_{j,l}(\rho)&= \int  _{\mathcal{X}} \left( \sum_{q \geq m_{1}+1} b_{q}^{(j,l)}(\rho)\p_{q}(x) \right)^{2} \, d \rho(x) \\&=\int  _{\mathcal{X}} \left( \sum_{q \geq m_{1}+1} b_{q}^{(j,l)}(\rho)\p_{q}(x) \right)^{2} (\w(x)) d\rk(x) \\ &\leq \| \w \|_{L^{\gamma}(\rk)} \left[ \int  _{\mathcal{X}} \left( \sum_{q \geq m_{1}+1} b_{q}^{(j,l)}(\rho)\p_{q}(x) \right)^{\frac{2\gamma}{\gamma-1}}  \, d \rk(x) \right]^{\frac{\gamma-1}{\gamma}}  \\ & \leq \| \w \|_{L^{\gamma}(\rk)} \left[ \int  _{\mathcal{X}} \left( \sum_{q\geq m_{1}+1} b_{q}^{(j,l)}(\rho) \p_{q}(x)  \right)^{2} \, d\rk(x)  \right]^{\frac{\gamma-1}{\gamma}} \left\| \sum_{q\geq m_{1}+1} b_{q}^{(j,l)}(\rho)\p_{q} \right\|_{\infty}^{\frac{2}{\gamma}} 
\end{align*}
\begin{align*}
& \leq \| \w \|_{L^{\gamma}(\rk)} \left\| \sum_{q \geq m_{1}+1} b_{q}^{(j,l)}(\rho) \p_{q} \right\|_{L^{2}(\rk)}^{\frac{2(\gamma-1)}{\gamma}} \left\| \sum_{q \geq m_{1}+1} b_{q}^{(j,l)}(\rho)\p_{q} \right\|_{\h} ^{\frac{2}{\gamma}} \\ & =  \| \w \|_{L^{\gamma}(\rk)} \left\| \sum_{q \geq m_{1}+1} b_{q}^{(j,l)}(\rho) \p_{q} \right\|_{L^{2}(\rk)} ^{\frac{2(\gamma-1)}{\gamma}} \left( \sum_{q \geq m_{1}+1} (b_{q}^{(j,l)}(\rho))^{2} \right)^{\frac{1}{\gamma}}.   
\end{align*} 
Insert the above estimation back into (\ref{place1}). It can be derived by the approximation result in Appendix \ref{appendix1} that \begin{align*}&\left\| \mathcal{N}_{2n} (\il(\cdot),1) - \Psi _{2n,m_{1}} \right\|^{2} _{L^{2}(\ngg)} \\ \leq & \| \alpha \|_{1}^{2} \int  _{\Bb} \| \w \|_{L^{\gamma}(\rk)} \sum_{l=1}^{d} \left[ \max_{1 \leq j \leq 2n} \left\| \sum_{q \geq m_{1}+1} b_{q}^{(j,l)}(\rho) \p_{q} \right\|_{L^{2}(\rk)} ^{\frac{2(\gamma-1)}{\gamma}} \left( \sum_{q \geq m_{1}+1} (b_{q}^{(j,l)}(\rho))^{2} \right)^{\frac{1}{\gamma}}  \right]\, d \pg(\rho) \\ \leq & 4C_{F}^{2} \int _{\Bb} \| \w \|_{L^{\gamma}(\rk)}  \sum_{l=1}^{d} \left[ \max_{1 \leq j \leq 2n} \left( \left( \sum_{q \in \mathbb{N}} (b_{q}^{(j,l)}(\rho))^{2} \right)^{\frac{1}{2}}C_{\kappa,\xi}m_{1}^{-\xi} \right)^{\frac{2(\gamma-1)}{\gamma}} \left( \sum_{q \in \mathbb{N}} (b_{q}^{(j,l)}(\rho))^{2} \right)^{\frac{1}{\gamma}}   \right]\, d \pg(\rho) \\=& 4C_{F}^{2} \int  _{\Bb} \| \w \|_{L^{\gamma}(\rk)} \sum_{l=1}^{d} \left[ \max_{1\leq j \leq 2n} \left(  \sum_{q \in\mathbb{N}} (b_{q}^{(j,l)}(\rho)) ^{2} \right)  \right]C_{\kappa,\xi,\gamma} m_{1} ^{-\xi \cdot\frac{ 2(\gamma-1)}{\gamma}}  \, d\pg(\rho) \\ \leq & 4C_{F}^{2} \int _{\Bb} \| \w \|_{L^{\gamma}(\rk)} d C_{\mathcal{B}}C_{\kappa,\xi,\gamma} m_{1}^{-\xi \cdot \frac{2(\gamma-1)}{\gamma}}  \, d\pg(\rho) \leq  C_{1} m_{1} ^{- \xi \cdot \frac{2(\gamma-1)}{\gamma}} 
\end{align*} 
where $C_{F}=\| F \|_{\mathcal{F}_{1}}$ and $C_{1}= 4dC_{F}^{2}C_{\mathcal{B}}C_{\kappa,\xi,\gamma}C_{\mathcal{G}}^{\frac{1}{2}}$.

\vskip 0.1in 

For $\gamma=\infty$, it's easy to obtain that $$\mathcal{E}_{j,l}(\rho)\leq \| \w \|_{L^{\infty}(\rk)} \| \sum_{q \geq m_{1}+1} b_{q}^{(j,l)}(\rho) \p_{q}\|_{L^{2}(\rk)}^{2}$$ and then $$\| \mathcal{N}_{2n} -\Psi _{2n,m_{1}} \|_{L^{2}(\pi)}^{2} \leq C_{1}m_{1}^{-2\xi}.$$

Next we perform a truncation on the index $p$: we let $$\Psi _{2n,m_{1},m_{2}}(\rho,x):= \sum_{j=1}^{2n}\alpha _{j}\sigma\left( \sum _{l=1}^{d}\sum_{q=1}^{m_{1}}\sum_{p=1}^{m_{2}}\sum_{s=1}^{d} a _{p,q,s}^{(j,l)}\p_{q}(x)\int \p_{p}(y) (e _{l}e _{s}^T)yd\rho(y) +b_{j} \right)+b_{0}.$$ and $a_{p,s}^{(j,l)}(x):= \sum_{q=1}^{m_{1}}a_{p,q,s}^{(j,l)}\p_{q}(x)$. Then we have 

\begin{align}& \left\| \Psi _{2n,m_{1}} - \Psi _{2n,m_{1}m_{2}} \right\|_{L^{2}(\ngg)}^{2} =  \int _{\ob} \| \Psi _{2n,m_{1}}(\rho,x)-\Psi _{2n,m_{1},m_{2}}(\rho,x) \|_{2}^{2}  \, d \ngg(\rho,x) \notag \\ \leq & \int  _{\ob} \left\| \sum_{j=1}^{2n} \left|  \alpha _{j} \right| \sum_{l=1}^{d} \left|  \sum_{q=1}^{m_{1}} \sum_{p\geq m_{2}+1} \sum_{s=1}^{d} a_{p,q,s}^{(j,l)}\p_{q}(x)\int  \p _{p}(y) e _{s}^{T}y \, d\rho(y)  \right|e  _{l}   \right\|_{2}^{2}  \, d \ngg (\rho ,x)\notag  \\ =& \int  _{\ob} \sum_{l=1}^{d} \left( \sum_{j=1}^{2n} \left|  \alpha _{j} \right|  \left|  \sum_{q=1}^{m_{1}} \sum_{p\geq m_{2}+1} \sum_{s=1}^{d} a_{p,q,s}^{(j,l)}\p_{q}(x)\int  \p _{p}(y) e _{s}^{T}y \, d\rho(y)  \right|   \right)^{2}  \, d \ngg (\rho ,x) \notag \\=: & \int  _{\ob} \sum_{l=1}^{d} \left( \sum_{j=1}^{2n} \left|  \alpha _{j} \right|  \left|  \sum_{p\geq m_{2}+1} \sum_{s=1}^{d} a_{p,s}^{(j,l)}(x)\int  \p _{p}(y) e _{s}^{T}y \, d\rho(y)  \right|   \right)^{2}  \, d \ngg (\rho ,x)  \, \notag  \\ \leq & \int  _{\ob} \| \alpha \|_{1} \sum_{l=1}^{d} \sum_{j=1}^{2n} \left|  \alpha _{j} \right|  \left( \sum_{p\geq m_{2}+1} \sum_{s=1}^{d} a_{p,s}^{(j,l)}(x)\int  \p _{p}(y) e _{s}^{T}y \, d\rho(y)   \right)^{2} \, d\ngg(\rho,x) \notag \\ =: & \int  _{\ob} \| \alpha \|_{1} \sum_{l=1}^{d} \sum_{j=1}^{2n} \left| \alpha _{j} \right|   \mathcal{E}'_{j,l}(\rho,x) \, d\ngg(\rho,x) . \label{place2}
\end{align}
Similarly, by dominated convergence theorem and the integral shift to the gaussian distribution $\rk$, we obtain 
\begin{align*}
\E_{j,l}'(\rho,x)&=\left( \int  \sum_{s=1}^{d} (\es^{T}y) \left( \sum_{p\geq m_{2}+1} a_{p,s}^{(j,l)}(x) \p_{p}(y) \right) \, d \rho(y)  \right)^{2} \\ & = \left( \int  \sum_{s=1}^{d} (\es^{T}y) \at \, d\rho(y)  \right)^{2} \leq \left( \int  \|y\|_{2} \left( \sum_{s=1}^{d} \at^{2}\right)^{\frac{1}{2}} \, d \rho(y)  \right)^{2} \\& \leq \mathbb{E}_{\rho}\|Y\|_{2}^{2} \int  \sum_{s=1}^{d} \at^{2} \, d\rho(y) \leq \bc \sum_{s=1}^{d}\int \at^{2}  \w(y)\, d\rk(y) \\&\leq \bc \| \w \|_{L^{\gamma}(\rk)}  \sum_{s=1}^{d} \left[ \int  \big(\at\big)^{\frac{2\gamma}{\gamma-1}} \, d\rk(y)  \right]^{\frac{{\gamma-1}}{\gamma}} \\& \leq \bc \|\w\|_{L^{\gamma}(\rk)} \sum_{s=1}^{d} \underbrace{ \|\ad\|_{L^{2}(\rk)}^{\frac{{2(\gamma-1)}}{\gamma}} }_{\mathrm{I}_{1}^{(j,l,s)}(x)} \underbrace{ \|\ad\|_{\infty}^{\frac{2}{\gamma}} }_{ \mathrm{I}_{2}^{(j,l,s)}(x) }
\end{align*}
Then (\ref{place2}) can be further bounded by 
\begin{align*}&\bc \int _{\Bb} \|\alpha\|_{1} \|\w\|_{L^{\gamma}(\rk)}  \sum_{j=1}^{2n} \left|  \alpha _{j}  \right| \sum_{s,l=1}^{d}\left( \int _{\mathcal{X}} \mathrm{I}_{1}^{(j,l,s)}(x)\mathrm{I}_{2}^{(j,l,s)}(x) \, d \rho(x)  \right)   \,  \, d\pg(\rho)\\ \leq &\,  \bc \int  _{\Bb} \|\alpha\|_{1}^{2}\|\w\|_{L^{\gamma}(\rk)}  \max_{1\leq j\leq 2n}\sum_{s,l=1}^{d}\left( \int  _{\mathcal{X}} \mathrm{I}_{1}^{(j,l,s)}(x)\mathrm{I}_{2}^{(j,l,s)}(x) \, d \rho(x)  \right) \, d\pg(\rho)=:\Delta_{0}  \end{align*}
where the integral of  $\mathrm{I}_{1}^{(j,l,s)}(x)\mathrm{I}_{2}^{(j,l,s)}(x)$ with respect to $\rho$ can be bounded by $$\int_{\mathcal{X}} \mathrm{I}_{1}^{(j,l,s)}(x) \mathrm{I}_{2}^{(j,l,s)}(x) \, d \rho(x) \leq \left\| \mathrm{I}_{1}^{(j,l,s)} \right\|_{L^{\frac{\gamma}{\gamma-1}}(\rho)} \left\| \mathrm{I}_{2}^{(j,l,s)} \right\|_{L^{\gamma}(\rho)}.   $$
For the first term $\mathrm{I}_{1}^{(j,l,s)}$, 
\begin{align*}
\left\| \mathrm{I}_{1}^{(j,l,s)} \right\|_{L^{\frac{\gamma}{\gamma-1}}(\rho)}& = \left( \int_{\mathcal{X}} \left\| \ad \right\|_{L^{2}(\rk)}^{2}    \, d\rho(x)  \right)^{\frac{{\gamma-1}}{\gamma}} \\ &= \left( \int  \int  \left( \sum_{p\geq m_{2}+1}a_{p,s}^{(j,l)}(x) \p_{p}(y) \right)^{2} \, d\rk(y)  \, d \rho(x)  \right)^{\frac{{\gamma-1}}{\gamma}} \\ & = \left( \int \int \sum_{p\geq m_{2}+1} (a_{p,s}^{(j,l)}(x))^{2} (\p_{p}(y))^{2} \, d\rk(y)   \, d\rho(x)  \right)^{\frac{{\gamma-1}}{\gamma}} \\ &= \left( \int  \sum_{p \geq m_{2}+1}\left( \int ( a_{p,s}^{(j,l)}(x))^{2} \, d \rho(x)  \right) (\p_{p}(y) )^{2}\, d\rk(y)  \right)^{\frac{{\gamma-1}}{\gamma}} \\ &= \left\|  \sum_{p \geq m_{2}+1} \left\| a_{p,s}^{(j,l)} \right\|_{L^{2}(\rho)} \p_{p}   \right\|_{L^{2}(\rk)}^{\frac{2(\gamma-1)}{\gamma}} .
\end{align*}

Perform the domain shift to each coefficient of $\p_{p}$ again and we can obtain 
\begin{align*}
\left\| a_{p,s}^{(j,l)} \right\|_{L^{2}(\rho)} &= \int \left( \sum_{q=1}^{m_{1}}a_{p,q,s}^{(j,l)}\p_{q}(x) \right)^{2} \w(x)  \, d\rk(x)   \\ & \leq \| \w \|_{L^{\gamma}(\rk)} \left\|  \sum_{q=1}^{m_{1}}a_{p,q,s}^{(j,l)}\p_{q} \right\|_{L^{2}(\rk)}^{\frac{2(\gamma-1)}{\gamma}} \left\|  \sum_{q=1}^{m_{1}}a_{p,q,s}^{(j,l)}\p_{q}\right\|_{\h}^{\frac{2}{\gamma}}   \\ &= \| \w \|_{L^{\gamma}(\rk)} \left( \sum_{q=1}^{m_{1}} (a_{p,q,s}^{(j,l)})^{2}\rb_{q} \right)^{\frac{\gamma-1}{\gamma}} \left( \sum_{q=1}^{m_{1}} (a_{p,q,s}^{(j,l)})^{2} \right)^{\frac{1}{\gamma}} \\ &\leq \| \w \|_{L^{\gamma}(\rk)} (\rb_{1})^{\frac{\gamma-1}{\gamma}} \sum_{q=1}^{m_{1}} (a_{p,q,s}^{(j,l)})^{2}  ,
\end{align*}
which implies that $\sum_{p\geq m_{2}+1} \left\| a_{p,s}^{(j,l)} \right\|_{L^{2}(\rho)} \p_{p} \in \h$ with the RKHS norm $$\left( \sum_{p\geq m_{2}+1} \left\| a_{p,s}^{(j,l)} \right\|_{L^{2}(\rho)}^{2} \right)^{\frac{1}{2}} \leq\left(  \|\w\|_{L^{\gamma}(\rk)}(\rb_{1})^{\frac{{\gamma-1}}{\gamma}} \sum_{p\geq m_{2}+1} \sum_{q=1}^{m_{1}} (a_{p,q,s}^{(j,l)})^{2} \right)^{\frac{1}{2}}.$$

For the second term $\mathrm{I}_{2}^{(j,l,s)}$, 
\begin{align*}
\left\| \mathrm{I}_{2}^{(j,l,s)} \right\|_{L^{\gamma}(\rho)}&= \left( \int  \left\| a_{\boldsymbol{\lambda}}^{(j,l,s)}(x, \cdot) \right\|_{\infty} ^{2} \, d \rho(x)  \right)^{\frac{1}{\gamma}} \leq \left( \int  \left\| a_{\boldsymbol{\lambda}}^{(j,l,s)}(x, \cdot) \right\|_{\h}^{2}  \, d \rho(x)  \right)^{\frac{1}{\gamma}} \\ & \leq \left( \int  \left\| \sum_{p\geq m_{2}+1} a_{p,s}^{(j,l)}(x) \p_{p} \right\|_{\h}^{2}  \, d\rho(x)  \right)^{\frac{1}{\gamma}}= \left( \int  \sum_{p\geq m_{2}+1}(a_{p,s}^{(j,l)}(x))^{2} \, d\rho(x)  \right)^{\frac{1}{\gamma}} \\ & \leq \left( \int \sum_{p\geq m_{2}+1}\left( \sum_{q=1}^{m_{1}} a_{p,q,s}^{(j,l)}\p_{q}(x) \right)^{2} \,  d\rho(x)  \right)^{\frac{1}{\gamma}} \\ & \leq \left( \int  \left( \sum _{p \geq m_{2}+1}\sum_{q=1}^{m_{1}}(a_{p,q,s}^{(j,l)})^{2} \right)\left( \sum_{q=1}^{m_{1}} \p_{q}(x)^{2} \right) \, d \rho(x)  \right)^{\frac{1}{\gamma}} \\& \leq \left( \sum _{p \geq m_{2}+1}\sum_{q=1}^{m_{1}}(a_{p,q,s}^{(j,l)})^{2} \right)^{\frac{1}{\gamma}} \left( \int  \sum_{q=1}^{\infty} \p_{q}(x)^{2} \, d\rho(x)  \right)^{\frac{1}{\gamma}} \\ & = \left( \sum _{p \geq m_{2}+1}\sum_{q=1}^{m_{1}}(a_{p,q,s}^{(j,l)})^{2} \right)^{\frac{1}{\gamma}}.
\end{align*}
Combine the estimations for $\mathrm{I}_{1}^{(j,l,s)}$ and $\mathrm{I}_{2}^{(j,l,s)}$ and we get an upper bound that 

\begin{align*}
\Delta _{0} &\leq 4C_{\mathcal{B}}C_{F}^{2}\int  _{\Bb} \| \w \|_{L^{\gamma}(\rk)}  \left( \max_{1 \leq j \leq 2n} \sum_{s,l =1}^{d}\left\| \mathrm{I}_{1}^{(j,l)} \right\|_{L^{\frac{\gamma}{\gamma-1}}(\rho)} \left\| \mathrm{I}_{2}^{(j,l)} \right\|_{L^{\gamma}(\rho)}   \right)  \, d\pg(\rho) \\ \leq 4C_{\mathcal{B}}C_{F}^{2} & \int  _{\Bb} \| \w \|_{L^{\gamma}(\rk)}  \Biggl\{\max_{1 \leq j \leq 2n} \sum_{s,l =1}^{d} \left[  \left( \| \w \|_{L^{\gamma}(\rk)} (\rb_{1})^{\frac{\gamma-1}{\gamma}} \sum_{p \in \mathbb{N}} \sum_{ q=1}^{m_1} (a_{p,q,s}^{(j,l)})^{2}  \right)^{\frac{1}{2}}C_{\kappa,\xi}m_{2}^{-\xi}  \right]^{\frac{2(\gamma-1)}{\gamma}}  \\&    \quad \quad\quad\quad\quad\quad\quad\quad\quad\quad\quad\quad\quad\quad\quad\quad\quad\quad\quad\quad \cdot \left( \sum_{p \in\mathbb{N}} \sum_{q=1}^{m_1} (a_{p,q,s}^{(j,l)})^{2} \right)^{\frac{1}{\gamma}}  \Biggr\} \, d \pg(\rho) \\ & \leq 4dC_{\mathcal{B}} C_{F}^{2} (\rb_{1})^{\frac{(\gamma-1)^{2}}{\gamma^{2}}} \int _{\Bb} \| \w \|_{L^{\gamma}(\rk)} ^{\frac{2\gamma-1}{\gamma}}  C_{\kappa,\xi,\gamma} m_{2} ^{-\xi \cdot \frac{2(\gamma-1)}{\gamma}} \, d \pg(\rho) 
\end{align*}
\begin{align*}
    & \leq 4dC_{\mathcal{B}}C_{F}^{2}C_{\kappa,\xi,\gamma}(\rb_{1})^{\frac{(\gamma-1)^{2}}{\gamma^{2}}} m_{2}^{-\xi  \cdot \frac{2(\gamma-1)}{\gamma}} \int  _{\Bb} \| \w \|_{L^{\gamma}(\rk)}^{ \frac{2\gamma-1}{\gamma}}  \, d \pg(\rho) \\ & \leq C_{2}m_{2}^{- \xi \cdot \frac{2(\gamma-1)}{\gamma}}   
\end{align*}
where $C_{2}=4d C_{\mathcal{B}}C_{F}^{2}C_{\kappa,\xi,\gamma}(\rb_{1})^{\frac{(\gamma-1)^{2}}{\gamma^{2}}}C_{\mathcal{G}}^{\frac{2\gamma-1}{2\gamma}}$.

\subsubsection{Linear Transformer with Adaptive Attention Heads} \label{sec:app3}

Recall that 
\begin{align*}\Psi _{2n,m}(\rho,x)&:= \sum_{j =1}^{2n}\alpha _{j}\sigma\left( \sum _{l =1}^{d}\sum_{p,q=1}^{m}\sum_{s=1}^{d} a _{p,q,s}^{(j,l)}\p_{q}(x)\int \p_{p}(y) (e _{l}e _{s}^T)yd\rho(y) +b_{j} \right)+b_{0} \\&= \sum_{j=1 }^{2n} \alpha _{j} \sigma \left( \sum_{p,q=1}^{m} \p_{q}(x) \int  \p_{p}(y) A_{p,q}^{(j)}y \, d \rho(y) +b_{j}  \right)+b_{0}
\end{align*} with $A_{p,q}^{(j)}= \sum_{l=1}^{d}\sum_{s=1}^{d} a_{p,q,s}^{(j,l)}e _{l}\es ^{T}=[a_{p,q,s}^{(j,l)}]_{1\leq l \leq d,1 \leq s \leq d}$.

\vskip 0.1in 

Now we approximate each $\p_{q}$ with a neural network $\phi _{q}$, and let 
\begin{align*}
    \mathrm{T}_{2n,m}(\rho,x):= \sum_{j=1}^{2n}\alpha _{j} \sigma\left( \sum_{q =1}^{m} \phi _{q}(x) \left( \sum_{p=1}^{m} \int  \phi _{p}(y) A_{p,q}^{(j)}y \, d \rho(y)  \right) +b_{j} \right)+b_{0}.
\end{align*}

Then we have the error decomposition: $$\left\| \Psi _{2n,m} - \mathrm{T}_{2n,m} \right\|_{L^{2}(\ngg)} \leq \| \Psi _{2n,m}- \widetilde{\mathrm{T}}_{2n,m} \|_{L^{2}(\ngg)} + \| \widetilde{\mathrm{T}}_{2n,m} - \mathrm{T}_{2n,m} \|_{L^{2}(\ngg)}   $$ where $$\widetilde{\mathrm{T}}_{2n,m}(\rho,x)= \sum_{j=1}^{2n} \alpha _{j} \sigma\left( \sum_{q=1}^{m} \phi _{q}(x) \left( \sum_{p=1}^{m} \int  \p_{p}(y) A_{p,q}^{(j)}y \, d \rho(y) \right) +b_{j}\right)+b_{0}.$$

\vskip 0.1in 

Similar with error estimations for the truncated error, let $$\bq:= \sum_{1\leq p \leq m,1\leq s \leq d} a_{p,q,s}^{(j,l)}\int  \p_{p}(y) \es ^{T}y \, d \rho(y) $$ and we have 
\begin{align*}&\left\| \Psi _{2n,m}- \widetilde{\mathrm{T}}_{2n,m} \right\|_{L^{2}(\ngg)}^{2} \\ & \leq \int  _{\ob} \left\| \sum_{j =1}^{2n} |\alpha _{j}| \sum_{l=1}^{d} \left|  \sum_{q =1}^{m} \bq \big(\p_{q}(x)-\phi _{q}(x)\big) \right| e _{l}  \right\|_{2}^{2}  \, d \ngg(\rho,x) \\ & = \int  _{\ob} \sum_{l =1}^{d}\left( \sum_{j=1}^{2n} \left|  \alpha _{j} \right| \left|  \sum_{q =1}^{m} \bq \big(\p_{q}(x)-\phi _{q}(x)\big)\right|  \right)^{2} \, d \ngg(\rho,x) \\ & \leq \int  _{\Bb} \| \alpha \|_{1} \sum_{l=1}^{d} \sum_{j=1}^{2n} \left| \alpha _{j} \right| \int_{\mathcal{X}}  \left( \sum_{q \in[m]}\bq \big(\p_{q}(x)-\phi _{q}(x)\big) \right)^{2}  \, d \rho(x) d \pg(\rho) \\ & \leq \int  _{\Bb} \| \alpha \|_{1}^{2} \sum_{l=1}^{d} \max_{1\leq j \leq 2n} \left( \sum_{q=1}^{m}(\bq)^{2}  \right) \int_{\mathcal{X}} \sum_{q=1}^{m} \big( \p_{q}(x)-\phi _{q}(x) \big)^{2}\, d \rho(x) d \pg(\rho) \\ & \leq d \| \alpha \|_{1}^{2} C_{\mathcal{B}} \int  _{\Bb}  \sum_{q=1}^{m}\int_{\mathcal{X}} \big( \p_{q}(x)-\phi _{q}(x) \big)^{2}\,    \, d \rho(x)   \, d \pg(\rho) \\ &=: dC_{\mathcal{B}} \| \alpha \|_{1}^{2} \sum_{q=1}^{m}\int  _{\Bb} \mathcal{\widetilde{E}}_{q}(\rho) \, d \pg(\rho).       
\end{align*}
Take $\phi _{q}= \phi_{q,\tm}$ to be the two-hidden-layer tanh neural network with a product gate in Appendix \ref{appendix2}. Then we have $$\widetilde{\mathcal{E}}_{q}(\rho) \leq (4d^{2}+C_{\ka,\gamma}\| \w \|_{L^{\gamma}(\rk)} ) \exp (-2\tm \log \tm).$$ Take the estimation back and it follows that $$\left\| \Psi _{2n,m} - \widetilde{\mathrm{T}}_{2n,m,\tm} \right\|_{L^{2}(\ngg)}^{2} \leq C_{3}m \exp(-2\tm\log\tm)$$with $C_{3}=4dC_{\mathcal{B}}C_{F}^{2}\left( 4d^{2}+C_{\ka,\gamma}C_{\mathcal{G}}^{\frac{1}{2}} \right)$ for $\tm > C_{\ka,\theta,\gamma}$.

Then we use the same group of two-hidden-layer neural networks $\{ \phi _{q,\tm} \}_{1\leq q \le m}$ to approximate $\{ \p_{q} \}_{1\leq q\leq m}$ in the context memory part. Now Let a linear Transformer $$\mathrm{T}_{2n,m,\tm}(\rho,x)=\sum_{j=1}^{2n} \alpha _{j} \sigma\left( \sum_{q=1}^{m} \phi _{q, \tm}(x) \left( \sum_{p=1}^{m}\int \phi _{p,\tm}(y)A_{p,q}^{(j)}y   \, d \rho(y) \right)+b_{j} \right)+b_{0}.$$
Then the approximation error for context functions can be measured as \begin{align*}
& \, \,  \, \,  \, \,  \, \,  \, \,  \, \, \left\| \widetilde{\mathrm{T}}_{2n,m,\tm}- \mathrm{T}_{2n,m,\tm} \right\|_{L^{2}(\ngg)}^{2} \\ &\leq \int _{\Bb} \| \alpha \|_{1}^{2} \sum_{l=1}^{d} \max_{1\leq j\leq 2n} \int_{\mathcal{X}} \left( \sum_{1 \leq p,q \leq m,1 \leq s \leq d} a_{p,q,s}^{(j,l)} \phi _{q,\tm}(x) \int  (\p_{p}(y)-\phi _{p,\tm}(y)) \es ^{T}y \, d \rho(y)  \right)^{2}   \, d\rho(x)d\pg(\rho) \\ &\leq \int  _{\Bb}\| \alpha \|_{1}^{2} \sum_{l=1}^{d} \max_{1\leq j \leq 2n}\left( \int_{\mathcal{X}} \sum_{1 \leq p \leq m, 1 \leq s \leq d} (\tilde{a}_{p,s}^{(j,l)}(x))^{2} d\rho(x) \right) \cdot\\& \, \,  \, \,   \,  \, \, \,  \, \, \,  \, \, \,  \, \, \,  \, \, \,  \, \, \,  \, \, \,  \, \, \,  \, \, \,  \, \, \,  \, \, \,  \, \, \,  \, \, \,  \, \, \,  \, \, \,  \, \, \,  \, \,\, \,  \left( \sum_{1\leq p \leq m, 1\leq s \leq d} \left(  \int  (\p_{p}(y)-\phi _{p, \tm}(y))\es ^{T}y\, d \rho(y)  \right)^{2}  \right) \, d \pg(\rho)\, 
\end{align*}
where $$\tilde{a}_{p,s}^{(j,l)}(x):= \sum_{q=1}^{m} a_{p,q,s}^{(j,l)}\phi _{q,\tm}(x).$$

For the first factor, let $$\mathrm{I}^{(j,l)}_{3}(\rho)= \int  _{\mathcal{X}} \sum_{1\leq p\leq m,1\leq s\leq d} (\tilde{a}_{p,s}^{(j,l)}(x))^{2} \, d \rho(x),$$ and we can obtain 
\begin{align*}
\sum_{l =1} ^{d} \max_{1\leq j \leq 2n} \mathrm{I}^{(j,l)} _{3}(\rho) &\leq d \sum_{q =1}^{m}\int_{\mathcal{X}} (\phi _{q,\tm}(x)  )^{2}\, d \rho(x)   \leq d\sum_{q =1}^{m} (2\| \p_{q} \|_{L^{2}(\rho)}^{2} + 2\| \p_{q}-\phi _{q,\tm} \|_{L^{2}(\rho)} ^{2})\\ &\leq 2d + 2dm (4d^{2}+C_{\ka,\gamma}\| \w \|_{L^{\gamma}(\rk)} ) \exp (-2\tm \log \tm).
\end{align*}
For the second factor, it's easy to observe that 
\begin{align*}
&\sum_{1\leq p \leq m, 1\leq s \leq d} \left(  \int  (\p_{p}(y)-\phi _{p, \tm}(y))\es ^{T}y\, d \rho(y)  \right)^{2} \leq \sum_{p=1}^{m} \| \p_{p}-\phi _{q,\tm} \|^{2}_{L^{2}(\rho)} \mathbb{E}_{X \sim\rho}\| X \|_{2}^{2} \\ & \leq C_{\mathcal{B}} m  (4d^{2}+C_{\ka,\gamma}\| \w \|_{L^{\gamma}(\rk)} ) \exp (-2\tm \log \tm).
\end{align*}
Choose $\tm$ such that $m\exp(-2\tm\log\tm)<1$. Combine two estimations and it is obtained that 
\begin{align*}
&\, \,\, \,\, \,\, \,\, \,\left\| \tilde{\mathrm{T}}_{n,m,\tm} - \mathrm{T}_{n,m, \tm} \right\|_{L^{2}(\ngg)} ^{2} \\&\leq 8C_{F}^{2}C_{\mathcal{B}}d\int  _{\bx} (20d^{4}+9d^{2}C_{\ka,\gamma}\| \w \|_{L^{\gamma}(\rk)}+ C_{\ka,\gamma}^{2}\| \w \|_{L^{\gamma}(\rk)}^{2}  )m \exp(-2\tm\log\tm) \, d\pg(\rho) \\ &\leq  C_{4} m \exp(-2\tm \log\tm)
\end{align*}
with $C_{4}=8C_{F}^{2}C_{\mathcal{B}}d(20d^{4}+9d^{2}(1+C_{\ka,\gamma}C_{\mathcal{G}})^{2}).$

Combine all the estimations and we can achieve the following convergence rate for $n \geq C'_{\ka,\theta,\gamma}$ with $C'_{\ka,\theta,\gamma}=\exp\Big(\frac{4(\gamma-1)\xi C_{\ka,\theta,\gamma}}{2(\gamma-1)\xi+\gamma}\Big)$:
\begin{align*}
&\left\| F(\il(\cdot),1)- \mathrm{T}_{2n,m,\tm} \right\|_{L^{2}(\ngg)} \\ \leq &\, \| F(\il(\cdot),1)-\mathcal{N}_{2n}(\il(\cdot),1) \|_{L^{2}(\ngg)}+ \| \mathcal{N}_{2n}(\il(\cdot),1)- \Psi _{2n,m} \|_{L^{2}(\ngg)}+ \left\| \Psi _{2n,m}- \mathrm{T}_{2n,m,\tm} \right\|_{L^{2}(\ngg)} \\  \leq & \, ((1+\bc) C_{F}^{2})^{\frac{1}{2}} n^{-\frac{1}{2}} + \left( C_{1}^{\frac{1}{2}} + C_{2}^{\frac{1}{2}} \right) m^{-\frac{\xi (\gamma-1)}{\gamma}}+ \left( C_{3}^{\frac{1}{2}}+C_{4}^{\frac{1}{2}}  \right)m^{\frac{1}{2}} \exp(-\tm\log\tm) \leq C_{*} n^{-\frac{1}{2}} ,
\end{align*}
with $C_{*}=3 \max \left\{  (1+\bc)^{\frac{1}{2}}C_{F}, C_{1}^{\frac{1}{2}}+C_{2}^{\frac{1}{2}}, C_{3}^{\frac{1}{2}}+C_{4}^{\frac{1}{2}}  \right\}$, 
\begin{align*}
    m=\left\lceil{n^{\frac{\gamma}{2(\gamma-1)\xi}}}\right\rceil \text{ and } \tm=\left\lceil{\left( \frac{1}{2}+\frac{\gamma}{4(\gamma-1)\xi} \right) \log n}\right\rceil.
\end{align*} \qedblack

\subsection{Oracle Inequality: Sampling Error for Linear Transformers} \label{sec:oracle}

In this Subsection, we derive an oracle inequality for the two-stage sampling estimation.  We first prove the compactness of the hypothesis space in \ref{sec:compact}, which guarantees the existence of $\mathrm{T}_{\mathbb{S},n}$ in \eqref{eq:minimizer}. We then establish covering number estimates in \ref{sec:covering} to bound the empirical processes of the first-stage sampling \eqref{error:first}, the pseudo second-stage sampling \eqref{error: gtc} and the second-stage sampling \eqref{error:acc}.

\subsubsection[Compact subspaces in C(Omega)]%
        {Compact subspaces in \texorpdfstring{$C(\Omega)$}{C(Omega)}} \label{sec:compact}

Recall that $(\Omega, d_{\Omega})$ is a complete separable metric space. To prove that $\mathcal{H}_{\mathrm{T}_{n}}$ is compact in $\co$, it's sufficient to show that $\mathcal{H}_{\mathrm{T}_{n}}$ is sequentially compact in $\co$. By Arzelà-Ascoli Theorem, it's sufficient to check the equi-boundedness and equi-continuity of $\mathcal{H}_{\mathrm{T}_{n}}$. For any $\mathrm{T}_{n} \in \htt$, we can pick a group of parameters $\Big((\alpha _{j}), (b_{j}), (A_{p,q}^{(j)}), \Theta _{\tanh}\Big)$ satisfying the conditions of the hypothesis space $\htt$.
For the simplicity, we let $$\sigma _{j}(\rho,x)= \sigma\left( \sum_{q=1}^{m(n)} {{\phi _{q,\tm(n)}(x)}} {{\left( \sum_{p=1}^{m(n)}\mathcal{T}_{C_{\mathcal{B}}}\left[ \int \phi _{p,\tm(n)}(y) A_{p,q}^{(j)} y \, d \rho(y) \right]  \right)}}+b_{j}\right).$$ Then we have $$\begin{aligned}&\| \mathrm{T}_{n}(\rho,x)- \mathrm{T}_{n}(\rho',x') \|_{2} = \left\| \sum_{j=1}^{n} \alpha _{j}(\sigma _{j}(\rho,x)-\sigma _{j}(\rho',x')) \right\|_{2} \leq \| \alpha \|_{1} \max _{1\leq j \leq n} \| \sigma _{j}(\rho,x)-\sigma _{j}(\rho',x') \|_{2} \\ \leq & \, \| \alpha \|_{1} \max _{1 \leq j \leq n} \sum_{p,q=1}^{m(n)}\left\|  \phi _{q,\tm(n)}(x) \mathcal{T}_{C_{\mathcal{B}}}\left[ \int  \phi _{p,\tm(n)}(y) A_{p,q}^{(j)}y \, d \rho(y)  \right]- \phi _{q,\tm(n)}(x') \mathcal{T}_{C_{\mathcal{B}}}\left[ \int  \phi _{p,\tm(n)}(y) A_{p,q}^{(j)}y \, d \rho'(y)  \right] \right\|_{2} \\ \leq & \, \| \alpha \|_{1} \max _{1 \leq j \leq n} \sum_{p,q=1}^{m(n)} \Biggl\| \phi _{q,\tm(n)}(x) \left( \tc\left[ \int  \phi _{p,\tm(n)}(y) A_{p,q}^{(j)} y \, d \rho(y)  \right]-\tc\left[ \int  \phi _{p,\tm(n)}(y) A_{p,q}^{(j)} y \, d \rho'(y)  \right]   \right)\\ & \,\,\,\,\,\,\,\,\,\,\,\,\,\,\,\,\,\,\,\,\,\,\,\,\,\,\,\,\,\,\,\,\,\,\, \,\,\,\, + \tc\left[ \int  \phi _{p,\tm(n)}(y) A_{p,q}^{(j)} y \, d \rho'(y)  \right](\phi _{q,\tm(n)}(x)- \phi _{q,\tm(n)}(x')) \Biggr\|_{2} 
\end{aligned}$$

$$\begin{aligned}
\leq & \, \|  \alpha \| _{1} \max_{1\leq j\leq n} \sum_{p,q=1}^{m(n)} \Bigg( \underbrace{ \left\|   \int  (\phi _{p,\tm(n)}(y) A_{p,q}^{(j)}y - \phi _{p,\tm(n)}(y') A_{p,q}^{(j)}y') \, d \rho(y)d\rho'(y')   \right\|_{2} }_{ \Delta(\rho,\rho') } \\ &\quad \quad\quad\quad\quad\quad\quad\quad+ \sqrt{ d }C_{\mathcal{B}} \left|  \phi _{q,\tm(n)}(x)-\phi _{q,\tm(n)}(x') \right|  \Bigg)  . \end{aligned}$$
Recall that $\phi _{q,\tm(n)}(x)= \prod_{l=1}^{d} \mathcal{T}_{1}(\phi _{q,\tm(n)}^{(l)}(x))$, where $(\phi _{q,\tm(n)}^{(l)})_{l=1}^{d}$ is a group of two-hidden-layer tanh neural networks shown in Appendix \ref{appendix2}. It's easy to observe that $$\begin{aligned} &\left|  \phi _{q, \tm(n)}(x) - \phi _{q,\tm(n)}(x') \right| \\ = & \left|  \prod _{l=1}^{d} \mathcal{T}_{1}(\phi _{q, \tm(n)}^{(l)}(x)) - \prod _{l=1}^{d} \mathcal{T}_{1}(\phi _{q,\tm(n)}^{(l)}(x')) \right| \leq d \max_{1 \leq l \leq d} \left|  \phi^{(l)}_{q,\tm(n)}(x) - \phi _{q,\tm(n)}^{(l)}(x') \right|  \\ \leq &d \max_{1 \leq l \leq d}  \left|  c_{q,l} ^{T} [\sigma _{\tanh}(W_{q,1}^{(l)}\sigma _{\tanh}(W_{q,0}^{(l)}x+b_{q,0}^{(l)})+b_{q,1}^{(l)})- \sigma _{\tanh}(W_{q,1}^{(l)}\sigma _{\tanh}(W_{q,0}^{(l)}x'+b_{q,0}^{(l)})+b_{q,1}^{(l)})] \right|  \\ \leq & d \max _{ 1 \leq l \leq d} \| c_{q,l}^{T} \|_{\infty}  \left\| W_{q,1}^{(l)} (\sigma _{\tanh}(W_{q,0}^{(l)}x+b_{q,0}^{(l)})-\sigma _{\tanh}(W_{q,0}^{(l)}x'+b_{q,0}^{(l)})) \right\| _{\infty} \\ \leq & d\max _{ 1 \leq l \leq d} \| c_{q,l}^{T} \|_{\infty} \left\| W_{q,1}^{(l)} \right\|_{\infty}   \left\| W_{q,0}^{(l)} \right\|_{\infty} \| x-x' \|_{\infty}  \leq \Big(d\max_{1 \leq l \leq d} \| c_{q,l}^{T} \|_{\infty} \left\| W_{q,1}^{(l)} \right\|_{\infty}\left\| W_{q,0}^{(l)} \right\|_{\infty}    \Big) \| x-x' \|_{2} \\ \leq& C_{n,d} \| x-x' \|_{2}.\end{aligned}$$ Since the parameters in tanh neural networks are uniformly bounded by $\| \Theta _{\tanh} \|_{\infty}$, $C_{n,d}$ just depends on $n$ and $d$. It also follows that for any $\tau \in \prod(\rho, \rho')$, $$\begin{aligned}\Delta(\rho, \rho')&\leq \int  \left\| \phi _{p,\tm(n)}(y) A_{p,q}^{(j)} y - \phi _{p,\tm(n)}(y') A_{p,q}^{(j)}y'\right\|_{2}  \, d \tau(y,y') \\ &\leq \left\| A_{p,q}^{(j)} \right\|_{F} \int \| \phi _{p,\tm(n)}(y)y- \phi _{p,\tm(n)}(y')y' \|_{2}   \, d \tau(y,y')  \\ & \leq \left\| A_{p,q}^{(j)} \right\|_{F} \left( \int  \left\| \phi _{p,\tm(n)}(y) (y-y') \right\|_{2}  \, d \tau (y,y') + \int  \left\| y' (\phi _{p,\tm(n)}(y)- \phi _{p,\tm(n)}(y')) \right\|_{2}  \, d \tau(y,y')  \right)\\ & \leq \left\| A_{p,q}^{(j)} \right\|_{F} \left( \int  \| y-y' \|_{2}  \, d \tau(y,y') + C_{n,d}\sqrt{ \mathbb{E}_{\rho'}\left\| Y' \right\|_{2}^{2}  }  \sqrt{ \mathbb{E}_{\tau} \| Y-Y' \|_{2}^{2}  } \right) \\ & \leq \left\| A_{p,q}^{(j)} \right\|_{F} (1+ C_{n,d} \sqrt{ \mathbb{E}_{\rho'} \| Y' \|_{2}^{2}  }) \sqrt{ \mathbb{E}_{\tau} \| Y-Y' \|_{2}^{2}  } .  \end{aligned}$$
Take the above estimation back and it can be obtained that $$\begin{aligned}&\left\| \mathrm{T}_{n}(\rho,x)- \mathrm{T}_{n}(\rho',x') \right\|_{2} \\\leq & \,  \| \alpha \|_{1} \max_{1 \leq j \leq n} \sum_{p,q=1}^{m(n)} \Big(\left\| A_{p,q}^{(j)} \right\|_{F} (1+C_{n,d}\| Y \|_{L^{2}(\rho')} ) \| Y-Y' \|_{L^{2}(\tau)} + \sqrt{ d }C_{\mathcal{B}} C_{n,d} \| x-x' \|_{2}   \Big)\\ \leq &\, \| \alpha \|_{1} m\sqrt{ d } (1+ C_{n,d}\| Y \|_{L^{2}(\rho')} +\sqrt{ d }C_{\mathcal{B}}C_{n,d} )(\| Y-Y' \|_{L^{2}(\tau)}+ \| x-x' \|_{2}  )\end{aligned}$$ which holds for any $\tau \in \prod(\rho,\rho')$. It follows that $$\|  \mathrm{T}_{n}(\rho,x) - \mathrm{T}_{n}(\rho',x') \|_{2} \leq C_{F}m\sqrt{ d }(1+C_{n,d}\| Y \|_{L^{2}(\rho')}+\sqrt{ d }C_{\mathcal{B}}C_{n,d})d_{\Omega}((\rho,x), (\rho',x'))$$ and proves that $\htt$ is equi-continuous at each $(\rho',x') \in \Omega$.

For any $(\rho,x) \in \Omega$, we have $$\begin{aligned}
\| \mathrm{T}_{n}(\rho,x) \|_{2} &\leq \| b_{0} \|_{2}+ \| \alpha \|_{1} \max_{1 \leq j \leq n} \left( \| b_{j} \|_{2}+ \left\| \sum_{p,q=1}^{m(n)} \phi _{q, \tm(n)}(x) \tc \left[ \int \phi _{p,\tm(n)}(y) A_{p,q}^{(j) }y \, d \rho(y)  \right] \right\|_{2}   \right) \\ &\leq  C_{F}\sqrt{ d\bc }+C_{F}  \left(  \sqrt{ 2d\bc } + \sum_{p,q=1}^{m(n)} \left\| \tc\left[ \int  \phi _{p, \tm(n)}(y) A_{p,q}^{(j)}y \, d \rho(y)  \right] \right\|_{2}   \right) \\& \leq C_{F}(2\sqrt{ d\bc } + m^{2}\bc \sqrt{ d }).
\end{aligned}$$ which shows that $\{ \mathrm{T}_{n }(\rho,x): \mathrm{T}_{n} \in \htt \}$ is bounded for each $(\rho,x )\in \Omega$. Therefore, $\htt$ is compact in $C(\Omega)$ which guarantees the existence of $\mathrm{T}_{\mathbb{S},n}$ and the below covering number.

\subsubsection{Covering Number Estimations}\label{sec:covering}

Note that in the first stage and pseudo second stage sampling, we have access to the true distributions sampled from $\mathcal{P}_{\mathcal{G}}$ on $\Bb$ such that the second moments are uniformly bounded by $\bc$, which however doesn't hold true for the empirical distributions in the second stage sampling.

Recall that the hypothesis space is defined as $$\begin{aligned}\mathcal{H}_{\mathrm{T}_{n}}=\Biggl\{ \mathrm{T}_{n}: \| \alpha \|_{1} \leq 2C_{F} , \sum_{p,q =1}^{m(n)}\| A_{p,q}^{(j)} \|_{F}^{2} \leq d,  \| b_{j} \|_{2} \leq \sqrt{ 2d\bc } \text{ for each }1\leq j\leq n, \\ \| b_{0} \|_{2} \leq C_{F}\sqrt{ 2d\bc }, \text{ and }   \left\| \Theta _{\tanh} \right\|_{\infty} \leq c_{1}(c'_{2}\log(n))^{c'_{3}(\log n)^{2}}         \Biggr\}.\end{aligned}$$
For $\sigma _{j}(\rho,x)$ defined above, it's easy to observe that for any $(\rho,x) \in \Omega _{\mathcal{B}}$, 
$$\begin{aligned}\| \sigma _{j}(\rho,x) \|_{2}  &\leq \left\| b_{j} + \sum_{p,q=1}^{m(n)} \phi _{q,\tm(n)}(x) \tc \left[  \int  \phi _{p,\tm(n)} (y) A_{p,q}^{(j)}y \, d \rho(y)  \right]  \right\|_{2} \\&\leq \| b_{j} \|_{2}+ \sum_{p,q=1}^{m(n)} \int   \left\| A_{p,q}^{(j)}y \right\|_{2}  d \rho(y)  \\ & \leq \| b_{j} \|_{2} + \bc^{\frac{1}{2}}\sum_{p,q=1}^{m(n)} \left\| A_{p,q}^{(j)} \right\|_{F}\\ &\leq \sqrt{ 2d\bc } + \sqrt{ \bc }m\sqrt{ d }\leq m\sqrt{2 d \bc }. 
\end{aligned}$$ 
Choose $\mathrm{T}_{n}, \mathrm{ \bar{T}}_{n} \in \ho$ with $\| \alpha-\bar{\alpha} \|_{1} \leq \epsilon$, $\left( \sum_{p,q=1}^{m(n)}\left\| A_{p,q}^{(j)}- \bar{A}_{p,q}^{(j)} \right\|_{F}^{2} \right)^{\frac{1}{2}}\leq \epsilon$, $\| b_{j}-\bar{b}_{j} \|_{2}  \leq \epsilon$ for $1 \leq j \leq n$, and $\| b_{0}-\bar{b}_{0} \|_{2} \leq \epsilon$. Then we have 
$$\begin{aligned}\left\| \mathrm{T}_{n}(\rho,x)- \bar{\mathrm{T}}_{n}(\rho,x) \right\|_{2} &= \left\|(b_{0}-\bar{b}_{0})+ \sum_{j=1}^{n} (\alpha _{j} \sigma _{j}(\rho,x) - \bar{\alpha}_{j}\bar{\sigma}_{j}(\rho,x))  \right\| _{2} \\&\leq  \| b_{0} -\bar{b}_{0}\|_{2} +  \left\| \sum_{j=1}^{n} (\alpha _{j}-\bar{\alpha}_{j})\sigma _{j}(\rho,x)+\bar{\alpha}_{j}(\sigma _{j}(\rho,x)-\bar{\sigma}_{j}(\rho,x)) \right\|_{2}  \\ & \leq \epsilon +\| \alpha-\bar{\alpha} \|_{1} \max_{1 \leq j \leq n} \| \sigma _{j}(\rho,x) \|_{2} + \| \bar{\alpha} \|_{1} \max_{1\leq j \leq n} \| \sigma _{j}(\rho,x)- \bar{\sigma}_{j}(\rho,x) \|_{2} \\ & \leq 2\sqrt{ d\bc } m\epsilon + 2C_{F}\max_{1 \leq j\leq n} \| \sigma _{j}(\rho ,x)- \bar{\sigma}_{j}(\rho ,x) \|_{2}.   \end{aligned}$$ For the second term in the above inequality, it follows that 
$$\begin{aligned}&\| \sigma _{j}(\rho,x)- \bar{\sigma} _{j}(\rho ,x)\|_{2} \\\leq &\, \Biggl\|(b_{j}-\bar{b}_{j})+\sum_{p,q=1}^{m(n)} \Biggl(\phi _{q,\tm(n)}(x)\tc\left[ \int  \phi _{p,\tm(n)}(y)A_{p,q}^{(j)}y \, d \rho(y)  \right] \\& \,\,\,\,\,\,\,\,\,\,\,  - \bar{\phi}_{q,\tm(n)}(x) \tc\left[ \int \bar{\phi}_{p,\tm(n)}(y) \bar{A}_{p,q}^{(j)}y \, d \rho(y)  \right] \Biggr) \Biggr\|_{2}  \\\leq & \, \| b_{j} - \bar{b} _{j}\|_{2} + \sum_{p,q=1}^{m(n)}  \Biggl\|(\phi _{q,\tm(n)}(x) -\bar{\phi}_{q,\tm(n)}(x)) \tc\left[\int \phi _{p,\tm(n)}(y) A_{p,q}^{(j)}y  \, d \rho(y) \right]\Biggr\|_{2} \\ &+ \sum_{p,q=1}^{m(n)} \left\| \bar{\phi}_{q,\tm(n)}(x) \left( \tc\left[\int \phi _{p,\tm(n)}(y) A_{p,q}^{(j)}y \, d\rho(y) \right] -  \tc\left[\int \bar{\phi} _{p,\tm(n)}(y) \bar{A}_{p,q}^{(j)}y \, d\rho(y) \right]\right) \right\|_{2}  \\ \leq & \, \| b_{j}-\bar{b}_{j} \|_{2} + \sum_{p,q=1}^{m(n)} \left|  \phi _{q,\tm(n)}(x)- \bar{\phi}_{q,\tm(n)}(x) \right| \int \left\| A_{p,q}^{(j)} y\right\|_{2}  \, d \rho(y) \\ & \,\,\,\,\,\,\,\,\,\, + \sum_{p,q=1}^{m(n)}\left\| \int \Big(\phi _{p,\tm(n)}(y) A_{p,q}^{(j)}y- \bar{\phi}_{p,\tm(n)}(y)\bar{A}_{p,q}^{(j)}y \Big)\, d \rho(y)  \right\|_{2} \\ \leq&\| b_{j}- \bar{b}_{j} \|_{2} + \bc^{\frac{1}{2}}\sum_{p,q=1}^{m(n)} \left\| A_{p,q}^{(j)} \right\|_{F} \left| \phi _{q,\tm(n)}(x)-\bar{\phi} _{q,\tm(n)}(x)\right| + \sum_{p,q=1}^{m(n)} \left\| \int  \phi _{p,\tm(n)}(y) (A_{p,q}^{(j)}-\bar{A}_{p,q}^{(j)})y \, d \rho(y)  \right\|_{2}    \\ & \,\,\,\,\,\,\,\,\,\, +  \sum_{p,q=1}^{m(n)} \left\| \int (\phi _{p,\tm(n)}(y)-\bar{\phi}_{p,\tm(n)}(y)) \bar{A}_{p,q}^{(j)}y \, d \rho(y)  \right\|_{2} \\
\leq & \, \| b_{j}- \bar{b}_{j} \|_{2} + \bc^{\frac{1}{2}} \max_{1 \leq q \leq m(n)} \left|  \phi _{q,\tm(n)}(x)- \bar{\phi}_{q, \tm(n)}(x) \right| \sum_{p,q=1}^{m(n)} \left\| A_{p,q}^{(j)} \right\|_{F} + \bc^{\frac{1}{2}} \sum_{p,q=1}^{m(n)}\left\| A_{p,q}^{(j)}- \bar{A}_{p,q}^{(j)} \right\|_{F} \\ &  \,\,\,\,\,\,\,\,\,\, + \sum_{p,q=1}^{m(n)}  \int  \left|  \phi _{p,\tm(n)}(y)- \bar{\phi} _{p, \tm(n)}(y) \right| \left\| \bar{A}_{p,q}^{(j)}y \right\|_{2}   \, d \rho(y) \\ \leq \, &\epsilon + \sqrt{ \bc }m\epsilon+ \sqrt{d\bc }m \max_{1\leq q \leq m(n)} \left|  \phi _{q,\tm(n)}(x)-\bar{\phi}_{q,\tm(n)}(x) \right|\\ &\quad\quad+\sqrt{ d\bc }m \max_{1 \leq p \leq m(n)} \| \phi _{p,\tm(n)}-\bar{\phi}_{p,\tm(n)} \|_{L^{2}(\rho)}.  \end{aligned}$$
Recall that the above two-hidden-layer tanh neural networks have the form $\phi _{p,\tm(n)}=\pt \big(\phi^{(1)}_{p,\tm(n)}, \dots, \phi^{(d)}_{p, \tm(n)}\big)$ and then it can be obtained that for any $x \in \mathbb{R}^{d}$,  
$$\begin{aligned}& \left|  \phi _{p,\tm(n)}(x)- \bar{\phi}_{p,\tm(n)}(x) \right|   \leq d \max _{1 \leq l \leq d} \left| \phi^{(l)} _{p,\tm(n)}(x)- \bar{\phi}_{p,\tm(n)}^{(l)}(x)  \right|  \\ \leq & \,d \max _{1 \leq l \leq d}  \left|  c_{p,l} ^{T} \underbrace{ \left[\sigma _{\tanh}(W_{p,1}^{(l)}\sigma _{\tanh}(W_{p,0}^{(l)}x+b_{p,0}^{(l)})+b_{p,1}^{(l)})\right] }_{ =: f_{p}^{(l)}(x) } - \bar{c}_{p,l} ^{T} \underbrace{ \left[\sigma _{\tanh}(\overline{W}_{p,1}^{(l)}\sigma _{\tanh}(\overline{W}_{p,0}^{(l)}x+\bar{b}_{p,0}^{(l)})+\bar{b}_{p,1}^{(l)})\right] }_{ =: \bar{f}_{p}^{(l)}(x) } \right| \\ \leq & \, d \max _{ 1 \leq l \leq d}  \left|  c_{p,l } ^{T}f_{p}^{(l)}(x)- \bar{c}_{p,l}^{T} f_{p}^{(l)}(x)+ \bar{c}_{p,l}^{T}f_{p}^{(l)}(x)-\bar{c}_{p,l}^{T} \bar{f}_{p}^{(l)}(x)\right| \\ \leq & \,  d \max_{1 \leq l \leq d} \left(\| c_{p,l} - \bar{c}_{p,l} \|_{1} +\| \bar{c}_{p,l} \|_{1} \left\| f_{p}^{(l)}(x)- \bar{f}_{p}^{(l)}(x) \right\|_{\infty}  \right) \\\leq& \, 8d\,\tm(n)\max_{1\leq l \leq d} \left\{\| c_{p,l}- \bar{c}_{p,l} \|_{\infty}  + \| \bar{c}_{p,l} \|_{\infty} \left\| f_{p}^{(l)}(x)- \bar{f}_{p}^{(l)}(x) \right\|_{\infty}  \right\} ,\end{aligned}$$ 
where 
$$\begin{aligned}&\left\| f_{p}^{(l)}(x)- \bar{f}_{p}^{(l)}(x) \right\|_{\infty} \\\leq & \left\| (W_{p,1}^{(l)}\sigma _{\tanh}(W_{p,0}^{(l)}x+b_{p,0}^{(l)})+b_{p,1}^{(l)})-(\overline{W}_{p,1}^{(l)}\sigma _{\tanh}(\overline{W}_{p,0}^{(l)}x+\bar{b}_{p,0}^{(l)})+\bar{b}_{p,1}^{(l)}) \right\|_{\infty} \\  \leq & \left\| b_{p,1}^{(l)}- \bar{b}_{p,1}^{(l)} \right\|_{\infty} + \left\| W_{p,1}^{(l)} - \overline{W}_{p,1}^{(l)} \right\|_{\infty} + \left\| W_{p,1}^{(l)} \right\|_{\infty} \left\| ( W_{p,0}^{(l)}-\overline{W}_{p,0}^{(l)})x + b_{p,0}^{(l)}- \bar{b}_{p,0 }^{(l)} \right\|_{\infty} .      \end{aligned}$$
Choose that
\begin{equation} \label{ineq:cov}
    \begin{aligned}
    \| c_{p,l}-\bar{c}_{p,l} \|_{\infty} \leq \epsilon, &\left\| b_{p,1}^{(l)}- \bar{b}_{p,1}^{(l)} \right\|_{\infty} \leq \epsilon, \left\| W_{p,1}^{(l)}- \overline{W}_{p,1}^{(l)} \right\|_{\max} \leq \epsilon,\\ \left\| W_{p,0}^{(l)}- \overline{W}_{p,0}^{(l)} \right\|_{\max}&\leq \epsilon , \left\| b_{p,0}^{(l)}- \bar{b}_{p,0}^{(l)} \right\|_{\infty} \leq \epsilon.
\end{aligned}
\end{equation}
 Then it's easy to see that 
\begin{align*}
    \left\| f_{p}^{(l)}(x) - \bar{f}_{p}^{(l)}(x) \right\|_{\infty} \leq 24(1+\| x \|_{\infty})[c_{4}\tm(n)]^{c_{5}(\tm(n))^{2}}\epsilon,
\end{align*} which implies that 
\begin{align*}
    \left|  \phi _{p,\tm(n)}(x) - \bar{\phi}_{p, \tm(n)}(x) \right| \leq 192d \, (2+ \| x \|_{\infty})[c_{4}\tm(n)]^{3c_{5}(\tm(n))^{2}}\epsilon
\end{align*} and that 
$$\begin{aligned}
&\sqrt{ d\bc } m \Big(\max _{1\leq p,q\leq m(n)} \left| \phi _{q,\tm}(x) - \bar{\phi}_{q,\tm}(x) \right|+ \| \phi _{p,\tm}-\bar{\phi}_{p,\tm} \|_{L^{2}(\rho)}  \Big) \\ \leq \, & \sqrt{ d\bc } m \Big( 192d\left(4+\| x \|_{\infty} + \int_{\mathcal{X}} \| x \|_{\infty}   \, d\rho  \right)[c_{4}\tm(n)]^{3c_{5}(\tm(n))^{2}} \Big)\epsilon \\ \leq\, & 192C_{\mathcal{B}}d^{\frac{3}{2}}(5+\| x \|_{\infty})(m(n))[c_{4}\tm(n)]^{3c_{5}(\tm(n))^{2}}\epsilon.
\end{aligned}$$
 We can conclude that $$\left\| \mathrm{T}_{n}(\rho,x)-\bar{\mathrm{T}}_{n}(\rho,x) \right\|_{2} \leq 386 C_{F}C_{\mathcal{B}}d^{\frac{3}{2}} (9+\| x \|_{\infty} )(m(n)) [c_{4}\tm(n)]^{3c_{5}(\tm(n))^{2}}\epsilon=: \mathrm{\Xi}(x, \epsilon). $$

For any pseudometric space $(\mathcal{H},d_{\mathcal{H}})$ and $\epsilon >0$, the covering number of $(\mathcal{H}, d_{\mathcal{H}})$ with radius $\epsilon$ is defined as $\inf\{ \left|  \mathcal{D} \right|: \mathcal{D} \subset \mathcal{H}, \text{for any } \Psi \in \mathcal{H} \text{ there exists } \Phi \in \mathcal{D}  \text{ with } d_{\mathcal{H}}(\Psi,\Phi) \leq \epsilon  \}.$

We denote by $\mathrm{N}(\mathcal{H}_{\mathrm{T}_{n}}, \epsilon, d_{\Bb})$ the uniform $\epsilon$-covering number for the first stage sampling with the pseudometric 
\begin{align*}
    d_{\Bb}(\Phi _{1}, \Phi _{2})= \sup_{\rho \in \Bb} \mathbb{E} _{\rho}\| \Phi_{1}(\tilde{X})- \Phi _{2}(\tilde{X}) \|_{2}.
\end{align*}
It follows that for the above $\mathrm{T}_{n}, \bar{\mathrm{T}}_{n}$ and the parameter selections, 

\begin{align*}
    d_{\Bb}(\mathrm{T}_{n}, \bar{\mathrm{T}}_{n}) \leq \sup_{\rho \in \Bb}\mathbb{E}_{\rho} \Xi(X, \epsilon)\leq c_{6}C_{F}C_{\mathcal{B}}^{2}d^{\frac{3}{2}} (m(n))[c_{4}\tm(n)]^{3c_{5}(\tm(n))^{2}}\epsilon.
\end{align*}
Then the $\tilde{\epsilon}$-covering number of $\htt$ with respect to $d_{\Bb}$ can bounded as $$\begin{aligned}
&\quad\mathrm{N}(\htt, \tilde{\epsilon}, d_{\Bb}) \\&\leq \left(  1+ \frac{4C_{F}}{\epsilon} \right)^{n}\left( 1+ \frac{2C_{F}\sqrt{ 2d\bc }}{\epsilon} \right)^{n\Big(m^{2}d^{2}+d\Big)+1}  \left( 1+ \frac{(c'_{2}\tm)^{c'_{3}\tm^{2}}}{\epsilon} \right)^{[8\tm d+8\tm+(8\tm)^{2}+8\tm+8\tm]dm} \\ & \leq \left( 1+ \frac{4C_{F}\sqrt{ d\bc}}{\epsilon} \right)^{3nm^{2}d^{2}} \left(1+ \frac{{(c'_{2}\tm)^{c'_{3}\tm^{2}}}}{\epsilon} \right)^{96d^{2}\tm^{2} m} \\ & \leq \left( 1+ \frac{4c_{6}C_{F}C_{\mathcal{B}}^{3}d^{2}m(c_{4}\tm)^{3c_{5}\tm^{2}}}{\tilde{\epsilon}}  \right)^{3d^{2}nm^{2}} \left( 1+\frac{c_{6}C_{F}C_{\mathcal{B}}^{2}d^{\frac{3}{2}}m(c_{4}\tm)^{6c_{5}\tm^{2}}}{\tilde{\epsilon}} \right)^{96d^{2}\tm^{2}m} \\ &\leq \left(1+ \frac{8C_{d,F,\mathcal{B}}}{\tilde{\epsilon}} \right)^{100d^{2}nm^{2}} ((c_{4}m\tm)^{600c_{5}d^{2}nm^{2}\tm^{2}}),
\end{aligned}$$ 
with $C_{d,F,\mathcal{B}}=c_{6}C_{F}C_{\mathrm{\mathcal{B}}}^{3}d^{2}$, which is followed by $$\log \mathrm{N}(\htt, \tilde{\epsilon}, d_{\Bb}) \leq 100nd^{2}m^{2}\log\left( 1+\frac{8C_{d,F,\mathcal{B}}}{\te}  \right)+ 600c_{5}d^{2}nm^{2}\tm^{2}\log(c_{4}m\tm).$$
Conditioned on the pseudo second-stage samples $(\rho_{X}^{(i)}, X_{ij})_{1 \leq i \leq N, 1 \leq j \leq n_{i}}$, we can define the empirical $L^{2}_{\hat{P}}$ $\epsilon$-covering number $\mathrm{N}(\htt, \epsilon, d_{\hat{P}})$ with the pseudometric $$d _{\hat{P}}(\Phi_{1},\Phi_{2})=\left( \frac{1}{N}\sum_{i=1}^{N} \frac{1}{n_{i}} \sum_{j=1}^{n_{i}} \| \Phi _{1}(\rho_{X}^{(i)},X_{ij})- \Phi _{2}(\rho_{X}^{(i)},X_{ij}) \|_{2}^{2} \right)^{\frac{1}{2}}. $$
Then for the above $\mathrm{T}_{n}, \bar{\mathrm{T}}_{n}$ and the parameter selections, we have $$\begin{aligned}
d_{\hat{P}}( \mathrm{T}_{n}, \bar{\mathrm{T}}_{n})&\leq \left( \frac{1}{N}\sum_{i=1}^{N} \frac{1}{n_{i}} \sum_{j=1}^{n_{i}} \Xi(X_{ij}, \epsilon)^{2} \right)^{\frac{1}{2}}\\&\leq 386\sqrt{ 2 }C_{F}C_{\mathcal{B}}d^{\frac{3}{2}} \left( 9 + \sqrt{ \frac{1}{N}\sum_{i=1}^{N} \frac{1}{n_{i}} \sum_{j=1}^{n_{i}} \| X_{ij} \|_{\infty}^{2} } \right)m[c_{4}\tm]^{3c_{5}\tm^{2}}\epsilon \\ &\leq c_{6}C_{F}C_{\mathcal{B}}d^{\frac{3}{2}} \underbrace{ \left(9+ \sqrt{ \frac{1}{N}\sum_{i=1}^{N} \frac{1}{n_{i}} \sum_{j=1}^{n_{i}} \| X_{ij} \|_{2}^{2} }\right) }_{ =: \| \hat{X} \|_{2}  } m(c_{4}\tm)^{3c_{5}\tm^{2}} \epsilon.
\end{aligned}$$
Similarly, it's easy to see that $$\log \mathrm{N}(\htt, \tilde{\epsilon}, d_{\hat{P}}) \leq 100nd^{2}m^{2}\log\left(1+ \frac{8c_{6}C_{F}C_{\mathcal{B}}^{2}d^{2}\| \hat{X} \|_{2}}{\te} \right)+ 600 c_{5}nd^{2}m^{2}\tm^{2}\log(c_{4}m\tm).$$

\subsubsection{First-Stage Sampling Error Estimation} \label{error:first}

The following lemma is from Lemma 3.19 \cite{cuckerlearning}.

\begin{lemma}
    Let $\mathscr{J}$ be a set of functions on $Z$ and $\mathscr{B},c>0$ such that for each $\mathcal{J} \in \mathscr{J}$, $\left| \mathcal{J} - \mathbb{E} (\mathcal{J}) \right| \leq \mathscr{B}$ and $\mathbb{E} (\Gamma^{2}) \leq c \mathbb{E}(\Gamma)$ almost surely. Then for every $\epsilon >0$ and $0 < r \leq 1$, $$\mathbb{P}_{\mathbf{z}\in Z^{m}}\left\{ \sup_{\mathcal{J} \in \mathscr{J}}\frac{\mathbb{E}(\mathcal{J}) - \mathbb{E}_{\mathbf{z}}(\mathcal{J})}{\sqrt{ \mathbb{E}(\mathcal{J}) + \epsilon }} >4r\sqrt{ \epsilon } \right\} \leq \mathrm{N}(\mathscr{J}, r\epsilon, \| \cdot \|_{L^{\infty}(Z)} )\exp \left\{  - \frac{r^{2}m\epsilon}{2c+\frac{2}{3}\mathscr{B}}  \right\}.$$ 
\end{lemma}

\vskip 0.1in

We consider the class of functions $\gp:\mathcal{B}_{2,b}(\mathcal{X} \times \mathcal{Y}) \to \mathbb{R}$, denoted by $$\mathscr{J}(\htt):= \Big\{ \gp :\gp(\rho_{XY})= \mathbb{E}[\| \mathscr{T}_{M}\Phi(\tx)-Y \|_{2}^{2} | \rho_{XY} ] - \mathbb{E}[\| \Phi _{\mathcal{G}}(\tx)-{Y} \|_{2}^{2} |\rho_{XY} ], \Phi \in \htt \Big\}.$$
Then for each $\gp \in \mathscr{J}(\htt)$, we have $$\mathbb{E} (\gp)=\mathcal{E}(\mathscr{T}_{M}\Phi)-\mathcal{E}(\Phi _{\mathcal{G}}) \text{ and } \frac{1}{N} \sum_{i=1}^{N}\gp(\rho_{XY}^{(i)})= \mathcal{E}_{N}(\mathscr{T}_{M}\Phi)-\eg(\Phi _{\mathcal{G}}),$$ where $\begin{aligned}\E(\mathscr{T}_{M}\Phi)-\mathcal{E}(\Phi _{\mathcal{G}})= \| \mathscr{T}_{M}\Phi - \Pg \|_{L^{2}_{\ngg}}^{2}. \end{aligned}$
Also notice that $$\begin{aligned}\left| \gp(\rho_{XY})  \right|&= \left|  \mathbb{E}\big[\| \mathscr{T}_{M}\Phi(\tx)-Y \|_{2}^{2}-\| \Pg(\tx)-Y \|^{2}_{2}\big|\rho_{XY}\big] \right|\\& \leq  \left|  \mathbb{E}\big[(\| \mathscr{T}_{M}\Phi(\tx)-Y \|_{2}+\| \Pg(\tx)-Y \|_{2})(\| \mathscr{T}_{M}\Phi(\tx)- \Pg \|_{2} )\big|\rho_{XY}\big] \right|\\ & \leq 8M^{2},\end{aligned}$$ which implies that $\left|  \gp(\rho_{XY})- \mathbb{E}(\gp) \right|\leq 16M^{2}$ and $$\begin{aligned}\mathbb{E}(\gp^{2})& \leq \mathbb{E}_{\rho_{XY} \sim \mathcal{P}_{\mathcal{G}}}\Big[ \mathbb{E}\big[(\| \mathscr{T}_{M} \Phi(\tx)-Y \|_{2}^{2}-\| \Pg(\tx)-Y \|_{2}^{2} )^{2}|\rho_{XY} \big]\Big]\\&\leq 16M^{2}\mathbb{E}_{\rho_{X} \sim \pg}(\mathbb{E}[\| \mathscr{T}_{M}\Phi(\tx)-\Pg \|_{2}|\rho_{X}])^{2} \leq 16M^{2}\|  \mathscr{T}_{M} \Phi- \Pg \|_{L_{\nu}^{2}}^{2}=16M^{2}\mathbb{E}(\gp). \end{aligned}$$
Then for any $\Phi _{1}, \Phi _{2} \in \htt$, it follows that $$\begin{aligned}\left| \mathcal{J} _{\Phi _{1}}(\rho_{XY})- \mathcal{J} _{\Phi_{2}}(\rho_{XY}) \right|&= \left| \mathbb{E}[\| \mathscr{T}_{M}\Phi_{1}(\tx)-Y \|_{2}^{2} | \rho_{XY} ] -\mathbb{E}[\| \mathscr{T}_{M}\Phi_{2}(\tx)-Y \|_{2}^{2} | \rho_{XY} ] \right|  \\ & \leq 4M \left|  \mathbb{E} \big[\| \mathscr{T}_{M} \Phi_{1}(\tx)-\mathscr{T}_{M}\Phi_{2}(\tx) \|_{2} | \rho_{X}\big]  \right| \\ &\leq 4M \left|  \mathbb{E} \big[\|  \Phi_{1}(\tx)-\Phi_{2}(\tx) \|_{2} | \rho_{X}\big]  \right|  \leq 4M d_{\Bb}(\Phi_{1}, \Phi_{2}),\end{aligned}$$which shows that $$\mathrm{N}(\mathscr{J}(\htt), \epsilon, \| \cdot \|_{L^{\infty}(\mathcal{B}_{2,b}(\mathcal{X}\times \mathcal{Y}))} ) \leq \mathrm{N}\left( \htt, \frac{\epsilon}{4M}, d_{\Bb} \right).$$
Then with the uniform ratio inequality, we have that $$\begin{aligned}
&\mathbb{P} \left\{ \sup_{\Phi \in \htt} \frac{(\E(\mathscr{T}_{M}\Phi)-\E(\Pg)) -(\eg(\mathscr{T}_{M}\Phi) -\eg(\Pg))}{\sqrt{ \E(\mathscr{T}_{M}\Phi)-\E(\Pg) +\epsilon }} > \sqrt{ \epsilon } \right\} \\ &\leq \mathrm{N}\left( \mathscr{J}(\htt), \frac{\epsilon}{4}, \| \cdot \|_{L^{\infty}(\mathcal{B}_{2,b}(\mathcal{X}\times \mathcal{Y}))}  \right) \exp \left\{   - \frac{3N\epsilon}{2048M^{2}}   \right\} \leq \mathrm{N}\left( \htt, \frac{\epsilon}{16M}, d_{\Bb} \right) \exp \left\{  - \frac{3N\epsilon}{2048M^{2}}  \right\}.
\end{aligned}$$
It's easy to see that $\sqrt{( \E(\mathscr{T}_{M}\Phi)-\E(\Pg)+\epsilon) \epsilon } \leq \frac{1}{2}(\E(\mathscr{T}_{M}\Phi)-\E(\Pg))+\epsilon$, which follows by taking $\Phi= \mathrm{T}_{\mathbb{S},n}$ that 
\begin{align} \label{ineq:e1}
    \mathbb{P}\left\{  \E_{1}(\ts) > \frac{1}{2}\big(\E(\ts)-\E(\Pg)\big) +\epsilon  \right\} \leq \mathrm{N}\left( \htt, \frac{\epsilon}{16M}, d _{\Bb} \right)\exp \left\{  - \frac{3N\epsilon}{2048M^{2}}  \right\}.
\end{align}

Similarly, 
For each $\Phi \in \htt$, note that $\| \Phi \|_{C(\ob)}$ is uniformly bounded. We define 
\begin{align*}
    \wg_{\Phi}(\rho_{XY})= \mathbb{E}[\| \Phi(\tx)-Y \|_{2}^{2}|\rho_{XY}] - \mathbb{E}[\| \Pg(\tx)- Y \|_{2}^{2} | \rho_{XY}].
\end{align*} $\wg_{\Phi}$ can be considered as a random variable with $\left|  \wg_{\Phi}(\rho_{XY}) \right| \leq (3M+ \| \Phi \|_{C(\ob)})^{2}$ and it follows that $$\left|  \wg_{\Phi}(\rho_{XY})- \mathbb{E}(\wg_{\Phi}) \right| \leq 2(3M+\| \Phi \|_{C(\ob)})^{2}$$ and $$\mathbb{E}(\wg_{\Phi}^{2}) \leq (3M+\| \Phi \|_{C(\ob)})^{2}\E_{4}(\Phi).$$ Then with the Bernstein inequality, we have 
\begin{align} \label{ineq:e11}
   \mathbb{P}\{ \E'_{1}(\Phi) > \epsilon \} \leq \exp \left\{  -  \frac{N\epsilon^{2}}{2(3M+\| \Phi \|_{C(\ob)} )^{2}\left( \E_{4}(\Phi)+\frac{2}{3}\epsilon \right)}  \right\}. 
\end{align}

\subsubsection{Second-Stage Sampling Error with Ground Truth Context} \label{error: gtc}

Assume that $n_{1}=\dots=n_{N}=\vartheta$. Conditioned on the given first-stage samples $(\rho_{XY}^{(i)})_{1 \leq i \leq N}$, the random variables $(X_{ij},Y_{ij})_{1\leq i\leq N,1 \leq j\leq \vartheta}$ are independent but not identically distributed: for each $1 \leq i \leq N$, $(X_{ij}, Y_{ij}) \sim \rho_{XY}^{(i)}$. Let $\st=(\rho_{X}^{(i)}, X_{ij},Y_{ij})_{1\leq i \leq N,1 \leq j \leq \vartheta}$. We introduce a random variable $$\Lambda(\st)= \sup_{\bar{\Phi} \in \mathscr{T}_{M}(\htt)} \frac{1}{N}\sum_{i=1}^N\frac{1}{\vartheta}\sum_{j=1}^{\vartheta} \Big(\mathbb{E}\big[\| \bar{\Phi}(\tx)-Y \|_{2}^{2} \big|\rho_{XY}^{(i)} \big] - \| \bar{\Phi}(\rho_{X}^{(i)}, X_{ij}) - Y_{ij}  \|_{2}^{2} \Big),$$ which follows that $$\left|  \Lambda(\st) - \Lambda (\st^{\backslash(i,j)}) \right| \leq \sup_{\bar{\Phi} \in \mathscr{T}_{M}(\htt)} \frac{1}{N\vartheta}\left|  \| \bar{\Phi}(\rho_{X}^{(i)},X_{ij})-Y_{ij} \|_{2}^{2}- \| \bar{\Phi}(\rho_{X}^{(i)},X_{ij}')-Y_{ij}' \|_{2}^{2} \right| \leq \frac{8M^{2}}{N\vartheta}$$ where $\st^{\backslash(i,j)}$ denotes the sample $\st$ with a change on $(i,j)$-th variable with $(X_{ij},Y_{ij})\overset{i.i.d}{\sim}(X_{ij}',Y_{ij}')$ while all others fixed.
Then by Azuma-McDiarmind's inequality, it can be derived that $$\mathbb{P} \left\{\left|  \Lambda(\st)- \mathbb{E}\big[\Lambda(\st)|(\rho_{XY}^{(i)})_{i}\big]\right| > \epsilon \right\} \leq 2 \exp \left\{  - \frac{\left( N\vartheta \right)\epsilon^{2}}{32 M^{4}}  \right\},\text{ for any }\epsilon>0.$$
Let $(\zeta_{ij})_{1 \leq i\leq N, 1\leq j \leq \vartheta}$ be i.i.d Rademacher random variables. Then by the symmetrization \citep{ledoux1991probability}, we have $$\begin{aligned} \mathbb{E}\big[\Lambda(\st)\big|(\rho_{XY}^{(i)})_{i}\big]&\leq\mathbb{E}_{\big[(X_{ij},Y_{ij})\sim \rho_{XY}^{(i)}\big]_{i}} \mathbb{E}_{(\zeta _{ij})} \sup_{\bar{\Phi} \in \mathscr{T}_{M}(\htt)} \frac{2}{N}\sum_{i=1}^{N} \frac{1}{\vartheta}\sum_{j=1}^{\vartheta} \zeta _{ij}\| \bar{\Phi}(\rho_{X}^{(i)},X_{ij})-Y_{ij} \|_{2}^{2} . \end{aligned}$$
Since $\left|  l_{y}(u)-l_{y}(u') \right| \leq 4M \| u-u' \|_{2}$ where $l_{y}(u):= \| u-y \|_{2}^{2}$ with $\| u \|_{2}$ and $\| y \|_{2}$ less than $M$,
it follows by the vector-contraction inequality \citep{maurer2016vectorcontraction} that $$\begin{aligned}
\mathbb{E} [ \Lambda(\st) | (\rho_{XY}^{(i)})_{i}] &\leq 8\sqrt{ 2 }M\mathbb{E}_{\big[X_{ij} \sim \rho_{X}^{(i)}\big]_{i}} \mathbb{E}_{(\boldsymbol{\zeta}_{ij})}\sup_{\bar{\Phi} \in \mathscr{T}_{M}(\htt)} \frac{1}{N}\sum_{i=1}^{N} \frac{1}{\vartheta} \sum_{j=1}^{\vartheta} \langle \boldsymbol{\zeta}_{ij}, \bar{\Phi}(\rho_{X}^{(i)},X_{ij}) \rangle
\end{aligned}$$ where $\boldsymbol{\zeta}_{ij}$ is a random vector in $\mathbb{R}^{d}$ with each component being i.i.d Rademacher random variable. 

We define a family of zero-mean random variables indexed by $\bar{\Phi} \in \mathscr{T}_{M}(\htt)$ as $$Z_{\bar{\Phi}}:= \frac{1}{\sqrt{ N\vartheta }} \sum_{i=1}^{N} \sum_{j=1}^{\vartheta} \langle  \boldsymbol{\zeta}_{ij}, \bar{\Phi}(\rho_{X}^{(i)},X_{ij}) \rangle,$$ which implies that $$\mathbb{E}[\Lambda(\st) |(\rho_{XY}^{(i)})_{i}] \leq 8\sqrt{ 2 }M \mathbb{E}_{\left[ X_{ij} \sim \rho_{X}^{(i)} \right]_{i}} \mathbb{E}\left[ \sup_{\bar{\Phi}\in \mathscr{T}_{M}(\htt)} \frac{1}{\sqrt{ N\vartheta }}Z_{\pb} \Bigg| (X_{ij})_{i,j}\right].$$
For any $\bar{\Phi}, \bar{\Phi}' \in \mathscr{T}_{M}(\htt)$, $$\mathbb{E} [ \exp(v (Z_{\bar{\Phi}}-Z_{\bar{\Phi}'})) |(X_{ij})_{i,j}] \leq \exp(v^{2} d _{\vartheta}(\bar{\Phi},\bar{\Phi}')^{2}/2), \forall v \in \mathbb{R}$$ where $$d_{\vartheta}(\pb, \pb'):= \| \pb-\pb' \|_{L^{2}({P}_{\vartheta})}=\left( \frac{1}{N\vartheta} \sum_{i=1}^{N}\sum_{j=1}^{\vartheta}\| \pb(\rho_{X}^{(i)},X_{ij})-\pb'(\rho_{X}^{(i)},X_{ij}) \|_{2}^{2} \right)^{\frac{1}{2}}.$$
Then, conditioned on $(X_{ij})_{i,j}$, $Z_{\bar{\Phi}}$ is a subgaussian process indexed by $\bar{\Phi} \in \mathscr{T}_{M}(\htt)$ with respect to $d_{\vartheta}$. It follows by the Dudley Integral \citep{wainwright_high-dimensional_2019} that 
$$\mathbb{E} \left[\sup_{\pb \in \mathscr{T}_{M}(\htt)}Z_{\pb}\Bigg|(X_{ij})_{i,j}\right] \leq 32\int  ^{2M}_{0} \sqrt{ \log \mathrm{N}(u, \mathscr{T}_{M}(\htt),d_{\vartheta}) } \, du.$$ It follows that $$\begin{aligned}&\mathbb{E}[\Lambda(\st)|(\rho_{XY}^{(i)})_{i}] \leq \frac{256\sqrt{ 2 }M}{\sqrt{ N\vartheta }} \mathbb{E}_{[X_{ij} \sim \rho_{X}^{(i)}]_{i}}\left( \int  _{0}^{2M} \sqrt{ \log \mathrm{N}(u, \htt, d_{\hat{\rho}_{\vartheta}}) } \, du  \right) \\ & \leq \frac{256\sqrt{ 2 }M}{\sqrt{ N\vartheta }} \mathbb{E}_{[ X_{ij} \sim \rho_{X}^{(i)}]_{i}} \left( \int  _{0}^{2M}  \sqrt{ \underbrace{ 100nd^{2}m^{2} }_{ \vr_{1} }\log\left(1+ \frac{8c_{6}C_{F}C_{\mathcal{B}}^{2}d^{2}\| \hat{X} \|_{2}}{u} \right)+ \underbrace{ 600 c_{5}nd^{2}m^{2}\tm^{2}\log(c_{4}m\tm) }_{ \vr_{2} } }  \, du  \right) \\  &\leq \frac{256\sqrt{ 2 }M}{\sqrt{ N\vartheta }} \mathbb{E}_{[X_{ij} \sim \rho_{X}^{(i)}]_{i}} \left( 2M\sqrt{ \vr_{2} }+ \sqrt{ \vr_{1} }\int  _{0}^{2M} \sqrt{ \log\left( 1+ \frac{8c_{6}C_{F}\bc ^2d^{2}\| \hat{X} \|_{2}}{u}  \right) }\, du  \right) \\ & \leq \frac{256\sqrt{ 2 }M}{\sqrt{ N \vartheta }} \mathbb{E}_{[X_{ij} \sim \rho _{X}^{(i)}]_{i}} \left( 2M \sqrt{ \vr_{2} }+ \sqrt{ \vr_{1} }\sqrt{ 8c_{6}C_{F}\bc^2d^{2}\| \hat{X} \|_{2} } \int  _{0}^{2M}  u^{-\frac{1}{2}}\, du  \right) \\ &= \frac{256\sqrt{ 2 }M}{\sqrt{ N\vartheta }}\left( 2M\sqrt{ \vr_{2} }+ 8\sqrt{ \vr_{1}c_{6}C_{F}\bc^2d^{2}M }\,\, \mathbb{E}_{[X_{ij} \sim \rho_{X}^{(i)}]_{i}}\sqrt{ \| \hat{X} \|_{2}  }  \right) \end{aligned}$$ where $$\mathbb{E}_{[X_{ij} \sim \rho _{X}^{(i)}]_{i}}\sqrt{ \| \hat{X} \|_{2} }\leq \sqrt{ \mathbb{E}_{[X_{ij}\sim \rho _{X}^{(i)}]_{i}} \| \hat{X} \|_{2} }= \sqrt{  9+\mathbb{E} \sqrt{ \frac{1}{N\vartheta} \sum_{i=1}^{N} \sum_{j=1}^{\vartheta} \| X_{ij} \|_{2}^{2}} }\leq 9+ C_{\mathcal{B}}^{\frac{1}{2}}.$$ Then we can obtain that $$\mathbb{E}[\Lambda(\st)|(\rho_{XY}^{(i)})_{i}] \leq \frac{C_{M,\mathcal{B},d}}{\sqrt{ N\vartheta }} \sqrt{ n }m\tm^{2}$$ with $C_{M,\mathcal{B},d}=256\sqrt{ 2 }M\max\left\{  20dM\sqrt{ 6c_{5} }, 80d^2(9+\bc^{1/2})\sqrt{c_6C_F\bc^2M} \right\}$.

Then we have 
\begin{align} \label{ineq:e2}
    \mathbb{P} \left\{  \sup_{\Phi \in \htt} \mathcal{E}_{2}(\mathscr{T}_{M}(\Phi)) > \epsilon + C_{M,\mathcal{B},d} \frac{\sqrt{ n }m\tm^{2}}{\sqrt{ N\vartheta }} \middle| \rho_{XY}^{(1)},\dots, \rho_{XY}^{(N)} \right\} \leq \exp\left( - \frac{(N\vartheta)\epsilon^{2}}{32M^{4}} \right).
\end{align}

Similarly, for each $\Phi \in \htt$, we define $$\lt_{\Phi}(\st)=\frac{1}{N\vartheta}\sum_{i=1}^{N}\sum_{j=1}^{\vartheta} \Big(\| \Phi(\rho_{X}^{(i)},X_{ij})-Y_{ij} \|_{2}^{2}- \mathbb{E}\big[\| \Phi(\tx)-Y \|_{2}^{2}| \rho_{XY}^{(i)} \big] \Big).$$ Similarly, we have $$\left| \lt_{\Phi}(\st) - \lt_{\Phi}(\st^{\backslash (i,j)}) \right| \leq \frac{1}{N\vartheta} \left|  \| \Phi(\rho_{X}^{(i)},X_{ij})-Y_{ij} \|_{2}^{2}- \| \Phi(\rho_{X}^{(i)}, X'_{ij})-Y'_{ij} \|_{2}^{2}   \right|\leq \frac{4(M+\| \Phi \|_{C(\ob)} )^{2}}{N\vartheta} . $$
Then by Azuma-McDiarmind's inequality, we can derive that 
\begin{align} \label{ineq:e22}
    \mathbb{P}\{ \mathcal{E}'_{2}(\Phi)  > \epsilon | \rho_{XY}^{(1)},\dots, \rho_{XY}^{(N)} \} \leq  \exp \left\{  -\frac{(N\vartheta)\epsilon^{2}}{8(M+\| \Phi \|_{C(\ob)} )^{4}}  \right\}. 
\end{align}

\subsubsection{Second-Stage Sampling Error with Accessible Context} \label{error:acc}

Now, we estimate the sampling error with accessible context information, i.e., the empirical distributions. We have $$\begin{aligned}&\quad\sup_{\Phi \in \htt}\left|  \mathcal{E}_{3}(\mathscr{T}_{M}(\Phi)) \right|\\&= \sup_{\Phi \in\htt}\left |  \frac{1}{N\vartheta} \sum_{i=1}^{N} \sum_{j=1}^{\vartheta}\Big(\| \mathscr{T}_{M}(\Phi)(\rho_{X}^{(i)}, X_{ij})-Y_{ij} \|_{2}^{2} - \| \mathscr{T}_{M}(\Phi)(\hat{\rho}_{X}^{(i)}, X_{ij})-Y_{ij} \|_{2}^{2} \Big) \right|\\& \leq \sup_{\Phi \in\htt}\frac{1}{N\vartheta} \sum_{i=1}^{N} \sum_{j=1}^{\vartheta} 4M \left\| \Phi(\rho_{X}^{(i)},X_{ij})- \Phi(\hat{\rho}_{X}^{(i)},X_{ij}) \right\|_{2}  \\ & \leq \sup_{\Phi \in\htt} \frac{4M}{N} \sum_{i=1}^{N}  \left( \| \alpha \|_{1} \max_{1\leq j'\leq n} \sum_{p,q=1}^{m(n)} \left\| \int_{\mathcal{X}} \phi _{p,\tm(n)}(x)A_{p,q}^{(j')}x \, d \rho_{X}^{(i)}-\int_{\mathcal{X}}  \phi _{p,\tm(n)}(x)A_{p,q}^{(j')}x \, d\hat{\rho}_{X}^{(i)}   \right\|_{2}    \right)  \\ &\leq  \sup_{\phi \in \mathcal{N N}(\Theta _{\tm})} \frac{8MC_{F}}{N} \sum_{i=1}^{N} m\sqrt{ d } \max_{1\leq p\leq m} \left\|\int_{\mathcal{X}} \phi _{p,\tm(n)}(x)x  \, d\rho_{X}^{(i)}- \int  _{\mathcal{X}}\phi _{p,\tm(n)}(x)x \, d\hat{\rho}_{X}^{(i)}    \right\|_{2} \\& \leq \frac{8MC_{F}\sqrt{ d }m}{N}\sum_{i=1}^{N} \underbrace{ \sup_{\phi \in \mathcal{N N}(\Theta _{\tm})} \left\| \int  _{\mathcal{X}}\phi(x)x \, d\rho_{X}^{(i)}-\int  _{\mathcal{X}}\phi(x) x\, d\hat{\rho}_{X}^{(i)}   \right\|_{2} }_{ =:\mathcal{V}^{(i)}(\hat{\rho}_{X}^{(i)}) } \\&  \leq 8MC_{F} \sqrt{ d }m \max_{1 \leq i \leq N}  \mathcal{V}^{(i)}(\hat{\rho}_{X}^{(i)}). \end{aligned}$$
Then for $x=(x_{1},\dots,x_{\vartheta}) \in \mathcal{X}^{\vartheta}$ and the samples $(X_{i 1},\dots, X_{i \vartheta})$ in $\hat{\rho}_{X}^{(i)}$, we define $$\begin{aligned}
    \mathcal{V}_{j}^{(i)}(\hat{\rho}_{X}^{(i)})(x)&= \mathcal{V}^{(i)}(\delta[x_{1},\dots, x_{j-1}, X_{ij}, x_{j+1},\dots,x_{\vartheta}])\\&\quad\quad- \mathbb{E}_{X'_{ij} \sim \rho_{X}^{(i)}}(\mathcal{V}^{(i)}(\delta[x_{1},\dots, x_{j-1}, X'_{ij},x_{j+1},\dots, x_{\vartheta}]))
\end{aligned}$$ for $1 \leq j \leq \vartheta$ where $\delta[S]$ is defined as the empirical distribution generated by the dataset $S$.

We easily obtain that $$\begin{aligned}& \quad\left|\mathcal{V}_{j}^{(i)}(\hat{\rho}_{X}^{(i)})(x)\right|\\&= \left|\mathbb{E}_{X_{ij}'\sim \rho_{X}^{(i)}}\Big(\mathcal{V}^{(i)}(\delta[x_{1},\dots, x_{j-1},X_{ij}, x_{j+1},\dots, x_{\vartheta}])- \mathcal{V}^{(i)}(\delta[x_{1},\dots,x_{j-1}, X_{ij}', x_{j+1},\dots, x_{\vartheta}])\Big)\right|\\& \leq \mathbb{E}_{X_{ij}'\sim \rho_{X}^{(i)}} \Biggl|  \sup_{\phi \in \mathcal{N N}(\Theta _{\tm})}\left\| \frac{1}{\vartheta}\left( \sum_{j' \neq j}\phi(x_{j'})x_{j'}+\phi(X_{ij})X_{ij} \right)- \mathbb{E}_{\rho_{X}^{(i)}}(\phi(X)X ) \right\|_{2}  \\& \,\,\,\,\,\,\,\,\,\,\,\,\,\,\,\,\,\,\,\,\,\,\,\,\,\,\,\,\,- \sup_{\phi \in \mathcal{N N}(\Theta _{\tm})} \left\| \frac{1}{\vartheta}\left( \sum_{j' \neq j} \phi(x_{j'})x_{j'}+\phi(X_{ij}')X_{ij}'\right)- \mathbb{E}_{\rho_{X}^{(i)}}(\phi(X)X) \right\|_{2}  \Biggr|  \\&\leq \mathbb{E}_{X_{ij}' \sim \rho_{X}^{(i)}} \left( \sup_{\phi \in \mathcal{N N}(\Theta _{\tm})} \frac{1}{\vartheta}\left\| \phi(X_{ij})X_{ij}- \phi(X_{ij}')X_{ij}' \right\|_{2}  \right)\\ & \leq \mathbb{E}_{X_{ij}' \sim \rho_{X}^{(i)}} \left( \sup_{\phi \in \mathcal{N N}(\Theta _{\tm })} \frac{1}{\vartheta} (\| \phi(X_{ij})X_{ij} \|_{2}+ \| \phi(X_{ij}')X_{ij}' \|_{2}  ) \right)\\ &\leq \frac{1}{\vartheta} (\| X_{ij} \|_{2} + \mathbb{E}_{X_{ij}' \sim \rho_{X}^{(i)}} \| X_{ij}' \|_{2}  ) \leq \frac{1}{\vartheta} (\| X_{ij} \|_{2}+\sqrt{ \bc } ). \end{aligned}$$ 
For each $\rho_{X}^{(i)} \in \Bb$, we have the ratio condition that $\| \wi \|_{L^{\gamma}(\rk)}<\infty$. Then for $X_{ij} \sim \rho_{X}^{(i)}$ and any $p\geq 1$, we have $$\begin{aligned}\Big(\mathbb{E}_{\rho_{X}^{(i)}} \| X_{ij} \|_{2}^{p}\Big)^{\frac{1}{p}}&= \left( \int  _{\mathcal{X}} \| x \|_{2}^{p} \cdot \wi (x)\, d\rk \right)^{\frac{1}{p}} \leq \| \wi \|_{L^{\gamma}(\rk)}^{\frac{1}{p}} (\mathbb{E}_{\rk} \| X \|_{2}^{p\gamma'})^{\frac{1}{p\gamma'}} \\ & \leq (1+ \| \wi \|_{L^{\gamma}(\rk)}) (\mathbb{E}_{\rk} \| X \|_{2}^{p\gamma'} )^{\frac{1}{p\gamma'}}\end{aligned}$$ with $\gamma' = \frac{\gamma}{\gamma-1}$. For $X \sim \rk$, $\| X \|_{2}$ is a norm subgaussian random variable \citep{jin2019short} such that $(\mathbb{E}_{\rk}\| X \|_{2}^{p})^{\frac{1}{p}}\leq c\ka\sqrt{ d p }$ for any $p \geq 1$ where $c$ is an absolute constant. Then we have 
\begin{align*}
    \Big(\mathbb{E}_{\rho_{X}^{(i)}} \| X_{ij} \|_{2}^{p}\Big)^{\frac{1}{p}} \leq c\ka\sqrt{ \gamma' }\big(1+ \left\| \wi \right\|_{L^{\gamma}(\rk)}\big)\sqrt{ p }
\end{align*}
for any $p\geq 1$. We define the subgaussian norm \citep{wainwright_high-dimensional_2019} for a random variable $Z$ as $\| Z \|_{\uppsi_{2}}= \sup_{p\geq 1} \frac{(\mathbb{E} \left|  Z \right|^{p})^{\frac{1}{p}}}{\sqrt{ p }}$. Then $\left\| X_{ij} \right\|_{2}$ is a subgaussian random variable with $\left\| \| X_{ij} \|_{2} \right\|_{\uppsi_{2}}\leq c\ka\sqrt{ \gamma' }(1+\| \wi \|_{L^{\gamma}(\rk)})$. It also implies that for any $x \in X^{\vartheta}$, $\mathcal{V}_{j}^{(i)}(\hat{\rho}_{X}^{(i)})(x)$ is also a subgaussian random variable with $$\left\| \mathcal{V}_{j}^{(i)}(\hat{\rho}_{X}^{(i)})(x) \right\|_{\uppsi_{2}}\leq \frac{1}{\vartheta}\Big(\left\| \| X_{ij} \|_{2} \right\|_{\uppsi_{2}}+ \sqrt{ \bc }\Big),$$
because for any $p\geq 1$, 
\begin{align*}
    \Big(\mathbb{E}_{X_{ij}\sim \rho_{X}^{(i)}} \left|  \mathcal{V}_{j}^{(i)}(\hat{\rho}_{X}^{(i)})(x) \right|^{p}\Big)^{\frac{1}{p}} \leq \frac{1}{\vartheta} \left[ (\mathbb{E}_{\rho_{X}^{(i)}}\| X_{ij} \|_{2}^{p})^{\frac{1}{p}}+ \sqrt{ \bc } \right]
\end{align*}
by Minkowski inequality.
Then by Theorem 3 in \citet{maurer2021concentration}, we have for any $\epsilon>0$, $$\mathbb{P}\{ \mathcal{V}^{(i)}(\hat{\rho}_{X}^{(i)})- \mathbb{E}\mathcal{V}^{(i)}(\hat{\rho}_{X}^{(i)}) > \epsilon  \} \leq \exp \left( - \frac{\vartheta\epsilon^{2}}{ \uppsi_{2}(\rho_{X}^{(i)})} \right) $$ where $$\upp= 32e\Big(c\ka\sqrt{ \gamma' }(1+\| \wi \|_{L^{\gamma}(\rk)})+\sqrt{ \bc }\Big)^{2}.$$

We also have $$\begin{aligned}\mathbb{E}\mathcal{V}^{(i)}(\hat{\rho}_{X}^{(i)})&= \mathbb{E}_{X_{ij} \sim P_{X}^{(i)}}\sup_{\phi \in \mathcal{N N}(\Theta _{\tm})} \left\| \frac{1}{\vartheta} \sum_{j=1}^{\vartheta} \phi(X_{ij})X_{ij}- \mathbb{E}_{P_{X}^{(i)}}(\phi(X)X)  \right\|_{2}\\&\leq \mathbb{E}_{X_{ij},X_{ij}'\overset{i.i.d}{\sim}P_{X}^{(i)}} \sup_{\phi \in \mathcal{NN}(\Theta _{\tm})} \left\| \frac{1}{\vartheta} \sum_{j=1}^{\vartheta} [\phi(X_{ij})X_{ij}-\phi(X_{ij}')X_{ij}'] \right\|_{2} \\& =\frac{1}{\vartheta}\mathbb{E}_{X_{ij},X'_{ij} \overset{i.i.d}{\sim} P_{X}^{(i)}} \mathbb{E}_{\zeta _{ij}} \sup_{ \phi \in \mathcal{N N}(\Theta _{\tm})} \left\| \sum_{j=1}^{\vartheta} \zeta _{ij}[\phi(X_{ij})X_{ij}-\phi(X'_{ij})X'_{ij}] \right\|_{2}   \\&\leq \frac{2}{\vartheta} \mathbb{E}_{X_{ij}} \mathbb{E}_{\zeta _{ij}}\sup_{\phi \in \mathcal{N N}(\Theta _{\tm})}  \left\|\sum^{\vartheta}_{j=1} \zeta _{ij} \phi(X_{ij})X_{ij}  \right\|_{2} \\ &= \frac{2}{\vartheta} \mathbb{E}_{X_{ij}} \mathbb{E}_{\zeta _{ij}} \sup_{\phi \in \mathcal{N N}(\Theta _{\tm})} \sup_{ \mathrm{u} \in S^{d-1}} \left( \sum_{j=1}^{\vartheta} \zeta _{ij}\phi(X_{ij})\mathrm{u}^{T}X_{ij}\right) \\&=\frac{2}{\vartheta} \mathbb{E}_{X_{ij}} \mathbb{E}_{\zeta _{ij}} \sup_{f \in \mathrm{U}} \left( \sum_{j=1}^{\vartheta} \zeta _{ij} f(X_{ij}) \right)  \end{aligned}$$
where $(\zeta _{ij})_{j=1}^{\vartheta}$ are independent Rademacher random variables and $$\mathrm{U}:= \{ f: f(x)=\phi(x)\mathrm{u}^{T}x \text{ with } \phi \in \mathcal{N N}(\Theta _{\tm}), \mathrm{u} \in S^{d-1} \}.$$
We define a family of zero-mean random variables index by $f \in \mathrm{U}$ as $$Z_{f}^{(i)}:= \frac{1}{\sqrt{ \vartheta }} \sum_{j=1}^{\vartheta}\zeta _{ij}f(X_{ij}),$$ which implies that $$\mathbb{E} \mathcal{V}^{(i)}(\hat{\rho}_{X}^{(i)})\leq \mathbb{E}_{X_{ij}} \mathbb{E}\left[ \sup_{f \in \mathrm{U}} \frac{1}{\sqrt{ \vartheta }} Z_{f}^{(i)}\middle| (X_{ij})_{j} \right].$$
For any $f,f' \in \mathrm{U}$, $$ \mathbb{E}[\exp(v(Z_{f}^{(i)}- Z_{f'}^{(i)}))| (X_{ij})_{i,j}] \leq \exp (v^{2}\di(f,f') ^{2}/2), \,\forall v \in \mathbb{R}$$ where $$\di(f,f'):= \left( \frac{1}{\vartheta} \sum_{j=1}^{\vartheta} (f(X_{ij})-f'(X_{ij}))^{2} \right)^{\frac{1}{2}}.$$
Then, conditioned on $(X_{ij})_{j}$, $Z_{f}^{(i)}$ is a subgaussian process indexed by $f \in \mathrm{U}$ with respect to $\di$. It's easy to see that for any $f,f' \in \mathrm{U}$, 
\begin{align*}
\di(f,f')&=\left( \frac{1}{\vartheta} \sum_{j=1}^{\vartheta}\Big(\phi _{f}(X_{ij})\mathrm{u}_{f}^{T}X_{ij}- \phi _{f'}(X_{ij})\mathrm{u}_{f'}^{T}X_{ij}\Big)^{2} \right)^{\frac{1}{2}}  \leq 2 \sqrt{ \frac{1}{\vartheta} \sum_{j=1}^{\vartheta} \left\| X_{ij} \right\|_{2}^{2}  }=:2M_{i}.
\end{align*}
If  $\left\| \mathrm{u}_{f}- \mathrm{u}_{f'} \right\|_{2} \leq \epsilon$ and parameters in $\phi _{f}$ and $\phi _{f'}$ satisfy (\ref{ineq:cov}),
\begin{align*}
\di(f,f') &\leq \left( \frac{1}{\vartheta}\sum_{j=1}^{\vartheta} \Big(\phi _{f}(X_{ij})\mathrm{u}_{f}^{T}X_{ij}- \phi _{f'}(X_{ij})\mathrm{u}_{f}^{T}X_{ij}\Big)^{2} \right)^{\frac{1}{2}} \\ &\,\,\,\,\,\,\,\,\,\,+ \left( \frac{1}{\vartheta}\sum_{j=1}^{\vartheta} \Big(\phi _{f'}(X_{ij})\mathrm{u}_{f}^{T}X_{ij}- \phi _{f'}(X_{ij})\mathrm{u}_{f'}^{T}X_{ij}\Big)^{2} \right)^{\frac{1}{2}} \\ & \leq \left( \frac{1}{\vartheta} \sum_{j=1}^{\vartheta}\Big(\phi _{f}(X_{ij})-\phi _{f'}(X_{ij})\Big)^{2}\left\| X_{ij} \right\|_{2}^{2}  \right)^{\frac{1}{2}}+ M_{i} \left\|  \mathrm{u}_{f}- \mathrm{u}_{f'} \right\|_{2} \\& \leq \left( \frac{1}{\vartheta} \sum_{j=1}^{\vartheta}\Big(192d(2+\| X_{ij} \|_{\infty} )[c_{4}\tm]^{3c_{5}\tm^{2}}\epsilon\Big)^{2} \| X_{ij} \|_{2}^{2}  \right)^{\frac{1}{2}} + M_{i} \epsilon \\ & \leq \left( M_{i}+192d(c_4\tm)^{3c_{5}\tm^{2}}\sqrt{ \frac{1}{\vartheta} \sum_{j=1}^{\vartheta}(8\| X_{ij} \|_{2}^{2} +2\| X_{ij} \|_{2}^{4} ) } \right) \epsilon.
\end{align*}
Then the covering number 
\begin{align*}
\mathrm{N}(\tilde{\epsilon}, \mathrm{U}, \di) & \leq \left( 1+ \frac{(c_{2}'\tm)^{c_{3}'\tm^{2}}}{\epsilon} \right)^{96d^{2}\tm^{2}}\left( 1+ \frac{2}{\epsilon} \right)^{d}\leq \left( 1+ \frac{(c_{2}'\tm)^{c_{3}'\tm^{2}}}{\epsilon} \right)^{100d^{2}\tm^{2}} \\ & \leq \left( 1+ \frac{\Delta^{(i)}}{\tilde{\epsilon}} \right)^{100d^{2}\tm^{2}} (c_{6}\tm)^{c_{7}\tm^{4}},
\end{align*} where $\Delta^{(i)}:= M_{i}+ \sqrt{ \frac{1}{\vartheta}\sum_{j=1}^{\vartheta}(8\| X_{ij} \|_{2}^{2}+2\| X_{ij} \|_{2}^{4}) }$, $c_{6}=192d(c_{2}'+c_{4})$ and $c_{7}=100(c_{3}'+3c_{5})d^{2}$.

Then by Dudley integral, we have \begin{align*}
&\mathbb{E}\left[  \sup_{f \in \mathrm{U}} Z_{f}^{(i)}\middle|(X_{ij})_{j}  \right]\leq 32 \int  _{0}^{2M_{i}} \sqrt{ \log \mathrm{N}(v, \mathrm{U, \di}) } \, dv \\ \leq & \, 32 \int  _{0} ^{2M_{i}} \sqrt{100d^{2}\tm^{2} \log \left( 1+\frac{\Delta^{(i)}}{v}  \right) + c_{7}\tm^{4} \log(c_{6}\tm) } \, dv \\ \leq & \, 32 \left( 2M_{i} \tm^{2}\sqrt{ c_{7}\log(c_{6}\tm) } + 10d\tm\int _{0}^{2M_{i}} \sqrt{ \log\left( 1+ \frac{\Delta^{(i)}}{v} \right) }\, dv  \right)\\\leq & \, 32 \left( 2M_{i} \tm^{2}\sqrt{ c_{7}\log(c_{6}\tm) } + 10d\tm \int  _{0}^{2M_{i}} \sqrt{ \Delta^{(i)} } v^{-\frac{1}{2}} \, dv  \right) \\ \leq & \, 32 \Big(2M_{i} \tm^{2}\sqrt{ c_{7}\log(c_{6}\tm) } + 20d\tm  \sqrt{ 2M_{i} \Delta^{(i)} }\Big).
\end{align*}
Then we have \begin{align*}
\mathbb{E} \mathcal{V}^{(i)}(\hat{\rho}_{X}^{(i)}) &\leq \frac{1}{\sqrt{ \vartheta }}\mathbb{E}_{X_{ij}} 32 \Big(2M_{i} \tm^{2}\sqrt{ c_{7}\log(c_{6}\tm) } + 20d\tm  \sqrt{ 2M_{i} \Delta^{(i)} }\Big) \\ & = \frac{1}{\sqrt{ \vartheta }} \left(64\tm^{2}\sqrt{ c_{7}\log(c_{6}\tm) } \mathbb{E}_{X_{ij}}M_{i}+ 640d\tm\mathbb{E}_{X_{ij}}\sqrt{ 2M_{i}\Delta^{(i)} } \right)
\end{align*}
where $\mathbb{E}_{X_{ij}}M_{i}= \mathbb{E}_{X_{ij}}\sqrt{ \frac{1}{\vartheta}\sum_{j=1}^{\vartheta} \| X_{ij} \|_{2}^{2} } \leq \sqrt{ \frac{1}{\vartheta}\sum_{j=1}^{\vartheta} \mathbb{E}_{X_{ij}} \| X_{ij} \|_{2}^{2} }=\sqrt{ \bc }$ and 
\begin{align*}
\mathbb{E}_{X_{ij}} \sqrt{ 2M_{i}\Delta^{(i)} }&= \mathbb{E}_{X_{ij}} \sqrt{ 2M_{i}\left( M_{i}+\sqrt{ \frac{1}{\vartheta}\sum_{j=1}^{\vartheta}(8\| X_{ij} \|_{2}^{2} +2\| X_{ij} \|_{2}^{4} ) } \right) } \\ & \leq \mathbb{E}_{X_{ij}} \sqrt{ 2M_{i} \left( M_{i} + \sqrt{ \frac{1}{\vartheta} \sum_{j=1}^{\vartheta}8\| X_{ij} \|_{2}^{2}  }+ \sqrt{ \frac{1}{\vartheta}\sum_{j=1}^{\vartheta}2\| X_{ij} \|_{2}^{4}  } \right) } \\ & = \mathbb{E}_{X_{ij}} \sqrt{ (2+4\sqrt{ 2 })M_{i}^{2} + 2\sqrt{ 2 }M_{i}\sqrt{ \frac{1}{\vartheta}\sum_{j=1}^{\vartheta}\| X_{ij} \|_{2}^{4}  }} \\ & \leq \sqrt{ (2+4\sqrt{ 2 })\mathbb{E}_{X_{ij}}M_{i}^{2}+ 2\sqrt{ 2 } \sqrt{ \mathbb{E}_{X_{ij}}M_{i}^{2} } \sqrt{ \mathbb{E}_{X_{ij}}\left( \frac{1}{\vartheta}\sum_{j=1}^{\vartheta} \| X_{ij} \|_{2}^{4}  \right) } } \\ & \leq \sqrt{ (2+4\sqrt{ 2 })\bc + 2\sqrt{ 2 }\sqrt{ \bc }\bc} \leq \sqrt{ 2+6\sqrt{ 2 } } \bc ^{\frac{3}{4}}.
\end{align*}
It follows that $$\mathbb{E} \mathcal{V}^{(i)}(\hat{\rho}_{X}^{(i)}) \leq \frac{1}{\sqrt{ \vartheta }}\left( 64\sqrt{ \bc }\tm^{2}\sqrt{ c_{7}\log(c_{6}\tm) }+ 640d\sqrt{ 2+6\sqrt{ 2 } }\bc^{\frac{3}{4}}\tm  \right)\leq \frac{c_{8}}{\sqrt{ \vartheta }}\tm^{2}\sqrt{ \log(\tm) }$$ where $c_{8}=64\sqrt{ \bc c_{7}(\log c_{6}+1) }+ 640d\sqrt{ 2+6\sqrt{ 2 } }\bc^{\frac{3}{4}}$.

 Then we have $$\mathbb{P}\left\{   \mathcal{V}^{(i)}(\hat{\rho}_{X}^{(i)}) > \epsilon + \frac{c_{8}\tm^{2}\sqrt{  \log(\tm) }}{\sqrt{ \vartheta }}  \right\} \leq \exp\left( - \frac{\vartheta\epsilon^{2}}{\upp} \right),$$ and it follows that $$\begin{aligned}
     \mathbb{P}\left\{   \sup_{\Phi \in \htt} \left|\mathcal{E}_{3}(\mathscr{T}_{M}(\Phi))\right| > 8MC_{F}\sqrt{ d }m \left( \epsilon+ \frac{c_{8}\tm^{2}\sqrt{  \log(\tm) }}{\sqrt{ \vartheta }} \right) \middle| \rho_{X}^{(1)},\dots, \rho_{X}^{(N)} \right\} \\ \leq N \exp \left(  -  \frac{\vartheta \epsilon^{2}}{\max_{1\leq i \leq N}\upp}  \right)
 \end{aligned}$$
which is equivalent to the inequality that 
\begin{equation} \label{ineq:e3}
    \begin{aligned}
    \mathbb{P}\left\{   \sup_{\Phi \in \htt} |\mathcal{E}_{3}(\mathscr{T}_{M}(\Phi))| > \epsilon + \frac{8c_{8}\sqrt{d}MC_F\tm^{2}m\sqrt{  \log(\tm) }}{\sqrt{ \vartheta }} \middle| \rho_{X}^{(1)},\dots, \rho_{X}^{(N)} \right\} \\\leq  N \exp \left( - \frac{\vartheta \epsilon^{2}}{64M^{2}C_{F}^{2}dm^{2} \max_{1 \leq i \leq N}\upp}  \right).
\end{aligned}
\end{equation}

Similarly, for any $\Phi \in \htt$, we have both $\| \Phi \|_{C(\ob)}$ and $\| \Phi \|_{C(\Omega)}$ uniformly bounded respectively. Then we have $$\begin{aligned}\mathcal{E}_{3}(\Phi) &=\frac{1}{N\vartheta}\sum_{i=1}^{N} \sum_{j=1}^{\vartheta} \Big(\| \Phi(\hat{\rho}_{X}^{(i)},X_{ij})-Y_{ij} \|_{2}^{2}- \| \Phi(\rho_{X}^{(i)},X_{ij})-Y_{ij} \|_{2}^{2}  \Big) \\ & \leq \frac{1}{N \vartheta} \sum_{i=1}^{N} \sum_{j=1}^{\vartheta}(2M+\| \Phi \|_{C(\ob)}+\| \Phi \|_{C(\Omega)}  ) \left\| \Phi(\hat{\rho}_{X}^{(i)},X_{ij})- \Phi(\rho_{X}^{(i)},X_{ij}) \right\|_{2} \\ & \leq 4(M+\| \Phi \|_{C(\Omega)}  )C_{F}\sqrt{ d }m \max_{1\leq i \leq N} \left\| \int  _{\mathcal{X}} \phi _{\Phi}(x)x  \, d\hat{\rho}_{X}^{(i)} - \int  _{\mathcal{X}} \phi _{\Phi}(x) x\, d\rho_{X}^{(i)}   \right\|_{2}, \end{aligned}$$
which follows that 
\begin{equation} \label{ineq:e33}
    \begin{aligned} 
  &\mathbb{P} \left\{   \mathcal{E}'_{3} (\Phi) > \epsilon +  \frac{4c_8 C_F\sqrt{d}(M+\| \Phi \|_{C(\Omega)} )\tm^2m\sqrt{\log(\tm)}}{\sqrt{ \vartheta }} \middle|\rho_{X}^{(1)},\dots, \rho_{X}^{(N)}  \right\}\\ \leq& N \exp \left( - \frac{\vartheta \epsilon^{2}}{16(M+\| \Phi \|_{C(\Omega)} )^{2}C_{F}^{2}dm^{2}\max_{1 \leq i \leq N}\upp} \right)
\end{aligned}
\end{equation}

Recall that for any $\Phi \in \htt$,  
\begin{align*}&\mathcal{E}((\ts))- \mathcal{E}(\Phi _{\mathcal{G}})= \left\| \ts- \Phi _{\mathcal{G}} \right\|_{L^{2}(\ngg)} ^{2} \\ \leq&\,  \mathcal{E}_{1}(\ts)+ \mathcal{E}'_{1}(\Phi)+ \mathcal{E}_{2}(\ts)+ \mathcal{E}'_{2}(\Phi) +\mathcal{E}_{3}(\ts)+\mathcal{E}'_{3}(\Phi)+ \mathcal{E}_{4}(\Phi). 
\end{align*}
\vskip 0.1in
Then combine all inequalities \eqref{ineq:e1}\eqref{ineq:e11}\eqref{ineq:e2}\eqref{ineq:e22}\eqref{ineq:e3}\eqref{ineq:e33}, and by the union bound, for any $\Phi \in \htt$ and $\epsilon>0$ we have \begin{align*}&\mathbb{P} \left\{  \| \ts-\Pg \|_{L^{2}(\pi)}^{2}> 2 \mathcal{E}_{4}(\Phi) + 2C_{M,\mathcal{B},d} \frac{\sqrt{ n }m \tm^{2}}{\sqrt{ N\vartheta }} + c_9(3M+\| \Phi \|_{C(\Omega)} ) \frac{\tm^{2}m\sqrt{\log(\tm) }}{\sqrt{ \vartheta }} + 12\epsilon   \right\} \\ \leq &\, \mathrm{N}\left( \htt, \frac{\epsilon}{16M}, d_{\Bb}  \right) \exp \left\{  - \frac{3N\epsilon}{2048M^{2}}  \right\}+ \exp \left\{  - \frac{N \epsilon^{2}}{2(3M+\| \Phi \|_{C(\ob)} )^{2}\left( \mathcal{E}_{4}(\Phi)+ \frac{2}{3} \epsilon \right)}  \right\}\\ &+ 2\exp \left\{  - \frac{(N\vartheta)\epsilon^{2}}{32(M+\| \Phi \|_{C(\ob)} )^{4}}  \right\}\\&+ 2\mathbb{E}_{P_{X}^{(i)} \overset{}{\sim}\pg} N \exp \left\{  - \frac{\vartheta\epsilon^{2}}{64(M+\| \Phi \|_{C(\Omega)} )^{2}C_{F}^{2}d m^{2} \max_{1\leq i\leq N}\upp}  \right\}\end{align*} where $\mathcal{E}_{4}(\Phi)= \|  \Phi -\Pg \|_{L^{2}(\ngg)}^{2}$ and $c_9=8c_{8}C_{F}\sqrt{ d }$.
 \qedblack
\vskip 0.2in

\subsection{Theorem \ref{thm:gen}: Generalization Bound for Linear Transformers}

From the approximation in Theorem \ref{thm:app}, for $n \geq 3$, there exists a transformer $\mathrm{T} \in \htt$ such that $$\| \mathrm{T}- \Pg \|_{L^{2}(\ngg)}\leq C_{*} \left( \left\lfloor  \frac{n}{2}  \right\rfloor \right)^{-\frac{1}{2}} \leq C_{*}'n^{-\frac{1}{2}} \text{ with } C_{*}'= 2 C_{*}$$ and by the approximant construction, we have $\| \mathrm{T} \|_{C(\ob)} \leq \| \mathrm{T} \|_{C(\Omega)}\leq 2 C_{F}\sqrt{ d(1+\bc) }=\mathcal{A}_{1}$.
Apply the above oracle inequality with letting $\Phi=\mathrm{T}$ and $\upn:= \frac{1}{N}\sum_{i=1}^{N}\upp$, and we obtain that 
$$\begin{aligned}&\mathbb{P} \left\{  \|  \ts- \Pg \|_{L^{2}(\pi)}^{2} > 2C_{*}'^{2} n^{-1}+2C_{M,\mathcal{B},d} \frac{\sqrt{ n }m\tm^{2}}{\sqrt{ N\vartheta }}+ 8c_{8}C_{F}\sqrt{ d }(3M+\mathcal{A}_{1} ) \frac{\tm^{2}m\sqrt{\log(\tm) }}{\sqrt{ \vartheta }}+12\epsilon  \right\} \\ \leq & \, \exp \left\{  100d^{2}nm^{2}\log \left( 1+\frac{128MC_{d,F,\mathcal{B}}}{\epsilon } \right) + 600 c_{5}d^{2}nm^{2}\tm^{2}\log(c_{4}m\tm) - \frac{3N\epsilon}{2048M^{2}}\  \right\} \\ &+ \, \exp \left\{  - \frac{N\epsilon^{2}}{2(3M+\mathcal{A}_{1})^{2}\left( C_{*}'^{2}n^{-1}+ \frac{2}{3}\epsilon \right)}  \right\} + 2\exp \left\{  - \frac{(N\vartheta)\epsilon^{2}}{32(M+\mathcal{A}_{1})^{4}}  \right\} \\ &\, + \, 2\mathbb{E}_{P_{X}^{(i)} \sim\pg} N \exp \left\{  - \frac{\vartheta\epsilon^{2}}{64(M+\mathcal{A}_{1})^{2}dC_{F}^{2}m^{2}N\upn}  \right\}.  \end{aligned}$$
\vskip 0.1in

We follow the parameter selections in the approximation by letting $m= \lceil n^{\frac{\gamma}{2(\gamma-1)\xi}}\rceil$ and $\tm= \lceil (\frac{1}{2}+ \frac{\gamma}{4(\gamma-1)\xi})\log n \rceil$. It's obtained that 
\begin{align*}&\mathbb{P} \left\{  \| \ts- \Pg \|_{L^{2}(\pi)}^{2} > 2C_{*}'^{2}n^{-1} + \mathcal{A}_{2} \frac{n^{\frac{1}{2}+ \frac{\gamma}{2(\gamma-1)\xi}}(\log n)^{2}}{\sqrt{ N\vartheta }} + \mathcal{A}_{3} \frac{n^{\frac{\gamma}{2(\gamma-1)\xi}}(\log n)^{3}}{\sqrt{ \vartheta }} + 12 \epsilon   \right\}\\ \leq & \, \exp \left\{  200d^{2}n^{1+ \frac{\gamma}{(\gamma-1)\xi}} \log\left( 1+\frac{128MC_{d,F,\mathcal{B}}}{\epsilon} \right) +  \mathcal{A}_{4}n^{1+\frac{\gamma}{(\gamma-1)\xi}}(\log n)^{3} - \frac{3N\epsilon}{2048M^{2}}\right\} 
\end{align*}
\begin{align*}
&+ \, \exp \left\{   -  \frac{N\epsilon^{2}}{2(3M+\mathcal{A}_{1})^{2}\left( C_{*}'^{2}n^{-1}+ \frac{2}{3}\epsilon \right)}  \right\}+ 2 \exp \left\{  - \frac{(N\vartheta)\epsilon^{2}}{32(M+\mathcal{A}_{1})^{4}}  \right\} \\ & + \,2 \mathbb{E}_{P_{X}^{(i)} \sim \pg} N\exp \left\{   - \frac{\vartheta\epsilon^{2}}{128d(M+\mathcal{A}_{1})^{2}C_{F}^{2} n^{\frac{\gamma}{(\gamma-1)\xi}}N \upn}  \right\} \end{align*}
where $\mathcal{A}_{2}=\big( 1+ \frac{\gamma}{2(\gamma-1)\xi} \big)^{2}C_{M,\mathcal{B},d}$, $\mathcal{A}_{3}=32(3M+\mathcal{A}_{1})C_{F}\sqrt{ d\bc }$ and $$\mathcal{A}_{4}=\left( 1+ \frac{\gamma}{(\gamma-1)\xi} \right)\left[ \left( 1200c_{5}d^{2}\left( \frac{1}{2}+ \frac{\gamma}{4(\gamma-1)\xi} \right)^{2}+\log c_{4}\left( \frac{1}{2}+\frac{\gamma}{4(\gamma-1)\xi} \right)^{2} \right) \right].$$

If we take $\epsilon \geq 2C_{*}'^{2}n^{-1}(\log n)^{3}$, it follows that 
\begin{align*}&\mathbb{P}\left\{  \| \ts- \Pg \|_{L^{2}(\pi)}^{2} > 13\epsilon + \mathcal{A}_{2} \frac{n^{\frac{1}{2}+ \frac{\gamma}{2(\gamma-1)\xi}}(\log n)^{2}}{\sqrt{ N \vartheta }} +\mathcal{A}_{3} \frac{n^{\frac{\gamma}{2(\gamma-1)\xi}}(\log n)^{3}}{\sqrt{ \vartheta }}   \right\} \\ \leq & \, \exp \left\{  200d^{2} n^{1+ \frac{\gamma}{(\gamma-1)\xi}} \log \left(1+ \frac{64MC_{d,F,\mathcal{B}}}{C_{*}'^{2}}\, n \right)+ \mathcal{A}_{4} n^{1+ \frac{\gamma}{(\gamma-1)\xi}} (\log n )^{3} - \frac{3N\epsilon}{2048M^{2}}  \right\} \\ &+ \, \exp \left\{  - \frac{3N\epsilon}{8(3M+\mathcal{A}_{1})^{2}}  \right\} + 2 \exp \left\{  - \frac{(N\vartheta)\epsilon^{2}}{32(M+\mathcal{A}_{1})^{4}}  \right\}\\ &+\,2\mathbb{E}_{P_{X}^{(i)} \sim \pg} N\exp \left\{   - \frac{\vartheta\epsilon^{2}}{128d(M+\mathcal{A}_{1})^{2}C_{F}^{2} n^{\frac{\gamma}{(\gamma-1)\xi}}N \upn}  \right\}. \end{align*}
Take $n= \lfloor \mathcal{K}_{1} N^{\frac{1}{2+ \gamma/[(\gamma-1)\xi]}} \rfloor$ with $\mathcal{K}_{1}=\left( \min\left\{  \frac{3C_{*}'^{2}}{819200M^{2}d^{2}\left( 1+\log\left( \frac{64MC_{d,F,\mathcal{B}}}{C_{*}} \right) \right)} ,\frac{3C_{*}'^{2}}{4096M^{2}\mathcal{A}_{4}}  \right\} \right)^{\frac{1}{2+\frac{\gamma \xi}{(\gamma-1)}}}$ and the second stage data size $\vartheta\geq N$,  then we have 
\begin{align*}&\mathbb{P}\{ \| \ts - \Pg \|_{L^{2}(\pi)}^{2} > \mathcal{A}_{5} \epsilon \} \\ \leq &\, \exp \left\{   \frac{3N\epsilon}{8192M^{2}}+ \frac{3N\epsilon}{8192M^{2}}- \frac{3N\epsilon}{2048M^{2}}  \right\}+ \exp \left\{  - \frac{3N\epsilon}{8(3M+\mathcal{A}_{1})^{2}}  \right\} +\, 2\exp \left\{  - \frac{C_{*}'^{2}N\epsilon}{16\mathcal{K}_{1}(M+\mathcal{A}_{1})^{4}}  \right\} \\ &+ \, 2 \mathbb{E}_{P_{X}^{(i)}\sim \pg}N \exp \left\{   -  \frac{{ \vartheta }\epsilon}{\mathcal{A}_{6}N^{ \frac{3(\gamma-1)\xi+2\gamma}{2(\gamma-1)\xi+\gamma}}\upn}  \right\} \\ \leq &\, 4 \exp \left\{  - \frac{N\epsilon}{\mathcal{A}_{7}}  \right\}+ 2\mathbb{E}_{P_{X}^{(i)} \sim \pg} N \exp \left\{   - \frac{ { \vartheta  } \epsilon}{\mathcal{A}_{6}N^{\frac{3(\gamma-1)\xi+2\gamma}{2(\gamma-1)\xi+\gamma}}\upn}   \right\},\end{align*}
where $$\mathcal{A}_{5}=13+ \frac{\mathcal{A}_{2}\mathcal{K}_{1}^{\frac{3}{2}+\frac{\gamma}{2(\gamma-1)\xi}}+ \mathcal{A}_{3} \mathcal{K}_{1}^{1+\frac{\gamma}{(\gamma-1)\xi}}}{2C_{*}'^{2}} , \,  \mathcal{A}_{6}= 128d(M+\mathcal{A}_{1})^{2}C_{F}^{2}\mathcal{K}_{1}^{\frac{\gamma}{(\gamma-1)\xi}},  $$
and $\mathcal{A}_{7}= \min \left\{   \frac{3}{3096M^{2}}, \frac{3}{8(3M+\mathcal{A}_{1})^{2}}, \frac{C_{*}'^{2}}{16\mathcal{K}_{1}(M+\mathcal{A}_{1})^{4}}  \right\}$.

Take $t= \mathcal{A}_{5}\epsilon$. Then when $t \geq 2\mathcal{A}_{5}C_{*}'^{2}n^{-1}(\log n)^{3}$, we have $$ \mathbb{P}\{ \| \ts - \Pg \|_{L^{2}(\ngg)}^{2}> t  \}\leq 4 \exp \left\{  - \frac{Nt}{\mathcal{A}_{5}\mathcal{A}_{7}}  \right\}+ 2 \mathbb{E}_{\rho_{X}^{(i)}\sim \pg}N \exp \left\{  - \frac{{ \vartheta }t}{\mathcal{A}_{5}\mathcal{A}_{6}N^{\frac{3(\gamma-1)\xi+2\gamma}{2(\gamma-1)\xi+\gamma}}\upn}  \right\}.$$
It implies that 
\begin{align*}&\mathbb{E}\{ \mathcal{E}(\ts)- \mathcal{E}(\Pg) \}=\mathbb{E}\| \ts-\Pg \|_{L^{2}(\ngg)}^{2} = \int  _{0}^{\infty} \mathbb{P}\{ \| \ts - \Pg \|_{L^{2}(\ngg)}^{2} > t  \} \, dt \\ =&\,\left(  \int_{0}^{2\mathcal{A}_{5}C_{*}'^{2}n^{-1}(\log n)^{3}} +\int  _{2\mathcal{A}_{5}C_{*}'^{2}n^{-1}(\log n)^{3}}^{\infty}  \right)\,  \mathbb{P}\{ \| \ts -\Pg\|_{L^{2}(\ngg)}^{2} >t   \} \, dt  \\ \leq & \, 2\mathcal{A}_{5}C_{*}'^{2}n^{-1}(\log n)^{3}+ \int  _{0}^{\infty} \mathbb{P} \{ \| \ts -\Pg \|_{L^{2}(\ngg)}^{2} >t  \} \, d t \\ \leq & \, 2\mathcal{K}_{2} N^{- \frac{1}{2+ \gamma/[(\gamma-1)\xi]}} (\log N)^{3} + \int_{0}^{\infty} 4 \exp \left\{  - \frac{N^{\frac{1}{2+\gamma/[(\gamma-1)\xi]}}t}{\mathcal{A}_{5}\mathcal{A}_{7}}  \right\}   \, dt \\ & \,\,\,\,\,\,\,\,\,\,\,\,\,\,\,\,\,\,\,\,\,\,\,\,\,\,\,\,\,\,\,\,\,\,\,\,\,\,\,\,\,\,\,\,\,\,\,\,\,\,\,\,\,\,\,\,\,\,\,\,\,\,\,+2\mathbb{E}_{\rho_{X}^{(i)}\sim \pg}N \int  _{0}^{\infty}  \exp \left\{  - \frac{{ \vartheta }t}{\mathcal{A}_{5}\mathcal{A}_{6}N^{\frac{3(\gamma-1)\xi+2\gamma}{2(\gamma-1)\xi+\gamma}}\upn}  \right\} \, dt \\=& \, 2\mathcal{K}_{2} N^{- \frac{1}{2+ \gamma/[(\gamma-1)\xi]}} (\log N)^{3} + 4 \mathcal{A}_{5}\mathcal{A}_{7} N^{- \frac{1}{2+ \gamma/[(\gamma-1)\xi]}} + 2 \mathcal{A}_{5} \mathcal{A}_{6} \frac{N^{\frac{5(\gamma-1)\xi+3\gamma}{2(\gamma-1)\xi+\gamma}}}{{ \vartheta }} \mathbb{E}_{\rho_{X}^{(i)}\sim \pg} \Big(\upn\Big)  
\end{align*}
where $\mathcal{K}_{2}=4\mathcal{A}_{5}C_{*}'^{2}\mathcal{K}_{1}^{-1}\left( \log \mathcal{K}_{1}+ \frac{1}{2+\gamma/[(\gamma-1)\xi]} \right)^{3}$. Take $\vartheta= N^{3}$ and we obtain that 
\begin{align*}\mathbb{E} \{  \mathcal{E}(\ts)-\mathcal{E}(\Pg) \}\leq \left(2\mathcal{K}_{2}+4\mathcal{A}_{5}\mathcal{A}_{7}+2\mathcal{A}_{5}\mathcal{A}_{6}\mathbb{E}_{\rho_{X}^{(i)}\sim \pg}\Big(\upn\Big)\right)N^{- \frac{1}{2+ \gamma/[(\gamma-1)\xi]}} (\log N)^{3}\end{align*} where $$\begin{aligned}\mathbb{E}_{\rho_{X}^{(i)} \sim \pg}\Big(\upn\Big)&= \mathbb{E}_{\rho_{X}^{(i)}\sim \pg}\left( \frac{1}{N} \sum_{i=1}^{N} \upp \right)\\&=32e \, \mathbb{E}_{\rho_{X}^{(i)}\sim\pg} \frac{1}{N}\sum_{i=1}^{N}\Big(c\ka\sqrt{ \gamma' }(1+\| \wi \|_{L^{\gamma}(\rk)})+\sqrt{ \bc }\Big)^{2} \\& \leq 64e \,\mathbb{E}_{\rho_{X} \sim \pg} \Big(c ^{2}\ka^{2}\gamma'(1+\| \omega _{\ka}(\rho_{X}) \|_{L^{\gamma}(\rk)} )^{2}+ \bc \Big) \\ &\leq \, 64e(\bc+2 c ^{2}\ka^{2}\gamma')+ 128e \, \mathbb{E}_{\rho_{X}\sim\pg}\| \omega _{\ka}(\rho_{X}) \|^{2}_{L^{\gamma}(\rk)}\\&\leq \, 64 e (\bc+2 c ^{2}\ka^{2}\gamma'+ 2C_{\mathcal{G}}). \end{aligned}$$
It follows that $$\mathbb{E}\{ \mathcal{E}(\ts)- \mathcal{E}(\Pg) \}\leq \mathcal{K}_{3}N^{- \frac{1}{2+ \gamma/[(\gamma-1)\xi]}} (\log N)^{3}$$ with $\mathcal{K}_{3}= 2\mathcal{K}_{2}+4\mathcal{A}_{5}\mathcal{A}_{7}+128e\mathcal{A}_{5}\mathcal{A}_{6}(\bc+ 2c ^{2}\ka^{2} \gamma'+ 2C_{\mathcal{G}})$. \qedblack


\newpage

\section{Context Embedding and Feature Mapping} \label{app:inj}

\begin{proposition} \label{prop:Kr}
    $\Kr:\bx \to \mathcal{H}_{\kr} \otimes \mathbb{R}^d$ is an injective and continuous mapping. 
\end{proposition}

\begin{proposition} \label{prop:il}
    The embedding operator $\il: \Omega \to \mathcal{H}_{\mathcal{F}}$ is injective and continuous.
\end{proposition}

\noindent\textit{Proof.} {Continuity:} Recall that $\Omega=\bx \times \mathcal{X}$ is a metric space equipped with $d_{\Omega}$. Also observe that 
\begin{align*} \left\|  \il(\rho,x)-\il(\rho',x')  \right\|_{\mathcal{H}_{\mathcal{F}}} &\leq \| \il(\rho,x)-\il(\rho',x) \|_{\mathcal{H}_{\mathcal{F}}} + \|  \il(\rho',x) - \il(\rho',x') \|_{\mathcal{H}_{\mathcal{F}}}  \\ & =\| \Kr(\rho-\rho') \otimes\kr(x, \cdot) \|_{\mathcal{H}_{\mathcal{F}}} + \| \Kr(\rho') \otimes(\kr(x, \cdot)-\kr(x', \cdot)) \|_{\mathcal{H}_{\mathcal{F}}},
\end{align*}
in which 
\begin{align*}&\| \kr(x,\cdot)-\kr(x',\cdot) \|_{\mathcal{H}_{\kr}}^{2}= 2(1-\exp \{ -(x-x')^{T}\Sigma _{\boldsymbol{\lambda}}(x-x') \})\\&\leq 2(x-x')^{T}\Sigma _{\boldsymbol{\lambda}}(x-x')\leq 2\| \Sigma _{\boldsymbol{\lambda}} \|_{2}\| x-x \|_{2}^{2},
\end{align*} and for any $\tau \in \prod (\rho,\rho')$, 
\begin{align*}&\| \Kr(\rho-\rho') \|_{\h \otimes \mathbb{R}^{d}} = \left\| \int  _{\mathcal{X} \times \mathcal{X}} (\kr(\cdot,y)y-\kr(\cdot,y')y' )\, d\rho(y)d \rho'(y')  \right\|_{\h \otimes \mathbb{R}^{d}} \\ &\leq \int  _{\mathcal{X} \times \mathcal{X}} \| \kr(\cdot,y)y-\kr(\cdot,y')y' \|_{\h \otimes \mathbb{R}^{d}} \, d\tau(y,y') \\ & \leq \int _{\mathcal{X} \times \mathcal{X}} \| \kr(\cdot,y)(y-y') \|_{\h \otimes \mathbb{R}^{d}} + \| (\kr(\cdot,y)-\kr(\cdot,y'))y' \|_{\h \otimes \mathbb{R}^{d}}    \, d\tau(y,y') \\ &\leq \int  _{\mathcal{X} \times \mathcal{X}} \| y-y' \|_{2}  \, d\tau(y,y') + \int  _{\mathcal{X} \times \mathcal{X}} \sqrt{ 2 }\| \Sigma _{\boldsymbol{\lambda}} \|_{2}^{\frac{1}{2}}   \| y-y' \|_{2} \| y' \|_{2}   \, d \tau(y,y') \\ &\leq \left( \int  _{\mathcal{X} \times \mathcal{X}} \| y-y' \|_{2}^{2}  \, d \tau(y,y')    \right)^{\frac{1}{2}} + \sqrt{ 2 } \| \Sigma _{\boldsymbol{\lambda}} \|_{2}^{\frac{1}{2}} \left( \int  _{\mathcal{X} \times \mathcal{X}} \| y-y' \|_{2}^{2}  \, d \tau(y,y')    \right)^{\frac{1}{2}} (\mathbb{E}_{\rho'}\| Y' \|_{2}^{2})^{\frac{1}{2}} \\ & \leq \left( 1+ \sqrt{ 2 } \| \Sigma _{\boldsymbol{\lambda}} \|_{2}^{\frac{1}{2}}(\mathbb{E}_{\rho'}\| Y' \|_{2}^{2})^{\frac{1}{2}} \right)  \left( \int  _{\mathcal{X} \times \mathcal{X}} \| y-y' \|_{2}^{2}  \, d \tau(y,y')    \right)^{\frac{1}{2}}.
\end{align*} Since the above inequality holds for any $\tau \in \prod(\rho, \rho')$, it follows that $$\| \Kr(\rho-\rho') \|_{\h \otimes \mathbb{R}^{d}} \leq (1+ \sqrt{ 2 } \| \Sigma _{\boldsymbol{\lambda}} \|_{2}^{\frac{1}{2}}(\mathbb{E}_{\rho'}\| Y' \|_{2}^{2})^{\frac{1}{2}})W_{2}(\rho,\rho') $$ and that 
\begin{align*}\| \il(\rho,x)-\il(\rho',x') \|_{\mathcal{H}_{\mathcal{F}}} &\leq \left( 1+ \sqrt{ 2 } \| \Sigma _{\boldsymbol{\lambda}} \|_{2}^{\frac{1}{2}}(\mathbb{E}_{\rho'}\| Y' \|_{2}^{2})^{\frac{1}{2}} \right)W_{2}(\rho,\rho')+ \sqrt{ 2 }\| \Sigma _{\boldsymbol{\lambda}} \|_{2}^{\frac{1}{2}} (\mathbb{E}_{\rho'}\| Y' \|_{2}^{2}) \| x-x' \|_{2} \\ &\leq (1+2\sqrt{ 2 }\| \Sigma _{\boldsymbol{\lambda}} \|_{2}^{\frac{1}{2}} (\mathbb{E}_{\rho'}\| Y' \|_{2}^{2})) d_{\Omega}((\rho,x), (\rho',x')).   
\end{align*}

Injection: $\kr(x,y)=g_{\boldsymbol{\lambda}}(x-y)= \exp \left\{  -\frac{1}{2}(x-y)^{T} \Sigma _{\boldsymbol{\lambda}}^{-1}(x-y)  \right\}$ with $\Sigma _{\boldsymbol{\lambda}}=\mathrm{diag}(2\lambda_{1}^{-2},\dots, 2\lambda _{d}^{-2})\succ 0$. For each $1\leq j \leq d$, let $$\widehat{\Kr^{(j)}(\rho)}(\omega)= \int  _{\mathbb{R}^{d}} e^{ -i \omega^{T}y } \Kr^{(j)}(\rho)(y) \, dy=\int  _{\mathbb{R}^{d}}\int  _{\mathbb{R}^{d}} e^{ -i\omega^{T}y } g_{\boldsymbol{\lambda}}(y-x) x^{(j)}\, d\rho(x) \, dy.$$ By Fubini Theorem, we have 
\begin{align*}
\widehat{\Kr^{(j)}(\rho)}(\omega)&=\int _{\mathbb{R}^{d}} \left[ \int _{\mathbb{R}^{d}} e^{ -i\omega^{T}y }g_{\boldsymbol{\lambda}}(y-x) \, dy   \right] \,x^{(j)} d\rho(x) \\ &=(2\pi)^{\frac{d}{2}} \det(\Sigma _{\boldsymbol{\lambda}})^{\frac{1}{2}} e^{ -\frac{1}{2}\omega^{T}\Sigma _{\boldsymbol{\lambda}}\omega } \int  _{\mathbb{R}^{d}}x^{(j)} e^{ -i\omega^{T}x } \, d \rho(x) \\&=(2\pi)^{\frac{d}{2}} \det(\Sigma _{\boldsymbol{\lambda}})^{\frac{1}{2}} e^{ -\frac{1}{2}\omega^{T}\Sigma _{\boldsymbol{\lambda}}\omega } \,\,   i \partial _{j}   \mathrm{F}(\rho)(\omega),
\end{align*} 
where $ \mathrm{F}(\rho)$ is Fourier transform of probability measure $\rho$. 

It follows that $\widehat{\Kr(\rho)}(\omega)=(2\pi)^{\frac{d}{2}} \det(\Sigma _{\boldsymbol{\lambda}})^{\frac{1}{2}} e^{ -\frac{1}{2}\omega^{T}\Sigma _{\boldsymbol{\lambda}}\omega }\,\, i \nabla  \mathrm{F}(\rho)(\omega)$. Take $\mu=\rho_{1}-\rho_{2}$ with $\rho_1, \rho_2 \in \bx$. If ${\Kr(\mu)}=0$, then $\Kr(\mu)(y)=0$ for any $y \in \mathbb{R}^{d}$. It implies that $\widehat{\Kr(\mu)}\equiv0$ and $\nabla  \mathrm{F}({\mu})\equiv 0$. Note that $ \mathrm{F}({\mu})(0)=\mu(\mathcal{X})=\rho_{1}(\mathcal{X})-\rho_{2}(\mathcal{X})=0$. It can be obtained that $ \mathrm{F}({\mu})\equiv 0$. Then by the inversion of Fourier transform of measures, we have $\rho_{1}=\rho_{2}$, which shows that $\Kr$ is an injective mapping on $\bx$. It also follows that $\il$ is injective since $x \mapsto \kr(x,\cdot)$ is also an injective mapping. \qedblack

\section{Examples for Marginal Meta Probability Measure} \label{sec:mmp}

\begin{example}
    The Class of Distributions with Compact Support and Bounded Density
    
    For $\mathrm{B}, \mathrm{C}>0$, we define the probability class $\mathcal{G}(\mathrm{B},\mathrm{C})$ of all probability measures with a Lebesgue density bounded by $\mathrm{C}$ almost surely and supported on the the closed ball of radius $\mathrm{B}$ centered at zero. Then ${\mathcal{P}}_{\mathcal{G}}^{\mathcal{X}}$ is a probability measure supported on $\mathcal{G}(\mathrm{B},\mathrm{C})$.
\end{example}

\noindent\textit{Proof.} Take $(\mu_n)_{}$ a sequence in $\mathcal{G}(\mathrm{B},\mathrm{C})$ with $\mu_n \to \mu$ in $(\bx, W_2)$. Denote $\mathrm{K}$ the closed ball of radius B centered at zero. Then by Portmanteau Theorem, the weak convergence of measures implies that $1=\limsup_{n \to \infty} \mu_n(\mathrm{K})\leq \mu(\mathrm{K}) \leq1$. Therefore, $\mu(\mathrm{K})=1$. 

Also by the weak convergence, we have $\int \varphi d\mu = \lim_{n \to \infty} \int \varphi d \mu_n \leq \mathrm{C}\int \varphi d \nu$ for any $\varphi \in C_c^+(\mathrm{K})$, where $\nu$ is Lebesgue measure. Then by the density of function class $C_c^+(\mathrm{K})$, we have $\mu(E) \leq \mathrm{C}\nu(E)$ for any Borel set $E \subset \mathrm{K}$, which follows that $\mu$ is absolute continuous with respect to $\nu$ and $d \mu / d\nu \leq \mathrm{C}$ almost everywhere on $\mathrm{K}$. It implies that $\mu \in \mathcal{G}(\mathrm{B},\mathrm{C})$ and then $\mathcal{G}(\mathrm{B},\mathrm{C})$ is closed in $(\bx,W_2)$.

It's easy to obtain that $$\begin{aligned}
    \mathbb{E}_{\rho \sim \pg} \| \w \|^{2}_{L^{\gamma}(\rk)}&=  \, \int _{\mathcal{G}(\mathrm{B},\mathrm{C})} \left( \int  _{\mathrm{K}} \left( \frac{d\rho}{d\rk} \right)^{\gamma} \, d\rk  \right)^{\frac{2}{\gamma}}\, d\pg(\rho) \\& \leq \mathrm{C}^{2} (2\pi\ka^2)^{\frac{(\gamma-1)d}{\gamma}} \exp \left( \frac{(\gamma-1)\mathrm{B}^{2}}{\gamma\ka^{2}} \right).\end{aligned}$$ \qedblack






\begin{example} \label{exam:2}
   The Class of Distributions in Diffusion Generative Modeling \citep{song2019generative} 

For $\mathrm{B}>0$ and $0 < t_{0}< T< \sqrt{  \frac{\gamma}{\gamma-1}}\ka$, we define the probability class $\mathcal{G}_{[t_{0},T]}(\mathrm{B})$ as the collection of the convolutions between two probability distributions 
\begin{align*}
    \Big\{ \mu* \rho _{\kt}:& \mu \text{ supported on the closed ball with radius } \mathrm{B} \text{ in }\mathbb{R}^{d},\,  \\&\text{Gaussian measure }\rho _{\kt} \text{ with } t_{0} \leq \kt \leq T  \Big\}
\end{align*}

with $$(\mu*\rho _{\kt})(x)= (2\pi\kt^{2})^{-\frac{d}{2}}\int_{\| y \|_{2} \leq \mathrm{B} }  \exp\left( - \frac{\| x-y \|^{2}}{2\kt^{2}}  \right) \, d\mu(y).$$

The marginal meta probability measure $\pg$ is defined as a joint probability measure on $\mathcal{B}_{2,b}(\mathcal{X}) \times[t_{0}, T]$ as $\tilde{\mathcal{P}}^{\mathcal{X}}_{\mathcal{G}} \times \mathrm{Uniform}[t_{0}, T]$ where $\tilde{\mathcal{P}}_{\mathcal{G}}^{\mathcal{X}}$ is a probability measure defined on $\mathcal{B}_{2,b}(\mathcal{X})$. 
\end{example}

\noindent\textit{Proof.} Since the convolution between $\mu$ and $\rho _{\kt}$ can be considered as the probability distribution of random variable $X+Z_{\kt}$ with $X \sim \mu$ and $Z_{\kt} \sim \text{gaussian distribution } \rho _{\kt}$ independently, $W_{2}(\mu_{1}*\rho _{\kt}, \mu_{2}*\rho _{\kt})\leq W_{2}(\mu_{1}, \mu _{2})$ by the coupling argument. Similarly for $(\mu_{1}, t_{1}), (\mu_{2}, t_{2}) \in \mathcal{B}_{2,b}(\mathcal{X}) \times[t_{0},T]$, we have $$W_{2}(\mu_{1}*\rho _{t_{1}}, \mu _{2}*\rho _{t_{2}})\leq W_{2}(\mu _{1}, \mu_{2})+W_{2}(\rho _{t_{1}},\rho _{t_{2}})\leq W_{2}(\mu_{1}, \mu_{2})+\sqrt{ d } \left|  t_{1}-t_{2} \right|,$$ which shows that $\tilde{\mathrm{I}}: (\mu, \kt) \mapsto \mu * \rho _{\kt}$ is continuous. Then the meta probability in domain generalization framework can be defined with $\mathcal{P}_{\mathcal{G}}^{\mathcal{X}}= (\tilde{\mathcal{P}}_{\mathcal{G}}^{\mathcal{X}} \times \mathrm{Uniform}[t_{0},T]) \circ \tilde{\mathrm{I}}^{-1}$. 
It's easy to see that for any $\mu*\rho _{\kt} \in \gb$, we have 
\begin{align*}
    \mathbb{E}\| X+Z \|_{2}^{4} &\leq \mathbb{E} (2\| X \|_{2}^{2}+ 2\| Z \|_{2}^{2})^2=4(\mathbb{E}\| X \|_{2}^{4}+ \mathbb{E}\| Z \|_{2}^{4})+8(\mathbb{E}\| X \|_{2}^2)(\mathbb{E}\| Z \|_{2}^{2})\\&\leq 4(\mathrm{B}^4+d(d+2)T^4)+8d\mathrm{B}^{2}T^2,
\end{align*}
and 
\begin{align*} 
&\|\omega _{\ka}(\mu*\rho _{\kt})\|_{L^{\gamma}(\rk)}^{\gamma} = (2\pi\ka^{2})^{\frac{d(\gamma-1)}{2}}\int _{\mathbb{R}^{d}}  (\mu*\rho _{\kt}(x))^{\gamma} \exp\left( \frac{\gamma-1}{2\ka^{2}}\| x \|_{2}^{2}  \right) \, dx \\=\,& (2\pi\ka^{2})^{\frac{d(\gamma-1)}{2}} (2\pi\kt^{2})^{-\frac{d\gamma}{2}} \int  _{\mathbb{R}^{d}} \left( \int_{\| y \|_{2} \leq \mathrm{B} } \exp\left( - \frac{{\| x-y \|_{2}^{2}}}{2\kt^{2}}  \right)  \,   d\mu(y)\right)^{\gamma} \exp \left( \frac{{\gamma-1}}{2\ka^{2}} \| x \|_{2}^{2} \right) \, dx \\=\, & (2\pi\ka^{2})^{\frac{d(\gamma-1)}{2}} (2\pi\kt^{2})^{-\frac{d\gamma}{2}} \int  _{\mathbb{R}^{d}} \left( \int  _{\| y \|_{2} \leq \mathrm{B} } \exp \left( \frac{2x^{T}y-\| y \|_{2}^{2} }{2\kt^{2}} \right) \,   d\mu(y)\right)^{\gamma}\exp\left( - \left( \frac{\gamma}{2\kt^{2}}- \frac{{\gamma-1}}{2\ka^{2}} \right)\| x \|_{2}^{2}  \right) \, dx .  
\end{align*}

It's easy to see that $\| \omega _{\ka}(\mu* \rho _{\kt}) \|_{L^{\gamma}(\rk)} <\infty$ if and only if $\vk:=\frac{\gamma}{2\kt^{2}}- \frac{{\gamma-1}}{2\ka^{2}}$ which is equivalent to the condition $\kt< \sqrt{ \frac{\gamma}{\gamma-1} }\ka$. 
Moreover, Let $\bk= \frac{\gamma}{\kt^{2}}$. By Jensen's inequality, we have \begin{align*}&\| \omega _{\ka}(\mu* \rho _{\kt}) \|_{L^{\gamma}(\rk)}^{\gamma} \\&\leq \mathcal{A}_{d,\ka,\gamma} \kt^{-d\gamma} \int  _{\mathbb{R}^{d}} \int  _{\| y \|_{2} \leq \mathrm{B} } \exp\left( - \frac{{\gamma \| x-y \|_{2}^{2} }}{2\kt^{2}} \right)  \, d\mu(y) \exp\left( \frac{{\gamma-1}}{2\ka^{2}} \| x \|_{2}^{2} \right) \, dx \\ & \leq \mathcal{A}_{d,\ka,\gamma} \kt^{-d\gamma} \int  _{\| y \|_{2} \leq \mathrm{B}} \int  _{\mathbb{R}^{d}} \exp \left( - \Big(\frac{\gamma}{2\kt^{2}}-\frac{{\gamma-1}}{2\ka^{2}}\Big)\| x \|_{2}^{2}  \right) \exp\Big( \frac{{2\gamma x^{T}y-\gamma \| y \|_{2}^{2}}}{2\kt^{2}} \Big) \, dx  \, d\mu(y) \\ &\leq \mathcal{A}_{d,\ka,\gamma} \kt^{-d\gamma} \int_{\| y \|_{2} \leq \mathrm{B} } \int_{\mathbb{R}^{d}} \exp(-\vk\| x \|_{2}^{2} +\bk y^{T}x) \, dx \exp\Big(-\frac{{\gamma \| y \|_{2}^{2}}}{2\kt^{2}} \Big)   \, d\mu(y) \\& = \mathcal{A}_{d,\ka,\gamma} \kt^{-d\gamma} \int _{\| y \|_{2} \leq \mathrm{B} }\int _{\mathbb{R}^{d}} \exp\Big(-\vk\left\| x-\frac{{\bk y}}{2\vk} \right\|_{2}^{2}  \Big) \, dx \exp\Big( \frac{\bk^{2}}{4\vk}\| y \|_{2}^{2} -\frac{\gamma}{2\kt^{2}}\| y \|_{2}^{2}  \Big) \, d\mu(y) \\&=\mathcal{A}_{d,\ka,\gamma}\Big(\frac{\pi}{\vk}\Big)^{\frac{d}{2}} \kt^{-d\gamma} \int _{\| y \|_{2} \leq \mathrm{B} }   \,\exp\Big(\big(\frac{\bk^{2}}{4\vk}-\frac{\gamma}{2\kt^{2}}\big)\| y \|_{2}^{2} \Big) d\mu(y) \\ & \leq \mathcal{A}_{d,\ka,\gamma} \Big(\frac{\pi}{\vk}\Big)^{\frac{d}{2}} \kt^{-d\gamma} \exp\left(  \frac{\gamma(\gamma-1)\kt^{2}}{2(\gamma\ka^{2}\kt^{2}-(\gamma-1)\kt^{4})} \mathrm{B^{2}}\right)   =: f(\kt)^{\gamma} ,  
\end{align*}
We can observe that $f(\kt)$ is a continuous function on $[t_{0},T]$ and $f(\kt)\sim \kt^{-d\left( 1-\frac{1}{\gamma} \right)}, \kt \to 0$. For the domain generalization assumption, we have 
\begin{align*}
\int  _{\bx}  \| \w \|^{2}_{L^{\gamma}(\rk)} \, d \pg(\rho) &= \int  _{\gb} \| \w \|_{L^{\gamma}(\rk)}^{2} \, d \pg(\rho)\\&= \frac{1}{T-t_{0}}\int_{t_{0}}^{T}  \int_{\mathcal{B}_{2,c}(\mathcal{X})}  \| \omega _{\ka}(\mu*\rho _{\kt}) \|_{L^{\gamma}(\rk)}^{2}  \, d \tilde{\mathcal{P}}_{\mathcal{G}}^{\mathcal{X}}(\mu)   \, d\kt   \\ &\leq \frac{1}{T-t_{0}} \int  _{t_{0}}^{T} f(\kt)^{2} \, d \kt \leq C_{t_{0},T,\gamma,\mathrm{B},d}.  \end{align*} \qedblack

\section{Approximation in Gaussian Space}



\subsection{Optimal Linear Approximation} \label{appendix1}

Note that the space $\h$ is actually the tensor product of unvariate RKHS with the kernels $\exp \{- \lambda_{l}^2(a-b)^2\}$ for $a,b \in \mathbb{R}$, so we first consider the univariate case with $k_{\lambda_{1}}(a,b)=\exp \{ -\lambda _{1}^{2}(a-b)^{2} \}$ where $a, b \in\mathbb{R}$ and the gaussian measure $\rho _{1,\ka}$ with density function $(2\pi \kappa^{2})^{-\frac{1}{2}}\exp \left\{  - \frac{a^{2}}{2\kappa^{2}}  \right\}$. 

For $j \geq 1$, the eigenvalues and eigenfunctions are given in \citet{fasshauer2012dimensionindependent,rasmussen2006gaussian} by $$r_{\lambda_{1},j}= (\sqrt{ 2 }\kappa)^{-1} \lambda_{1}^{2j-2}/ \mathcal{C}_{1}^{j-\frac{1}{2}} \text{ where } \mathcal{C}_{1}= \lambda_{1}^{2}+\frac{1}{4\kappa^{2}}+\frac{1}{2\ka}\sqrt{ \frac{1}{4\ka^{2}} +2\lambda_{1}^{2}},$$ 
and 
\begin{align*}
   \tilde{\varphi} _{\lambda_{1},j}(a)=\exp \left( -\left( \frac{1}{2\kappa}\sqrt{ \frac{1}{4\kappa^{2}} +2\lambda _{1}^{2}} -\frac{1}{4\kappa^{2}}\right)a^{2} \right)H_{j-1}\left( \frac{1}{\kappa^{\frac{1}{2}}} \left( \frac{1}{4\kappa^{2}}+2\lambda_{1}^{2} \right)^{\frac{1}{4}} a\right)
\end{align*}
where $H_{j-1}$ is the Hermite polynomial of degree $j-1$, given by 
\begin{align*}
    H_{j-1}(a)=(-1)^{j-1} e^{a^{2}} \frac{d^{j-1}}{da^{j-1}}e^{-a^{2}} \text{ for  } a \in \mathbb{R}
\end{align*}
such that 
\begin{align*}
    \int  _{\mathbb{R}} H^{2}_{j-1}(a) \exp(-a^{2}) \, da=\sqrt{ \pi }2^{j-1} (j-1)! \text{ for  } j \in \mathbb{N}. 
\end{align*}

Then we can take a orthonormal basis of $L^{2}(\rho _{1,\ka})$ to be $\{ \varphi _{\lambda_{1},j} \}_{j \in \mathbb{N}}$ by 
\begin{align*}
    \varphi _{\lambda_{1},j}(a)=\sqrt{ \frac{(1+8\kappa^{2}\lambda_{1}^{2})^{\frac{1}{4}}}{2^{j-1}(j-1)!} } \exp\left( - \frac{2\lambda_{1}^{2}a^{2}}{\sqrt{ 1+8\ka^{2}\lambda_{1}^{2} }+1} \right)H_{j-1}\left( \frac{1}{\sqrt{ 2 }\ka} (1+8\ka^{2}\lambda_{1}^{2})^{\frac{1}{4}} a \right),
\end{align*}
and observe that $(\sqrt{ 2 }\kappa )^{-1}\mathcal{C}_{1}^{-\frac{1}{2}}=1- \frac{\lambda_{1}^{2}}{\mathcal{C}_{1}}$, which allows us to rewrite $r_{\lambda_{1},j}$ as $r_{\lambda_{1},j}=(1-\eta _{\lambda_{1}})\eta _{\lambda_{1}}^{j-1}$ with 

\begin{align*}
    \eta _{\lambda_{1}}=\frac{\lambda _{1}^{2}}{\mathcal{C}_{1}}= \frac{4\kappa^{2}\lambda_{1}^{2}}{4\kappa^{2}\lambda_{1}^{2}+1+\sqrt{ 1+8\kappa^{2}\lambda_{1}^{2} } }\in(0,1).
\end{align*}
For the multivariate case with $\boldsymbol{\lambda}=(\lambda_{1},\dots,\lambda _{d})$, let $\boldsymbol{j}$ be a multi-index with $\boldsymbol{j}=(j_{1},\dots,j_{d}) \in \mathbb{N}^{d}$. Then the pairs $(\rb_{\boldsymbol{j}}, \vp_{\boldsymbol{j}})$ of eigenvalues and eigenfunctions are given by 
\begin{align*}
    \rb_{\boldsymbol{j}}:= \prod _{l=1}^{d}r_{\lambda _{l},j_{l}}=\prod _{l=1}^{d}(1-\eta _{\lambda _{l}})\eta _{\lambda _{l}}^{j_{l}-1} \text{ and } \vp_{\boldsymbol{j}}(x):=\prod _{l=1}^{d} \varphi _{\lambda _{l},j_{l}}(x^{(l)}) \text{ for } x=[x^{(1)},..., x^{(d)}] \in \mathbb{R}^{d}.
\end{align*}
We can also define an orthonormal basis $(\p_{\boldsymbol{j}})$ on $\h$ by 
\begin{align*}
    \p_{\boldsymbol{j}}(x):=\prod _{l=1}^{d} \psi _{\lambda _{l},j_{l}}(x^{(l)}) \text{ where }\psi _{\lambda _{l},j_{l}}:= \sqrt{ r_{\lambda _{l},j_{l}} }\varphi _{\lambda _{l},j_{l}}. 
\end{align*}

For the simplicity of notation, we rearrange the sequence of eigenpairs $(r_{\boldsymbol{j}}^{\boldsymbol{\lambda}},\psi _{\boldsymbol{j}}^{\boldsymbol{\lambda}})_{\boldsymbol{j}\in \mathbb{N}^{d}}$ to the sequence $(\rb_{q}, \p_{q})_{q \in\mathbb{N}}$ with the order of a non-increasing sequence of eigenvalues, i.e., $\rb_{1} \geq \rb_{2} \geq \dots>0$.

By Corollary 4.12 in \cite{novak2008tractability}, the optimal linear approximation error is $$\mathcal{E}(n,\h):= \inf_{\Lambda _{n}\subset \h}\sup_{\| f \|_{\h}\leq 1}\left\| f- \operatorname{ Proj }_{\Lambda _{n}}(f)  \right\|_{L^{2}(\rk)}=\sqrt{ \rb_{n+1} }$$ where $\Lambda _{n}$ is an $n$-dimensional subspace of $\h$. By Theorem 5.2 in \cite{fasshauer2012dimensionindependent}, $\mathcal{E}(n,\h) \leq C_{\delta,\kappa,\theta}n^{-\max\left( \theta, \frac{1}{2} \right)+\delta}$ for any $\delta>0$ where $C_{\delta,\kappa,\theta}$ only depends on $\delta$, $\kappa$ and $\theta$.

\subsection{Approximation of Eigenfunctions by Two-Hidden-Layer Tanh Neural Networks} \label{appendix2}

In this part, we show $L^{2}(\rho)$-approximation of orthonormal basis defined in Appendix \ref{appendix1} by neural networks where $\rho$ is a probability distribution with $\| \w \|_{L^{\gamma}(\rk)}< \infty$. 

Recall that $\p_{\boldsymbol{j}}$ has a product form of factors being elements with a unit norm in $\mathcal{H}_{k_{\lambda _{l}}}$. To approximate analytic functions with this form, we apply the shallow neural network with and tanh activation functions and a product-gated output defined in (\ref{def:transformer}).

By scaling and translating the variable, for each pair $(\lambda _{l}, j_{l} )$, we define $\g(t):= \pj\left( 2B\left( t -\frac{1}{2} \right) \right)$ for $t \in [0,1]$ with some number $B>0$. Here we introduce the class of $(Q,R)$-analytic functions with $Q,R >0$ in which an analytic function $f$ satisfies the smoothness condition that $\| D^{\beta}f \|_{L^{\infty}([0,1]^{d})} \leq Q R^{-\beta}\beta!$ for all $\beta \in \mathbb{N}$. 

By Theorem 1 in \cite{zhou2008derivative}, for each $\psi _{\lambda _{l},j_{l}} \in \mathcal{H}_{k_{\lambda _{l}}}$, we have that 
\begin{align*}\left| D^{\beta} \g (t)   \right| &= \left|  (2B)^{\beta}D^{\beta}\pj\left( 2B\left( t-1 /2 \right)  \right) \right|  = \left|  (2B)^{\beta} \left\langle   (D^{\beta}k_{\lambda _{j}})_{2B\left( t-\frac{1}{2} \right)} ,\pj  \right\rangle  _{\mathcal{H}_{k_{\lambda _{l}}}}  \right| \\ &\leq (2B)^{\beta} \sqrt{  D^{(\beta,\beta)}k_{\lambda _{l}}(x,x)  } \leq (4B\lambda _{l})^{\beta} \beta! \leq (4BC_{\theta})^{\beta} \beta!.
\end{align*} 
It implies that $\g$ is a $(1, (4BC_{\theta})^{-1})$-analytic function for each pair $(\lambda _{l},j_{l})$. 

\vskip 0.1in 
Indeed, for $k_{\lambda _{l}}(a,b)= \exp \{ -\lambda _{l}^{2} (a-b)^{2}\}$, 
\begin{align*}
    D^{(\beta,\beta)}k_{\lambda _{l}}=\partial _{a}^{\beta} \partial _{b}^{\beta}k_{\lambda _{l}}=(-\lambda _{l}^{2})^{\beta} \partial _{c}^{2\beta} \exp \{ -c ^{2} \} \text{ with } c =\lambda _{l}(a-b).
\end{align*}
By the definition of Hermite polynomials, $\partial _{c}^{2\beta} \exp(-c ^{2})=H_{2\beta}(c)\exp(-c ^{2})$. It follows that 
\begin{align*}
   \partial _{c}^{2\beta} \exp(-c ^{2})|_{c=0}= H_{2\beta}(0)= (-1)^{\beta} \frac{(2\beta)!}{\beta!}
\end{align*}
which is called Hermite numbers of the even order. Then we have 
\begin{align*}
    \sqrt{ D^{(\beta,\beta)} k_{\lambda _{l}}(x,x) }\leq \lambda _{l}^{\beta} \sqrt{ \frac{(2\beta)!}{\beta!} } = \lambda _{l}^{\beta} \sqrt{ \binom{2\beta}{\beta} \beta!}\leq (2\lambda _{l})^{\beta}\beta!,
\end{align*}
which proves the claim with $Q=1, R=(4BC_{\theta})^{-1}$. 

\vskip 0.1in 
The following lemma follows from an application of Theorem B.7 in \cite{deryck2023error} and Corollary 5.5 in \cite{deryck2021approximation} by taking $s=4\tm, N=\tm$.
\begin{lemma} \label{lem2}
   For $B>1$, each $\g(t)=\pj\left( 2B\left( t- \frac{1}{2} \right) \right)$ on $[0,1]$. For $\tilde{m}>3$, There exists a tanh neural network $\gm$ with two hidden layers of width at most $8\tilde{m}$ such that 
   \begin{align*}
       \| \g-\gm \|_{L^{\infty}([0,1])} \leq 2 \exp \left( -4\tilde{m} \log  \left( \frac{\tilde{m}}{6BC_{\theta}} \right) \right) 
   \end{align*}
    with the parameters bounded by $c'_{1} (c'_{2} \tilde{m})^{160\tilde{m}^{2}}$ where $c'_{1},c'_{2}$ are two absolute constants. 
\end{lemma}

We define $\hat{\psi}_{\lambda _{l},j_{l}}^{\tilde{m}}(t)= \hat{g}_{\lambda _{l},j_{l}}^{\tilde{m}}\left( \frac{t}{2B} + \frac{1}{2} \right)$ and recall the product gate $\mathcal{T}_{1,\odot}$ defined as $$\pt(x) = \prod _{l=1}^{d} \mathcal{T}_{1}(x_{l}) \text{ with } \mathcal{T}_{1}(x_{l})=x_{l} \text{ if }\left| x_{l} \right|<1 \text{ otherwise } \frac{x_{l}}{\left| x_{l} \right| }$$
($\mathcal{T}_{1}$ can be also implemented by a fixed ReLU neural network as 
$\sigma(x_{l}+1)-\sigma(x_{l}-1)-1$). Then we can construct an approximant for $\p_{\boldsymbol{j}}= \prod_{l=1}^{d}\pj$ by $\hat{\psi}_{\boldsymbol{j},\tilde{m}}^{\boldsymbol{\lambda}}:=\pt\big((\hat{\psi}_{\lambda _{1},j_{1}}^{\tilde{m}},\dots, \hat{\psi}_{\lambda _{d},j_{d}}^{\tilde{m}})\big)$ with $L^{2}(\rho)$ approximation error $$\begin{aligned}
\left\| \p_{\boldsymbol{j}}- \hat{\psi}^{\boldsymbol{\lambda}}_{\boldsymbol{j},\tilde{m}} \right\|_{L^{2}(\rho)}^{2} &= \left( \int_{\| x \|_{\infty} \leq B } \ + \int  _{\| x \|_{\infty} >B } \, \right) (\p_{\boldsymbol{j}}(x)-\hat{\psi}^{\boldsymbol{\lambda}}_{\boldsymbol{j},\tilde{m}}(x))^{2} \, d \rho(x) \\ &    \leq \sup_{\| x \|_{\infty} \leq B } ( \p_{\boldsymbol{j}}(x)-\hat{\psi}^{\boldsymbol{\lambda}}_{\boldsymbol{j},\tilde{m}}(x)  )^{2}   + 2\rho(\{ x:\| x \|_{\infty}>B  \}).  
\end{aligned}$$
For the first term, we bound it with Lemma \ref{lem2} by introducing intermediate terms as follows: $$\begin{aligned}& \sup_{\| x \|_{\infty}\leq B } \bigl|\psi^{\boldsymbol{\lambda}}_{\boldsymbol j}(x)-\hat\psi^{\boldsymbol\lambda}_{\boldsymbol j,\tilde m}(x)\bigr|
=\Bigl\|\prod_{l=1}^{d}\pj-\prod_{l=1}^{d}\mathcal{T}_{1}(\hat{\psi}_{\lambda _{l},j_{l}}^{\tilde{m}})\Bigr\|_{L^{\infty}([-B,B]^{d})}\\[2pt]
\le &\Bigl\|\prod_{l=1}^{d}\pj-\mathcal{T}_{1}(\hat{\psi}_{\lambda _{1},j_{1}}^{\tilde{m}})\prod_{l=2}^{d}\pj+\cdots+\prod_{l'=1}^{h}\mathcal{T}_{1}(\hat{\psi}_{\lambda _{l'},j_{l'}}^{\tilde{m}})\prod_{l=h+1}^{d}\pj-\prod_{l'=1}^{h+1}\mathcal{T}_{1}(\hat{\psi}_{\lambda _{l'},j_{l'}}^{\tilde{m}})\prod_{l=h+2}^{d}\pj
\\ & +\cdots + \prod _{l'=1}^{d-1}\mathcal{T}_{1}(\hat{\psi}_{\lambda _{l},j_{l}}^{\tilde{m}}) \psi _{\lambda _{d},j_{d}}- \prod _{l=1}^{d}\mathcal{T}_{1}(\hat{\psi}_{\lambda _{l},j_{l}}^{\tilde{m}})\Bigr\|_{L^{\infty}([-B,B]^{d})} \\ \leq & d \max_{0 \leq h \leq d-1} \left\| \prod_{l'=1}^{h}\mathcal{T}_{1}(\hat{\psi}_{\lambda _{l'},j_{l'}}^{\tilde{m}})\prod_{l=h+1}^{d}\pj-\prod_{l'=1}^{h+1}\mathcal{T}_{1}(\hat{\psi}_{\lambda _{l'},j_{l'}}^{\tilde{m}})\prod_{l=h+2}^{d}\pj \right\|_{L^{\infty}([-B,B]^{d})} \\ \leq & d \max _{0\leq h \leq d-1} \| \psi _{\lambda _{h+1}, j_{h+1}} - \mathcal{T}_{1}(\hat{\psi}^{\tilde{m}}_{\lambda _{h+1},j_{h+1}}) \|_{L^{\infty}([-B,B])}   \leq d \max _{0 \leq h \leq d-1} \| \psi _{\lambda _{h+1}, j_{h+1}} - \hat{\psi}^{\tilde{m}}_{\lambda _{h+1},j_{h+1}} \|_{L^{\infty}([-B,B])} \\ \leq &  2d \exp \left( - 4\tilde{m} \log \left(  \frac{\tilde{m}}{6BC_{\theta}} \right) \right) . \end{aligned}$$
For the second term, we bound it by the subgaussian tail decay of probability measures: 
\begin{align*}
\rho(\{ x: \| x \|_{\infty} >B  \})&= \int_{\| x \|_{\infty} >B } \w (x) \,  d\rk(x) \leq \| \w \|_{L^{\gamma}(\rk)} (\rk(\{ x: \| x \|_{\infty}>B  \}))^{\frac{\gamma-1}{\gamma}} \\ & \leq \| \w \|_{L^{\gamma }(\rk)} \left( d \cdot \frac{\ka}{\sqrt{ 2\pi }}B^{-1} \exp\left( -\frac{B^{2}}{2\ka^{2}} \right) \right)^{\frac{\gamma-1}{\gamma}} \\ & \leq C_{\ka,\gamma}d \|  \w \|_{L^{\gamma}(\rk)} \exp \left( - \frac{\gamma-1}{\gamma} \Big(\frac{B^{2}}{2\ka^{2}}+\log B\Big) \right)
\end{align*}
with $C_{\ka,\gamma}=\left( \frac{\ka}{\sqrt{ 2\pi }} \right)^{\frac{\gamma-1}{\gamma}}$.
Let $B= \frac{\tm^{\frac{3}{4}}}{6C_{\theta}}$ and the above upper bound can written as 
\begin{align*}
\rho(\{ x:\| x \|_{\infty} >B  \}) & \leq dC_{\ka,\gamma} \| \w \|_{L^{\gamma}(\rk)} \exp \left( - \frac{\gamma-1}{\gamma} \Big(\frac{\tm^{\frac{3}{2}}}{72\ka^{2}C_{\theta}^{2}} + \frac{3}{4}\log \tm-\log 6C_{\theta}\Big) \right) \\ & \leq d C_{\ka,\gamma} \| \w \|_{L^{\gamma}(\rk)} \exp (- 2\tm \log \tm)  
\end{align*}
when $\tm > C_{\ka,\theta,\gamma}$ with $C_{\ka,\theta,\gamma}$ a constant only depending on $\ka, \gamma$ and $C_{\theta}$.

Then combine two estimations and we can the final bound as 
\begin{align*}
\| \p_{\boldsymbol{j}}-\hat{\psi}_{\boldsymbol{j},\tm}^{\boldsymbol{\lambda}} \|_{L^{2}(\rho)}^{2} &\leq (2d)^{2} \exp(-2\tm\log\tm) + 2d C_{\ka,\gamma}\| \w \|_{L^{\gamma}(\rk)} \exp(- 2\tm \log \tm) \\ & \leq (4d^{2}+C_{\ka,\gamma}\| \w \|_{L^{\gamma}(\rk)} ) \exp (-2\tm \log \tm)
\end{align*} for $\tm > C_{\ka,\theta,\gamma}$. \qedblack

\vskip 0.2in

\newpage
\bibliography{sample}

\end{document}